\documentclass[nohyperref]{article}

\usepackage{hyperref}

\usepackage[accepted]{icml2022}
\hypersetup{
    colorlinks = true,
    linkbordercolor = {blue},
    citecolor = {blue}
}
\usepackage{amsmath}
\usepackage{amssymb}
\usepackage{mathtools}
\usepackage{amsthm}
\usepackage{verbatim}
\usepackage{bbm}
\usepackage{microtype}
\usepackage{color}
\usepackage{eucal}
\usepackage{epsfig,amsmath,amssymb,amsfonts,amstext,amsthm,mathrsfs}
\usepackage{latexsym,graphics,epsf,epsfig,psfrag}
\usepackage{dsfont,color,epstopdf,fixmath}
\usepackage{enumitem}
\usepackage{algorithm,algorithmic,latexsym,subfigure}

\usepackage{mathtools} 
\usepackage{siunitx} 
\usepackage{booktabs} 
\usepackage{tikz} 
\usepackage[capitalize,noabbrev]{cleveref}

\theoremstyle{plain}
\newtheorem{theorem}{Theorem}[section]

\newtheorem{lemma}[theorem]{Lemma}
\newtheorem{corollary}[theorem]{Corollary}
\theoremstyle{definition}
\newtheorem{definition}[theorem]{Definition}
\newtheorem{assumption}[theorem]{Assumption}
\theoremstyle{remark}

\newcommand{\mcs}{\mathcal{S}}
\newcommand{\mca}{\mathcal{A}}
\newcommand{\nn}{\nonumber}
\newcommand{\mE}{\mathbb{E}}

\newcommand{\proj}{\mathbf{\prod}_{(\Delta(\mca)^{|\mcs|}}}
\newcommand{\st}{s_{\theta}}
\newcommand{\pit}{\pi_{\theta}}
\newcommand{\lse}{\text{LSE}}
\newcommand{\sv}{V_\sigma}
\newcommand{\sq}{Q_\sigma}
\newcommand{\epe}{\epsilon_{\text{est}}}
\newcommand{\pone}{\pi_{\theta_1}}
\newcommand{\ptwo}{\pi_{\theta_2}}

\allowdisplaybreaks

\usepackage[toc,page,header]{appendix}
\usepackage{minitoc}

\icmltitlerunning{Policy Gradient Method For Robust Reinforcement Learning}

\begin{document}

\twocolumn[
\icmltitle{Policy Gradient Method For Robust Reinforcement Learning}



\icmlsetsymbol{equal}{*}

\begin{icmlauthorlist}
\icmlauthor{Yue Wang}{ub}
\icmlauthor{Shaofeng Zou}{ub}
\end{icmlauthorlist}

\icmlaffiliation{ub}{Department of Electrical Engineering, University at Buffalo, New York, USA}
\icmlcorrespondingauthor{Shaofeng Zou}{szou3@buffalo.edu}

\icmlkeywords{Robust Reinforcement Learning, Policy Gradient, Actor-Critic, Global Optimality, Distributional Robustness}
\vskip 0.3in
]



\printAffiliationsAndNotice{}  
\doparttoc 
\faketableofcontents
\begin{abstract}
This paper develops the first policy gradient method with global optimality guarantee and complexity analysis for robust reinforcement learning under model mismatch. 
Robust reinforcement learning is to learn a policy robust to model mismatch between simulator and real environment. We first develop the robust policy (sub-)gradient, which is applicable for any differentiable parametric policy class. We show that the proposed robust policy gradient method converges to the global optimum asymptotically under direct policy parameterization. We further develop a smoothed robust policy gradient method, and show that to achieve an $\epsilon$-global optimum, the complexity is $\mathcal O(\epsilon^{-3})$. We then extend our methodology to the general model-free setting, and design the robust actor-critic method with differentiable parametric policy class and value function. We further characterize its asymptotic convergence and sample complexity under the tabular setting. Finally, we provide simulation results to demonstrate the robustness of our methods.
%
\end{abstract}

\section{Introduction}
In practical reinforcement learning (RL) \cite{sutton2018reinforcement} applications, the training environment may often times deviate from the test environment, resulting in a model mismatch between the two. 
Such model mismatch could be because of, e.g., modeling error between simulator and real-world applications, model deviation due to non-stationarity of the environment, unexpected perturbation and potential adversarial attacks. This may lead to a significant performance degradation in the testing environment.



To solve the issue of model mismatch, a framework of robust Markov decision process (MDP) was introduced in \cite{bagnell2001solving,nilim2004robustness,iyengar2005robust}, where the MDP model is not fixed but comes from some uncertainty set. The goal of robust RL is to find a policy that optimize the worst-case performance over all possible MDP models in the uncertainty set. Value-based approaches have been extensively studied under the tabular setting and with function approximation, e.g., \cite{iyengar2005robust,nilim2004robustness,badrinath2021robust,wiesemann2013robust,roy2017reinforcement,tamar2014scaling,lim2013reinforcement,bagnell2001solving,satia1973markovian,xu2010distributionally,wang2021online}. There are also other approaches that are shown to be successful empirically, e.g., based on adversarial training, \cite{vinitsky2020robust,pinto2017robust,abdullah2019wasserstein,hou2020robust,rajeswaran2017epopt,atkeson2003nonparametric,morimoto2005robust,huang2017adversarial,kos2017delving,lin2017tactics,pattanaik2018robust,mandlekar2017adversarially}, which however lack theoretical robustness and optimality guarantee. 

The policy gradient method \cite{williams1992simple,sutton1999policy,konda2000actor,kakade2001natural}, which models and optimizes the policy directly, has been widely used in RL thanks to its ease of implementation in model-free setting, scalability to large/continuous state and action spaces, and applicability to any differentiable policy parameterization. Despite a large body of empirical and theoretical work on policy gradient method, 
development of policy gradient approach for robust RL with provable robustness to model mismatch and optimality guarantee still remains largely open in the literature.

In this paper, we develop the first policy gradient method  for robust RL under model mismatch with provable robustness, global optimality and complexity analysis. In this paper, we focus on the $R$-contamination uncertainty set model \cite{hub65,du2018robust,huber2009robust,wang2021online,nishimura2004search,prasad2020learning,prasad2020robust}. Our robust policy gradient method inherits advantages of vanilla policy gradient methods and their variants, and provide provable guarantee on global optimality and robustness. In particular, the challenges and our major contributions are summarized as follows.
\begin{itemize}[leftmargin=*]
    \item  Robust RL aims to optimize the worst-case performance, named robust value function,  where the worst-case is taken over some uncertainty set of MDPs. However, the robust value function involves a ``max" and thus may not be differentiable in the policy. Our first contribution in this paper is the development of robust policy gradient, where we derive the Fr\'{e}chet sub-gradient of the robust value function, and further show that it is the gradient almost everywhere. We would like to highlight that our robust policy gradient applies to any differentiable and Lipschitz policy class. 
    \item Motivated by recent advancements on the global optimality of vanilla policy gradient methods, we are interested in a natural question that whether the global optimum of robust RL can be attained by our robust policy gradient method. The major challenge lies in that the robust value function involves a ``max" over the uncertainty set, and thus has a much more complicated landscape than the vanilla value function. We consider the direct parametric policy class and show that the robust value function satisfies the Polyak-{\L}ojasiewicz (PL) condition \cite{polyak1963gradient,lojasiewicz1963topological}, and our robust policy gradient method converges to a global optimum almost surely.
    
    
    \item The robust value function may not be differentiable everywhere, which is the major challenge in the convergence rate analysis. We then design a smoothed robust policy gradient method  as an approximation, where the corresponding smoothed objective function is differentiable. We show that smoothed robust policy gradient method converges to an $\epsilon$-global optimum of the original non-differentiable robust RL problem with a complexity of $\mathcal{O}(\epsilon^{-3})$. 
    
    \item  To understand the fundamentals of our robust policy gradient method, our results discussed so far focus on the ideal setting assuming perfect knowledge of the (smoothed) robust policy gradient. Although this is a commonly used setting in recent studies of vanilla policy gradient, e.g., \cite{agarwal2021theory,cen2021fast,mei2020global,bhandari2019global}, such knowledge is typically unknown in practice and need to be estimated from samples. We then focus on the model-free setting, where only samples from the centroid of the uncertainty set are available, and we
    design a model-free robust actor-critic algorithm. Our robust actor-critic can be applied with arbitrary differential parametric policy class and value function approximation in practice. Theoretically, we prove the global optimality of our robust actor critic method under the tabular setting with direct parametric policy class. 
\end{itemize}

\subsection{Related Works}

\textbf{Global optimality of vanilla policy gradient method.} In the non-robust setting, policy gradient methods \cite{williams1992simple,sutton1999policy} as well as their extensions \cite{sutton1999policy,konda2000actor,kakade2001natural,schulman2015trust,schulman2017proximal} have been  successful in various  applications, e.g., \cite{schulman2015trust,schulman2017proximal}. Despite the surge of interest in policy gradient methods, theoretical understanding remains limited to convergence to local optimum and stationary points. It was not until recently that the global optimality of various policy gradient methods was established \cite{bhandari2021linear,bhandari2019global,agarwal2021theory,mei2020global,li2021softmax,laroche2021dr,zhang2021global,cen2021fast,zhang2020variational,xiao2022On}.
In this paper, we focus on policy gradient methods for robust RL. The major new challenge lies in that the robust value function is not differentiable everywhere, and the landscape of the robust value function is much more complicated than the vanilla value function.

\textbf{Value-based approach for robust RL.} Robust MDP was introduced and studied in \cite{iyengar2005robust,nilim2004robustness,bagnell2001solving,satia1973markovian,wiesemann2013robust,lim2019kernel,xu2010distributionally,yu2015distributionally,lim2013reinforcement,tamar2014scaling}, where the uncertainty set is assumed to be \textit{known} to the learner, and the problem can be solved using dynamic programming. 
Later, the studies were generalized to the model-free setting where stochastic samples from the central MDP of the uncertainty set are available in an online fashion \cite{roy2017reinforcement,badrinath2021robust,wang2021online,tessler2019action} and an offline fashion \cite{zhou2021finite,yang2021towards,panaganti2021sample,goyal2018robust,kaufman2013robust,ho2018fast,ho2021partial,si2020distributionally}.
In this paper, we focus on approaches that model and optimize the policy directly, and develop robust policy gradient method. Our method inherits advantages of policy gradient, and has a broader applicability than value-based method for large-scale problems.

\textbf{Direct policy search for robust RL.} 
Robust policy gradient method for constrained MDP was studied in \cite{russel2020robust}, however, there are mistakes in the gradient derivation. More specifically, the fact that the worst-case transition kernel is a function of the policy was ignored when deriving the gradient. A recent paper \cite{derman2021twice} showed the equivalence between robust MDP and regularized MDP, and developed a policy gradient method for the case with only reward uncertainty. In \cite{derman2021twice}, it was discussed that it is difficult to their methods extend to problems with uncertain transition kernel because of the dependency between the worst-case transition kernel and the policy. \cite{eysenbach2021maximum} showed a similar result  that maximum entropy regularized MDP is robust to model mismatch.
In this paper, we focus on the challenging problem with uncertain transition kernel, and derive the robust policy gradient. 
We also note that a separate line of work \cite{zhang2021robust,zhang2021corruption} studies the corruption-robust RL problems, where the goal is to learn a robust policy to data corruption, and is fundamentally different from the problem in this paper.

\section{Preliminaries}
\textbf{Markov Decision Process.}
An MDP  $(\mathcal{S},\mathcal{A},  \mathsf P, c, \gamma)$ is specified by: a state space $\mcs$, an action space $\mca$, a transition kernel $\mathsf P=\left\{p^a_s \in \Delta(\mcs), a\in\mca, s\in\mcs\right\}$\footnote{$\Delta(\mcs)$ denotes the $(|\mcs|-1)$-dimensional probability simplex on $\mcs$. }, where $p^a_s$ is the distribution of the next state over $\mcs$ upon the agent taking action $a$ in state $s$, a cost function $c: \mcs\times\mca \to [0,1]$, and a discount factor $\gamma\in[0,1)$. At each time step $t$, the agent in state $s_t$ takes an action $a_t$. The environment then transits to the next state $s_{t+1}$ according to the distribution $p^{a_t}_{s_t}$, and provides a cost signal $c(s_t,a_t)\in [0,1]$ to the agent.

A stationary policy $\pi: \mcs\to \Delta(\mca)$ maps any state to a distribution over the action space $\mca$. More specifically, in state $s$, the agent takes action $a$ with probability $\pi(a|s)$.
The value function of a stationary policy $\pi$ starting from $s\in\mcs$ measures the expected accumulated discounted cost by following policy $\pi$: 
$\mathbb{E}\left[\sum_{t=0}^{\infty}\gamma^t   c(S_t,A_t )|S_0=s,\pi\right]$, and the goal is to find a policy that minimizes the value function for any initial state $s\in\mcs$.

\textbf{Robust MDP.}
The transition kernel of the robust MDP  is not fixed but is from some uncertainty set $\mathcal{P}$. In this paper, we focus on the $(s,a)$-rectangular uncertainty set \cite{nilim2004robustness,iyengar2005robust}, i.e., $\mathcal{P}=\bigotimes_{s,a} \mathcal{P}^a_s$, where $\mathcal{P}^a_s \subseteq \Delta(\mcs)$. 
At each time step, after the agent takes an action, the environment transits to the next state following any transition kernel $\mathsf P\in\mathcal{P}$, and the choice of kernels can be time-varying. A sequence of transition kernel $\kappa=(\mathsf P_0,\mathsf P_1...)\in\bigotimes_{t\geq 0} \mathcal{P}$  can be viewed as a policy chosen by the nature and is referred to as the nature's policy.

The robust value function of a policy $\pi$ is defined as the worst-case expected accumulated discounted cost over $\kappa$ when  following $\pi$ and starting from $s$:
\begin{align}\label{eq:Vdef}
    V^\pi(s)\triangleq \max_{\kappa\in\bigotimes_{t\geq 0} \mathcal{P}} \mathbb{E}_{\kappa}\left[\sum_{t=0}^{\infty}\gamma^t   c(S_t,A_t )|S_0=s,\pi\right].
\end{align}
The robust action-value function can be defined:
$
    Q^\pi(s,a)\triangleq \max_{\kappa\in\bigotimes_{t\geq 0} \mathcal{P}} \mathbb{E}_{\kappa}\left[\sum_{t=0}^{\infty}\gamma^t   c(S_t,A_t )|S_0=s,A_0=a,\pi\right].
$ It has been shown that $V^\pi(s)=\sum_{a\in\mca} \pi(a|s) Q^\pi(s,a)$ \cite{nilim2004robustness,iyengar2005robust}.

The robust Bellman operator of a policy $\pi$ is defined as 
\begin{align}\label{eq:bellman}
    \mathbf T_\pi V(s)\triangleq \sum_{a\in\mca} \pi(a|s) \left(c(s,a)+\gamma \sigma_{\mathcal{P}^a_s}(V) \right),
\end{align}
where $\sigma_{\mathcal{P}^a_s}(V)\triangleq \max_{p\in\mathcal{P}^a_s} p^\top V$ is the support function of $V$ on $\mathcal{P}^a_s$. $\mathbf T_\pi$ is a contraction and $V^\pi$ is the unique fixed point \cite{nilim2004robustness,iyengar2005robust,puterman2014markov}.

 Define the expected worst-case  total cost function under the initial distribution $\rho$ as $J_\rho(\pi)\triangleq\mathbb{E}_{S\sim\rho}[V^\pi(S)].$
The goal of the agent is to find an optimal policy that minimizes $J_\rho(\pi)$:
$
    \min_{\pi} J_\rho(\pi).
$

\textbf{$R$-Contamination Uncertainty Set.}
In this paper, we focus on an adversarial model of the uncertainty set, $R$-contamination, where the nature could arbitrarily perturb the state transition of the MDP with a small probability. Let $\mathsf P=\bigotimes_{s\in\mcs, a\in\mca} p_s^a $ be a transition kernel. The $R$-contamination uncertainty set centered at $\mathsf P$ is defined as $\mathcal{P}\triangleq\bigotimes_{s,a}\mathcal{P}^a_s$, where
\begin{align}
    \mathcal{P}^a_s\triangleq \left\{(1-R)p^a_s+Rq|q\in\Delta(\mcs) \right\},  s\in\mcs, a\in\mca.
\end{align}    
This uncertainty set model is widely used in the literature of robust learning and optimization, e.g., \cite{hub65,du2018robust,wang2021online,huber2009robust,nishimura2004search,Kiyohiko2006,prasad2020learning,prasad2020robust}. The $R$-contamination set models the scenario where the state transition could be arbitrarily perturbed with a small probability $R$, hence is more suitable for systems suffering from random perturbations, adversarial attacks, and outliers in sampling. $R$-contamination set can also be connected to  uncertainty sets defined by total variation, KL-divergence and Hellinger distance via inequalities, e.g., Pinsker’s inequality. On the other hand, the $R$-contamination model is more clean and straightforward, which makes the derivation of the robust policy gradient, and the convergence and complexity analyses tractable.



Under the $R$-contamination model, the support function can be easily computed as follows:
\begin{align}
    \sigma_{\mathcal{P}^a_s}(V)=(1-R)\sum_{s'\in\mcs}p^a_{s,{s'}}V({s'})+R\max_{{s'}} V({s'}),
\end{align}
where $p^a_{s,{s'}}=p^a_{s}(s')$.

\section{Robust Policy Gradient}\label{sec:rpgmain}
Consider a parametric policy class $\Pi_\Theta=\left\{ \pit: \theta \in\Theta\right\}$. Denote $J_\rho(\pit)$ by $J_\rho(\theta)$.  Robust RL aims to find an optimal policy $\pi_{\theta^*}\in \Pi_\Theta$ that minimizes the expected worst-case accumulated discounted cost:
\begin{align}
    {\theta^*}\in\arg\min_{\theta\in\Theta} J_\rho(\theta).
\end{align}
Let $J_\rho^*\triangleq \min_{\theta\in\Theta}J_\rho(\theta)$.
Recall the definition of $V^\pi(s)$ in \eqref{eq:Vdef}. Due to the $\max$ over $\kappa$, $V^\pi(s)$ may not be differentiable.  To solve this issue, we adopt the Fr\'{e}chet sub-differential \cite{kruger2003frechet}.
\begin{definition}
For a function $f:  \mathcal{X}\subseteq \mathbb{R}^N\to \mathbb{R}$, a vector $u\in\mathbb{R}^N$ is called a Fr\'{e}chet sub-differential of $f$ at $x$, if 
\begin{align}
    \lim_{h\to 0}\inf_{h\neq 0} \frac{f(x+h)-f(x)-\langle h, u\rangle}{\|h\|}\geq 0.
\end{align}
The set of all the sub-differential of $f$ at $x$ is denoted by $\partial f(x)$.
\end{definition}
Clearly, when $f(x)$ is differentiable at $x$, $\partial f(x)=\left\{ \nabla f(x)\right\}$. 

Without loss of generality, we assume that the parametric policy class is differentiable and Lipschitz.
\begin{assumption}\label{ass:lip}
The policy class $\Pi_\Theta$ is differentiable and $k_\pi$-Lipschitz, i.e., for any $s\in\mcs,a\in\mca$ and $\theta\in\Theta$, there exists a  universal constant $k_\pi$, such that
$
    \|\nabla \pit(a|s)\| \leq k_\pi.\label{eq:assum1}
$
\end{assumption}
This assumption  can be easily satisfied by many policy classes, e.g., direct parameterization \cite{agarwal2021theory}, soft-max \cite{mei2020global,li2021softmax,zhang2021global,wang2020finite}, or neural network with Lipschitz and smooth activation functions \cite{du2019gradient,neyshabur2017implicit,miyato2018spectral}.

Define the discounted visitation distribution as 
$d^\pi_s(s')\triangleq (1-\gamma+\gamma R) \sum^\infty_{t=0} \gamma^t(1-R)^t \cdot \mathbb{P}(S_t=s'|S_0=s,\pi),$
and let 
 $d^\pi_\rho(s')=\mathbb{E}_{S\sim\rho}[d^\pi_S(s')].$
Denote $\st\triangleq \arg\max_s V^{\pit}(s)$. Then, the (sub-)gradient of $J_\rho(\theta)$ can be computed in the following theorem. 
\begin{theorem}\label{coro:subJ}(Robust Policy Gradient)
Consider a class of policies $\Pi_\Theta$ satisfying \cref{ass:lip}.
For any distribution $\rho$, denote 
\begin{align}\label{eq:subgra}
    &\psi_\rho(\theta)\triangleq\nn\\
    &\frac{\gamma R}{(1-\gamma)(1-\gamma+\gamma R)}\sum_{s\in\mcs} d^{\pit}_{\st}(s) \sum_{a\in\mca} \nabla \pit(a|s) Q^{\pit}(s,a) \nn\\
    &+\frac{1}{1-\gamma+\gamma R} \sum_{s\in\mcs} d^{\pit}_{\rho}(s) \sum_{a\in\mca} \nabla \pit(a|s) Q^{\pit}(s,a),
\end{align}
Then, 
(1) almost everywhere in $\Theta$, $J_\rho(\theta)$ is differentiable and $\psi_\rho(\theta)=\nabla J_\rho(\theta)$; and
(2) at  non-differentiable $\theta$, $\psi_\rho(\theta)\in\partial J_\rho(\theta)$. 

\end{theorem}
If we set $R=0$, i.e., the uncertainty set $\mathcal P$ reduces to a singleton $\{\mathsf P\}$ and there is no robustness, then $\psi_\rho(\theta)$ in \eqref{eq:subgra} reduces to the vanilla policy gradient \cite{sutton1999policy}. 

As can be seen from \eqref{eq:subgra}, the robust policy (sub)-gradient is a function of the robust Q-function. Note that in robust RL, $J_\rho(\theta)$ may not be differentiable everywhere, and therefore, the sub-gradient is needed. 
The mistake in \cite{russel2020robust} is due to the ignorance of the dependence of the worst-case kernel  on $\theta$, and their robust policy gradient is a function of the vanilla Q-function not the robust Q-function, which should not be the case.

In policy gradient approaches, the agent often has the challenge of exploration. If $\rho$ highly concentrates on a subset of $\mcs$, then the agent may not be able to explore all the states and may end up with a sub-optimal solution. To solve this issue, we introduce an  optimization measure $\mu$ satisfying $\mu_{\min} \triangleq\min_s\mu(s)>0$ \cite{agarwal2021theory}. The initial state distribution $\rho$ is called the performance measure. As we will show in the next section, although we want to minimize $\mE_{\rho}[V^\pi]$, we can perform sub-gradient descent with respect to $J_\mu\triangleq\mE_{\mu}[V^\pi]$ and the algorithm can still find an optimum of $\mE_{\rho}[V^\pi]$.

Given the robust policy gradient in \Cref{coro:subJ}, we design our robust policy gradient algorithm in \cref{alg:rpgfull}. We note that our \Cref{alg:rpgfull} can be applied to any arbitrarily parameterized policy class that is differentiable and Lipschitz. Here $\mathbf{\prod}_\Theta$ denotes the projection onto $\Theta$.
\begin{algorithm}[!htb]
\caption{Robust Policy Gradient}
\label{alg:rpgfull}
\textbf{Input}:   $T$, $\alpha_t$ \\
\textbf{Initialization}: $\theta_0$
\begin{algorithmic} 
\FOR {$t=0,1,...,T-1$}
\STATE {$\theta_{t+1}\leftarrow \mathbf{\prod}_\Theta(\theta_t-\alpha_t \psi_\mu(\theta_t))$}
\ENDFOR
\end{algorithmic}
\textbf{Output}: $\theta_T$
\end{algorithm}

\section{Global Optimality: Direct Parameterization}\label{sec:globalrpg}

In this section, we show that the robust objective function $J_\rho(\theta)$ satisfies the Polyak-{\L}ojasiewicz (PL) condition when the direct parametric policy class is used, i.e., $\Theta=(\Delta(\mca))^{|\mcs|}$ and $\pit(a|s)=\theta_{s,a}$, and we further show the global optimality of \Cref{alg:rpgfull}.

\Cref{alg:rpgfull} is in fact a sub-gradient descent algorithm for a non-convex function $J_\mu$. Following classic results from stochastic approximation and optimization \cite{beck2017first,borkar2009stochastic,borkar2000ode}, \Cref{alg:rpgfull} is expected to converge to stationary points only. Showing the global optimality requires further characterization of $J_\mu$, which involves the ``max" over $\kappa$, and is thus more challenging than the vanilla non-robust case.

We first show that the robust objective function $J_\rho$ satisfies the PL-condition under the direct parameterization. 
Informally, a function $f(\theta)$ is said to satisfy the PL condition if 
$
    f(\theta)-f(\theta^*) \leq \mathcal{O}(F(\nabla f(\theta))),
$
for some suitable scalar notion of first-order stationarity $F$, which measures how large the gradient is \cite{karimi2016linear,bolte2007lojasiewicz}.
This condition implies that if a solution is close to some first-order stationary point, the function value is then close to the global optimum. 

\begin{theorem}[PL-Condition]\label{thm:PL}
Under direct policy parameterization and for any optimization measure $\mu\in\Delta(\mcs)$ and performance measure $\rho\in\Delta(\mcs)$,
\begin{align}\label{eq:PL}
    J_\rho(\theta)-J_\rho^*\leq C_{PL} \max_{\hat{\pi}\in (\Delta(\mca))^{|\mcs|}} \left\langle \pi_\theta-\hat{\pi}, \psi_\mu(\theta)  \right\rangle,
\end{align}
where $C_{PL}=\frac{1}{(1-\gamma)\mu_{\min}}$.
\end{theorem}
Note that $\psi_\mu(\theta)$ on the right hand side of \eqref{eq:PL} is a sub-gradient of $J_\mu$, and $J_\rho(\theta)$ on the left hand side of \eqref{eq:PL} is the objective function with respect to $\rho$. Therefore, for any optimization measure $\mu$, a stationary point of $J_\mu$ is a global optimum of $J_\rho$. 
Thus the PL-condition in \cref{thm:PL} with results from stochastic approximation will lead to the global optimality of Algorithm \ref{alg:rpgfull}  in the following theorem. 
\begin{theorem}\label{thm:conv}
If $\alpha_t>0$, $\sum^{\infty}_{t=0} \alpha_t=\infty$ and $\sum^{\infty}_{t=0} \alpha^2_t<\infty$, then Algorithm \ref{alg:rpgfull} converges to a global optimum of $J_\rho(\theta)$ almost surely. 
\end{theorem}

Theorem \ref{thm:conv} suggests that our robust policy gradient algorithm converges to a global optimum, which matches the global optimality results for vanilla policy gradients in e.g., \cite{agarwal2021theory,mei2020global}. However, the analysis here is much more challenging due to the non-differentiable objective function and min-max problem structure. 


\section{Smoothed Robust Policy Gradient}\label{sec:srpg}
It is in general challenging to analyze the convergence rate of \cref{alg:rpgfull}, which is a projected sub-gradient descent algorithm for non-differentiable non-convex function. In this section, we construct a smoothed approximation $J_{\sigma,\rho}$, which converges to $J_\rho$ as $\sigma\rightarrow\infty$. We develop a smoothed robust policy gradient, and show that the smoothed $J_{\sigma,\rho}$ satisfies the PL-condition. We characterize its global optimality and show that to achieve an $\epsilon$-global optimal, the complexity is $\mathcal O(\epsilon^{-3})$.

For convenience, in the remaining of this paper we assume that $\rho=\mu$ and  $\rho_{\min}>0$, and we omit the subscript $\rho$ in $J_\rho$ and $J_{\sigma,\rho}$. The algorithm design and theoretical results can be similarly extended to the general setting with $\rho\neq\mu$ as in  \cref{sec:rpgmain} and \cref{sec:globalrpg}.

We use the LogSumExp (LSE) operator to approximate the $\max$ operator, where $$\lse(\sigma,V)=\frac{\log (\sum^d_{i=1} e^{\sigma V(i)})}{\sigma}$$ for $V\in\mathbb{R}^d$ and some $\sigma>0$.  
The approximation error $|\lse(\sigma,V)-\max V|\to 0$ as $\sigma\to\infty$. By replacing $\max$ in \eqref{eq:bellman} using  LSE, the smoothed Bellman operator is 
\begin{align}
    \mathbf T_{\sigma}^{\pi} V(s)&=\mE_{A\sim\pi(\cdot|s)}\bigg[c(s,A)+\gamma(1-R)\sum_{{s'}\in\mcs} p^{A}_{s,{s'}} V({s'})\nn\\
    &\quad\quad+\gamma R\cdot \lse(\sigma, V)\bigg].
\end{align}
The reason why we do not use soft-max is because with soft-max, the induced Bellman operator is not a contraction anymore \cite{Asadi2016,wang2021online}. 
With LSE,  $\mathbf T_{\sigma}^{\pi}$ is a contraction and has a unique fixed point \cite{wang2021online}, which we denote by  $V^\pi_\sigma$ and name as the smoothed robust value function. We can also define the smoothed robust action-value function $Q^\pi_\sigma(s,a)\triangleq c(s,a)+\gamma(1-R)\sum_{{s'}\in\mcs} p^{a}_{s,{s'}} V^\pi_\sigma({s'}) +\gamma R\cdot \lse(\sigma, V^\pi_\sigma)$.

Clearly, $V^\pi_\sigma$ is differentiable in $\theta$ and it converges to $V^\pi$ as $\sigma\to\infty$. 
 We then define the smoothed objective function as 
\begin{align}
    J_\sigma(\theta)=\sum_{s\in\mcs} \rho(s)V^{\pit}_\sigma(s),
\end{align}
and let
$
    J_\sigma^*=\min_\theta  J_\sigma(\theta).
$
Note $J_\sigma(\theta)$ is an approximation of $J(\theta)$, and $J_\sigma^*$ also converges to $J^*$ as $\sigma\to\infty$. In the following theorem, we derive the gradient of $J_\sigma$, which holds for any differentiable policy class.
\begin{theorem}\label{thm:gradJsigma}
Consider a policy class $\Pi_\Theta$ that is differentiable.
The gradient of $J_\sigma(\theta)$ is
\begin{align}\label{eq:gradJ}
    \nabla J_\sigma(\theta)=B(\rho,\theta)+\frac{\gamma R\sum_{s\in\mcs} e^{\sigma \sv^{\pit}(s)} B(s,\theta)}{(1-\gamma)\sum_{s\in\mcs} e^{\sigma \sv^{\pit}(s)}},
\end{align}
where $B(s,\theta)\triangleq\frac{1}{1-\gamma+\gamma R} \sum_{{s'}\in\mcs} d^{\pi}_s({s'})\sum_{a\in\mca} \nabla \pit(a|{s'}) \\\cdot \sq^{\pit}({s'},a)$, and $B(\rho,\theta)\triangleq\mathbb{E}_{S\sim\rho}[B(S,\theta)]$.
\end{theorem}

We then design the smoothed robust policy gradient algorithm in \cref{alg:srpgfull}.
\begin{algorithm}[!htb]
\caption{Smoothed Robust Policy Gradient}
\label{alg:srpgfull}
\textbf{Input}:   $T$, $\sigma$, $\alpha_t$\\
\textbf{Initialization}: $\theta_0$
\begin{algorithmic}
\FOR {$t=0,1,...,T-1$}
\STATE {$\theta_{t+1}\leftarrow \mathbf{\prod}_\Theta \left(\theta_t-\alpha_t \nabla J_\sigma(\theta_t) \right)$}
\ENDFOR
\end{algorithmic}
\textbf{Output}: $\theta_T$
\end{algorithm}

\subsection{Global Optimality Under Direct Parameterization}
We focus on direct policy parameterization.
We show that the smoothed objective function $J_\sigma$ satisfies the PL-condition and develop the global optimality of \cref{alg:srpgfull}.

\begin{theorem}(PL-Condition)\label{thm:smPL}
Consider direct policy parameterization. Then
\begin{align}\label{eq:smPL}
    J_\sigma\left(\theta\right)-J_\sigma^*&\leq C_{PL} \max_{\hat{\pi}\in \left(\Delta\left(|\mca|\right)\right)^{|\mcs|}} \left\langle \pi_\theta-\hat{\pi}, \nabla J_\sigma\left(\theta\right) \right\rangle\nn\\
    &\quad+\left(\frac{\gamma R}{1-\gamma}\right)\frac{2\log |\mcs|}{\sigma}.
\end{align} 
\end{theorem}
This theorem implies that if the gradient $\nabla J_\sigma (\theta)$ is small, then $\theta$ falls into a small neighbour (radius of $\mathcal{O}(\sigma^{-1})$)  of the global optimum $J^*_\sigma$. By choosing a large $\sigma$, the difference between $J_\sigma\left(\theta\right)-J_\rho^*$ can be made arbitrarily small, and thus the global optimality with respect to $J_\rho^*$ can be established.

\subsection{Convergence Rate}
We then study the convergence rate of  \cref{alg:srpgfull} under direct policy parameterization. 
We first make an additional smoothness assumption on $\Pi_\Theta$.
\begin{assumption}\label{ass:smooth}
$\Pi_\Theta$ is $l_\pi$-smooth, i.e., for any $s\in\mcs$ and $a\in\mca$ and for any $\theta_1,\theta_2\in\Theta$, 
\begin{align}
     \|\nabla \pone(a|s)-\nabla \ptwo(a|s)\| &\leq l_\pi\|\theta_1-\theta_2\|. \label{eq:assum2}
\end{align}
\end{assumption}
Then under \cref{ass:smooth}, we show that $J_\sigma$ is $L_\sigma$-smooth.
\begin{lemma}\label{lemma:lipJsigma}
Under Assumptions \ref{ass:lip} and \ref{ass:smooth}, for any $\theta_1,\theta_2$,
\begin{align}
    \|\nabla J_\sigma(\theta_1)-\nabla J_\sigma(\theta_2)\|\leq L_\sigma \|\theta_1-\theta_2\|,
\end{align}
where $L_\sigma=\mathcal{O}(\sigma)$ and its exact definition is in  \eqref{eq:L-sigma}.
\end{lemma}
The value of $\sigma$ controls the tradeoff between the smoothness of $J_\sigma$ and the approximation error between $J_\sigma$ and $J$.
\begin{theorem}\label{thm:srpg}
For any $\epsilon>0$, set 
$
    \sigma=\frac{2\gamma R \log |\mcs|}{\epsilon(1-\gamma)}=\mathcal{O}(\epsilon^{-1})
$
and $
T= \frac{64|\mcs|C^2_{PL}L_\sigma C_\sigma}{\epsilon^2}=\mathcal{O}(\epsilon^{-3}) 
$ in \cref{alg:srpgfull} \footnote{$C_\sigma=\frac{1}{1-\gamma}(1+2\gamma R \frac{\log|\mcs|}{\sigma})$ denotes the upper bound of $Q^\pi_\sigma$.} , then 
\begin{align}
    \min_{t\leq T-1} J(\theta_t)-J^*\leq 3\epsilon.
\end{align}
\end{theorem}
Theorem \ref{thm:srpg} shows that Algorithm \ref{alg:srpgfull} converges to an $\epsilon$-global optimum within $\mathcal{O}(\epsilon^{-3})$ steps. This rate is slower than the one of non-robust policy gradient in \cite{agarwal2021theory} by a factor of $\mathcal{O}(\epsilon^{-1})$, which is due to the additional robustness requirement and the smoothing technique using $\sigma$. 
If we set $R=0$, i.e., no robustness, the value of $\sigma$ will then be irrelevant, and our algorithm reduces to non-robust policy gradient algorithm.
With $R=0$, $L_\sigma=\mathcal O(1)$ and $C_\sigma=\mathcal O(1)$, then our complexity also reduces to $\mathcal{O}(\epsilon^{-2})$, which is the same as the one in \cite{agarwal2021theory}.

\section{Robust Actor-Critic}
The results discussed so far assume full knowledge of robust value functions and visitation distributions, and thus the (smoothed) robust policy gradient is exactly known. Although this is a commonly used setting for theoretical analysis  e.g., in \cite{agarwal2021theory,mei2020global,bhandari2019global,cen2021fast}, such knowledge is usually unknown in practice. In this section, we focus on the practical \textit{model-free} setting where only training samples from the centroid transition kernel $\mathsf P$ can be obtained. 

As can be seen from \eqref{eq:subgra} and \eqref{eq:gradJ}, to obtain the (smoothed) robust policy (sub)-gradient, we first need to estimate the robust value function. Robust value function measures the  performance on the worst-case transition kernel which is typically different from the one that generates samples. However, Monte Carlo (MC) method can only be used to estimate the value function on the kernel that generates the samples.

To solve this issue, we design a robust TD algorithm, and combine it with our robust policy gradient descent to design the robust actor-critic algorithm. 
Consider a parametric robust action value function $Q_\zeta$, e.g., linear function approximation, neural network. The robust TD algorithm is given in \cref{alg:NTD}. Note that by replacing $\max$ in the algorithm by $\lse$ we can also get the smoothed robust TD algorithm to estimate $V_\sigma^\pi$. 
\begin{algorithm}[!htb]
\caption{Robust TD}
\label{alg:NTD}
\textbf{Input}:   $T_c$, $\pi$, $\beta_t$\\
\textbf{Initialization}: $\zeta$, $s_0$
\begin{algorithmic} 
\STATE {Choose $a_0\sim \pi(\cdot|s_0)$}
\FOR {$t=0,1,...,T_c-1$}
\STATE {Observe $c_t$, $s_{t+1}$}
\STATE{Choose $a_{t+1}\sim\pi(\cdot|s_{t+1})$}
\STATE {$V^*_t\leftarrow \max_s \left\{\sum_{a\in\mca} \pi(a|s) Q_\zeta(s,a) \right\}$}
\STATE {$\delta_t \leftarrow Q_{\zeta}(s_t,a_t)-c_t-\gamma (1-R) Q_\zeta(s_{t+1},a_{t+1})-\gamma R V^*_t$}
\STATE {$\zeta\leftarrow \zeta-\beta_t \delta_t \nabla_{\zeta}Q_\zeta(s_t,a_t) $}
\ENDFOR
\end{algorithmic}
\textbf{Output}: $\zeta$
\end{algorithm}

We provide the convergence proof of robust TD under the tabular setting in  \cref{sec:td}. For convergence under general function approximation, additional regularity conditions might be needed \cite{korda2015td,Dalal2018a,bhandari2018finite,cai2019neural,roy2017reinforcement}.

With the robust TD algorithm, we then develop our robust actor-critic algorithm in \cref{alg:act}. The algorithm can be applied with any differentiable value function approximation and parametric policy class.

\begin{algorithm}[!htb]
\caption{Robust Actor-Critic}
\label{alg:act}
\textbf{Input}:   $T$, $T_c$, $\sigma$, $\alpha_t$, M\\
\textbf{Initialization}: $\theta_0$
\begin{algorithmic} 
\FOR {$t=0,1,...,T-1$}
\STATE {Run Algorithm \ref{alg:NTD} for $T_c$ times} 
\STATE {$Q_t\leftarrow Q_{\zeta_{T_c}}$}
\STATE {$V_t(s)\leftarrow \sum_{a\in\mca} \pit(a|s)Q_t(s,a)$ \text{ for all }$s\in\mcs$}
\FOR {$j=1,...,M$}
\STATE Sample $T^j\sim\textbf{Geom}(1-\gamma+\gamma R)$
\STATE {Sample $s^j_0\sim\rho$}
\STATE Sample trajectory from $s^j_0$: $(s^j_0, a^j_0,...,s^j_{T^j})$ following $\pi_{\theta_t}$
\STATE $B^j_t\leftarrow \frac{1}{1-\gamma+\gamma R}\sum_{a\in\mca} \nabla\pit(a|s^j_{T^j})Q_t(s^j_{T^j},a)$
\STATE  $x^j_0\leftarrow \arg\max_{s} V_t(s)$ 
\STATE Sample trajectory from $x^j_0$: $(x^j_0, b^j_0,...,x^j_{T^j})$ following $\pi_{\theta_t}$
\STATE $D^j_t\leftarrow \frac{1}{1-\gamma+\gamma R}\sum_{a\in\mca} \nabla\pit(a|x^j_{T^j})Q_t(x^j_{T^j},a)$
\STATE $g^j_t\leftarrow B^j_t+\frac{\gamma R}{1-\gamma}D^j_t$
\ENDFOR
\STATE $g_t\leftarrow \frac{\sum^M_{j=1}g^j_t }{M}$
\STATE $\theta_{t+1}\leftarrow \mathbf{\prod}_\Theta (\theta_t-\alpha_tg_t)$
\ENDFOR
\end{algorithmic}
\textbf{Output}: $\theta_T$
\end{algorithm}

We then smooth and specialize \cref{alg:act} to the tabular setting with direct policy parameterization (see \cref{alg:sac} in \cref{sec:ssrpg} for the details).
We derive the global optimality and convergence rate for smoothed robust actor-critic in the following theorem. 
%
In the algorithm, we set $T_c$ large enough so that $\|Q_{T_c}-Q^{\pi}_\sigma\|_{\infty}\leq \epe$, where for $\epe$ denotes the estimate error of robust value function. 
We note that 
the smoothed robust TD algorithm in \cref{alg:rtd} can be shown to converges to an $\epe$-global optimum with $\mathcal{O}(\epe^{-2})$ samples following similar methods as in \cite{wang2021online}.
\begin{theorem}\label{thm:tabularAC}
For the smoothed robust actor-critic algorithm under the tabular setting with direct policy parameterization, if we set
$
    \epe=\mathcal O(\epsilon^{2}),
    M=\mathcal{O}(\epsilon^{-2})$ and $
    T=\mathcal{O}(\epsilon^{-3}),
$
then,  
\begin{align}
      \min_{t\leq T}\mE[J(\theta_{t})-J^*]\leq  7\epsilon.
\end{align}
\end{theorem}
An explicit bound can be found in \eqref{eq:acbound} in the Appendix.
%
The sample complexity of Robust TD (Algorithm \ref{alg:rtd}) is $\mathcal{O}(\epe^{-2})$, then the robust TD requires $T_c=\mathcal{O}(\epsilon^{-4})$ samples. Hence the overall sample complexity of Algorithm \ref{alg:sac} is $\mathcal{O}(T(M+T_c))=\mathcal{O}(\epsilon^{-7})$ to find an  $\epsilon$-global optimum.  

\section{Numerical Results}
In this section, we demonstrate the convergence and robustness of our algorithms using numerical experiments. 
We test our algorithms on the Garnet problem \cite{archibald1995generation} and the Taxi environment from OpenAI \cite{brockman2016openai}. 
The Garnet problem can be specified by $\mathcal{G}(S_n,A_n)$, where the state space $\mcs$ has $S_n$ states $(s_1,...,s_{S_n})$ and action space has $A_n$ actions $(a_1,...,a_{A_n})$. The agent can take any actions in any state, but only gets reward $r=1$ if it takes $a_1$ in $s_1$ or takes $a_2$ in other states (it will receive 0 reward in other cases). The transition kernels are randomly generated.

\subsection{Robust v.s. Non-robust Policy Gradient}
We first compare our robust policy gradient method with vanilla policy gradient. To show the robustness of our algorithm over the vanilla policy gradient method, we compare their robust value functions, i.e., worst-case performance, for different values of $R$. 
We first train the robust policy gradient algorithm and store the obtained policy $\theta_t$ at each time step. At each time step,  we run robust TD in \cref{alg:rtd} with a sample size 200 for 30 times to estimate the average objective function value $J(\theta_t)$. We then plot $J(\theta_t)$ v.s. the number of iterations $t$ on the Garnet problems $\mathcal{G}(12,6)$ and $\mathcal{G}(20,10)$ in Figure \ref{Fig.rpg} and Figure \ref{Fig.g20}, respectively, and plot results on the Taxi environment from OpenAI in Figure \ref{Fig.taxi}. 
We do the same for the vanilla policy gradient method. 
The upper and lower envelopes of the curves correspond to the 95 and 5 percentiles of the 30 curves, respectively. 

As can be seen from \cref{Fig.rpg}, when $R=0$, the robust policy gradient method reduces to the non-robust vanilla policy gradient, and our results show both algorithms have the same performance. When $R>0$, the robust policy gradient obtains a policy that performs much better than the non-robust vanilla policy gradient, which demonstrates the robustness of our method. 
\begin{figure}[htbp]
\begin{center}
\subfigure[$R=0$.]{
\label{Fig.g1}
\includegraphics[width=0.47\linewidth]{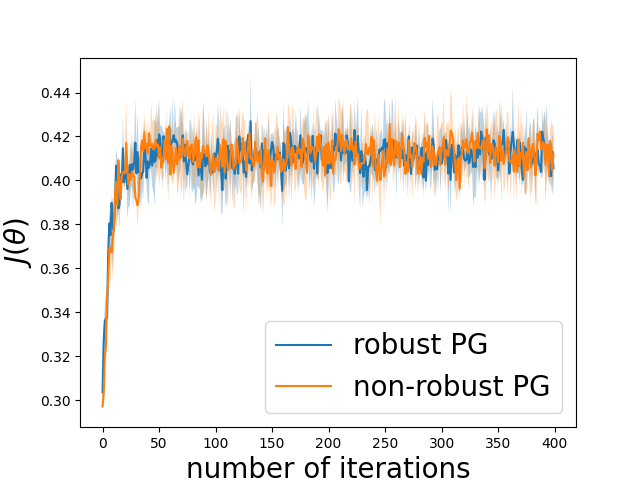}}
\subfigure[$R=0.1$.]{
\label{Fig.g2}
\includegraphics[width=0.47\linewidth]{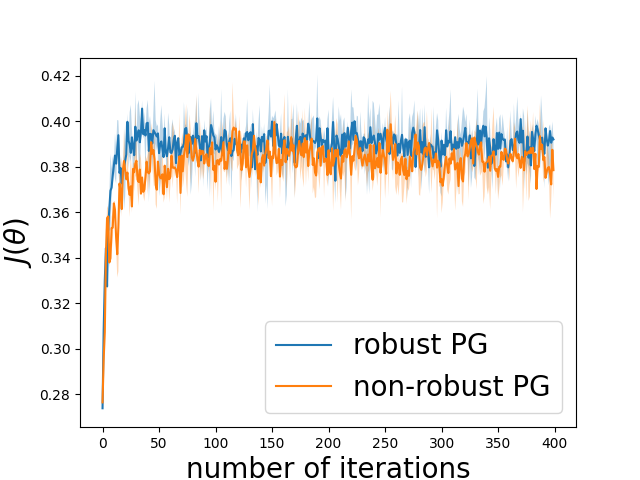}}
\subfigure[$R=0.15$.]{
\label{Fig.g3}
\includegraphics[width=0.47\linewidth]{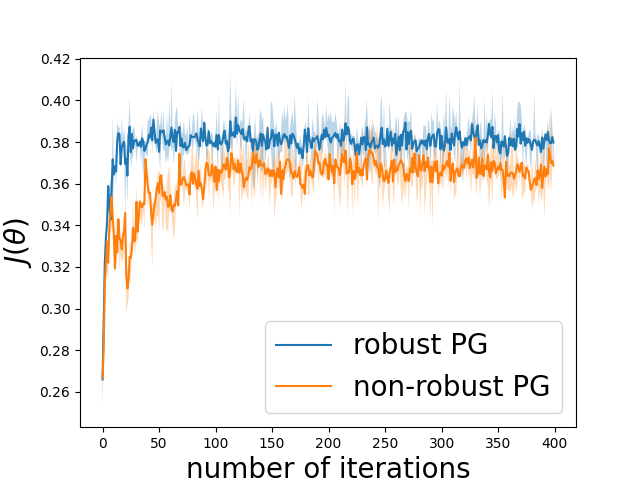}}
\subfigure[$R=0.25$.]{
\label{Fig.g4}
\includegraphics[width=0.47\linewidth]{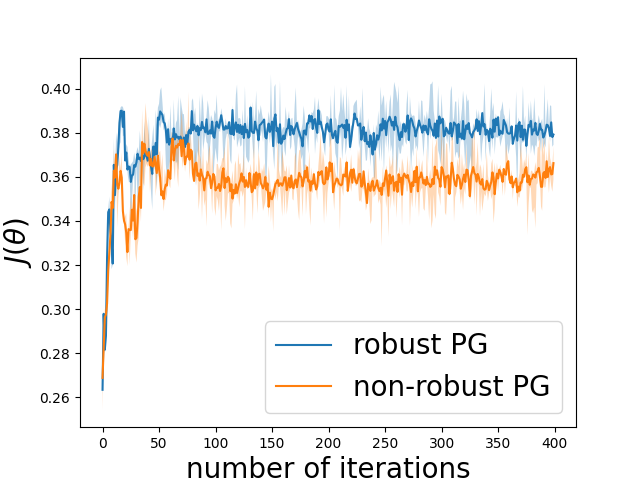}}
\caption{Robust Policy Gradient v.s. Non-robust Policy Gradient on Garnet Problem $\mathcal{G}(12,6)$.}
\label{Fig.rpg}
\end{center}
\vskip -0.2in
\end{figure}

\begin{figure}[ht]
\begin{center}
\subfigure[$R=0.2$.]{
\label{Fig.g201}
\includegraphics[width=0.47\linewidth]{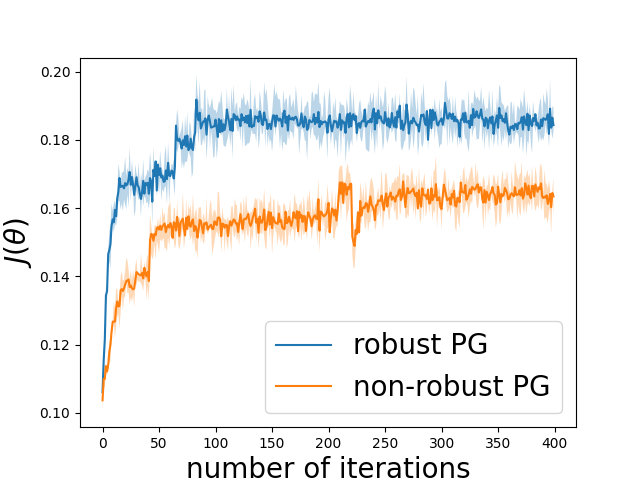}}
\subfigure[$R=0.25$.]{
\label{Fig.g202}
\includegraphics[width=0.47\linewidth]{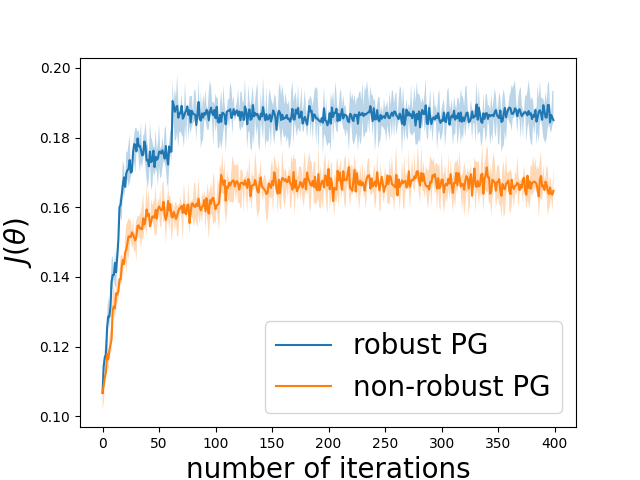}}
\caption{Robust Policy Gradient v.s. Non-robust Policy Gradient on Garnet Problem $\mathcal{G}(20,10)$.}
\label{Fig.g20}
\end{center}
\vskip -0.2in
\end{figure}

\begin{figure}[ht]
\begin{center}
\subfigure[$R=0.1$.]{
\label{Fig.tax1}
\includegraphics[width=0.47\linewidth]{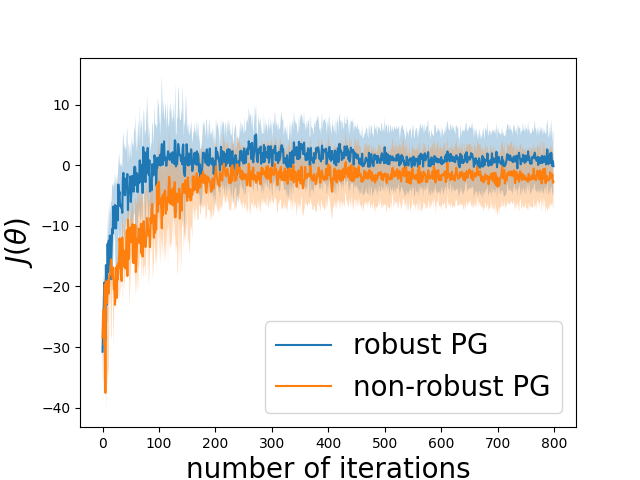}}
\subfigure[$R=0.15$.]{
\label{Fig.tax2}
\includegraphics[width=0.47\linewidth]{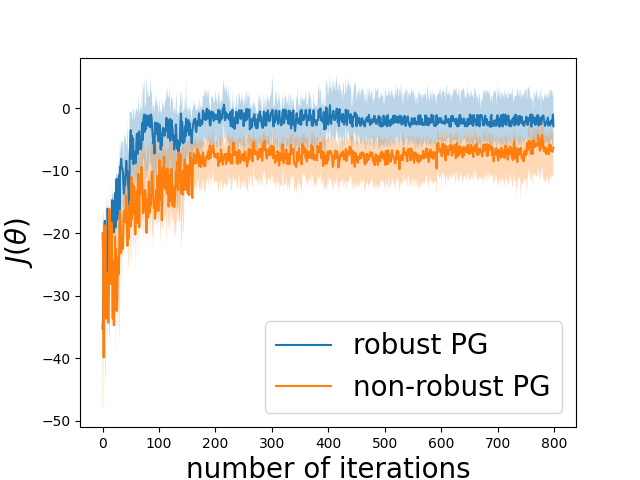}}
\caption{Robust Policy Gradient v.s. Non-robust Policy Gradient on Taxi Problem.}
\label{Fig.taxi}
\end{center}
\vskip -0.2in
\end{figure}

\subsection{Smoothed Robust Policy Gradient}
In this section, we demonstrate the performance of our smoothed robust policy gradient method on the Garnet problem $\mathcal{G}(12,6)$. As we showed in Section \ref{sec:srpg}, the smoothed algorithm approximate the robust policy gradient algorithm as $\sigma\to\infty$. Here, we set different values of $\sigma$ in Algorithm \ref{alg:srpgfull} and plot their objective functions v.s. number of iterations to demonstrate such an approximation. 

As shown in Figure \ref{Fig.smooth}, when $\sigma$ is small (e.g., $\sigma=1$), the performance of smoothed robust policy gradient is poor. As $\sigma$ increases, smoothed robust policy gradient behaves similarly to the robust policy gradient, which corresponds to the curve with $\sigma=\infty$. This experiment hence verifies our theoretical results that we can approximate the robust policy gradient by choosing a suitably large $\sigma$.
\begin{figure}[ht]
\begin{center}
\subfigure[$R=0.1$.]{
\label{Fig.sigma1}
\includegraphics[width=0.47\linewidth]{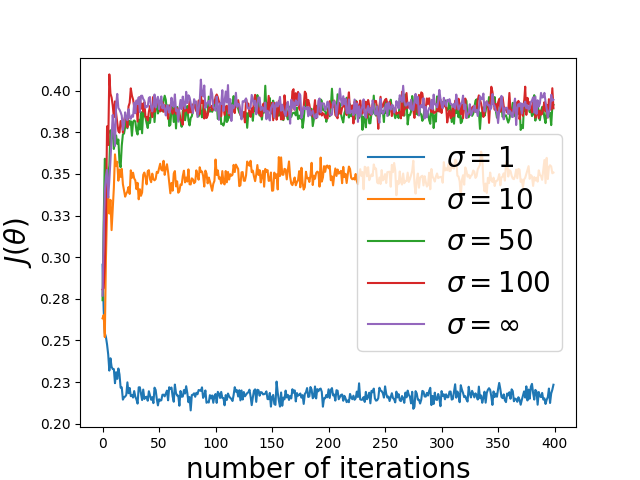}}
\subfigure[$R=0.15$.]{
\label{Fig.sigma2}
\includegraphics[width=0.47\linewidth]{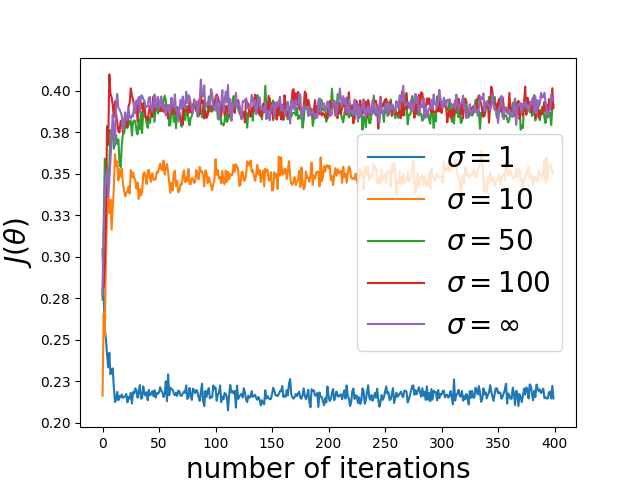}}
\caption{Smoothed Robust Policy Gradient.}
\label{Fig.smooth}
\end{center}
\vskip -0.2in
\end{figure}

\subsection{Robust Actor-Critic}
In \cref{Fig.nac}, we consider Garnet problem $\mathcal{G}(30,20)$ using neural network parameterized policy, where we use a two-layer neural network with 15 neurons in the hidden layer to parameterize the policy $\pit$. We then use a two-layer neural network (with 20 neurons in the hidden layer) in the critic.
At each time step, we run \cref{alg:NTD} for 30 times to estimate the robust value function. We then use the estimate to simulate \cref{alg:act}. We plot $J(\theta_t)$ v.s. the number of iterations in \cref{Fig.nac}, and the upper and lower envelopes of the curves correspond to the 95 and 5 percentiles of the 30 trajectories. As the results show, our robust actor-critic algorithm finds a policy that achieves a higher accumulated discounted reward on the worst-case transition kernel than the vanilla actor-critic algorithm \cite{sutton2018reinforcement}. 

\begin{figure}[ht]
\begin{center}
\subfigure[$R=0.15$.]{
\label{Fig.nac1}
\includegraphics[width=0.47\linewidth]{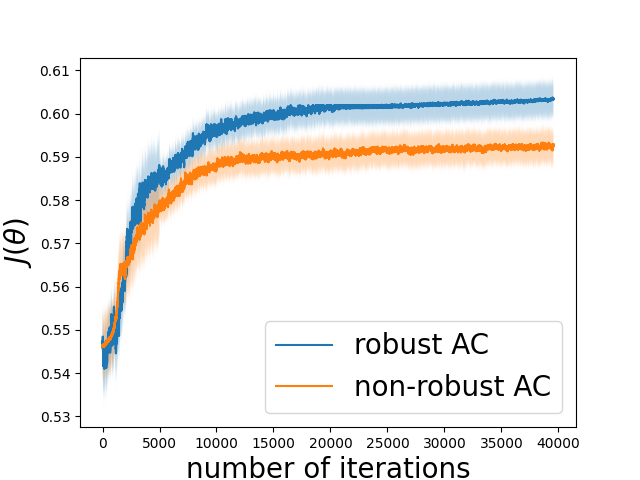}}
\subfigure[$R=0.2$.]{
\label{Fig.nac2}
\includegraphics[width=0.47\linewidth]{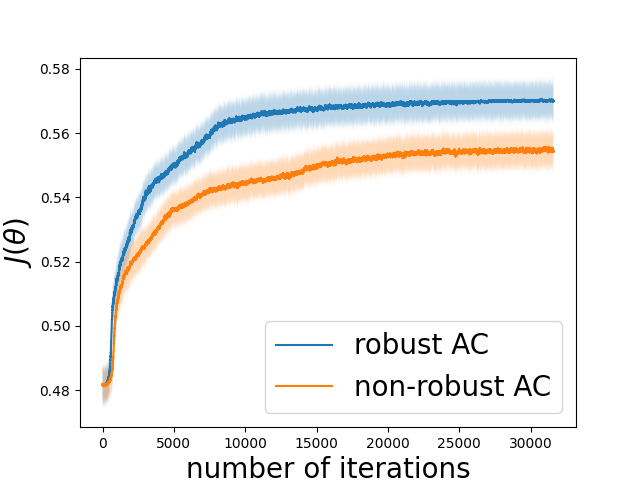}}
\caption{Robust Actor-Critic (AC) v.s. Non-robust Actor-Critic on Garnet Problem $\mathcal{G}(30,20)$.}
\label{Fig.nac}
\end{center}

\end{figure}

\subsection{Comparison with RARL}
We compare our robust algorithms with the robust adversarial reinforcement learning (RARL) approach in \citep{pinto2017robust}. 
The basic idea of the RARL approach is to introduce an adversary that perturbs the state transition to minimize the accumulated discounted reward. Then the agent and the adversary are trained alternatively using adversarial training.  
To apply their algorithm to our problem setting, we set an adversarial player, whose goal is to minimize the accumulated discounted reward that the agent receives. The action space $\mca_{ad}$ of the adversary is the state space $\mca_{ad}\triangleq \mcs$. The agent and the adversary take actions $a_{a}, a_{ad}$, then the environment will transit to state $a_{ad}$ with probability $R$ or transit following the unperturbed MDP $p_s^{a_a}$ with probability $1-R$. We compare our robust actor-critic algorithm with the RARL algorithm on the Taxi environment. Similarly, at each time step, we run Algorithm \ref{alg:NTD} with neural function approximation for 30 times to estimate the robust value function. We then use the results to simulate Algorithm \ref{alg:act} and RARL. We plot the robust value function $J(\theta_t)$ v.s. the number of iterations in Figure \ref{Fig.ra}. The upper and lower envelops correspond the 95 and 5 percentiles of the 30 trajectories.  As \cref{Fig.ra} shows, our robust 
actor-critic algorithm achieves a much higher accumulative discounted reward than the RARL approach under the worst-case transition kernel, and thus is more robust to model mismatch.
\begin{figure}[ht]
\begin{center}
\subfigure[$R=0.1$.]{
\label{Fig.ra1}
\includegraphics[width=0.47\linewidth]{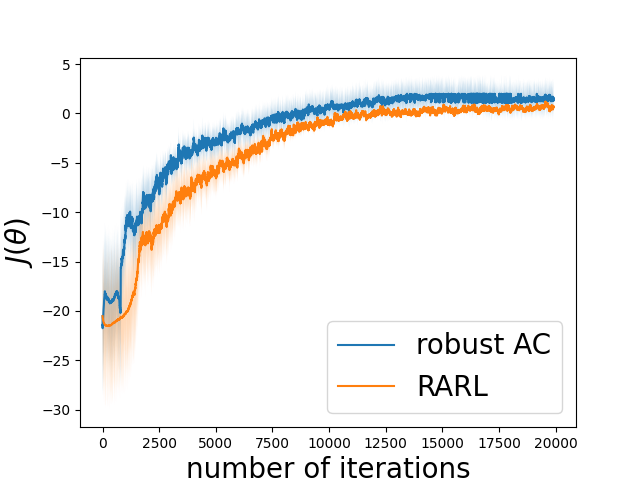}}
\subfigure[$R=0.15$.]{
\label{Fig.ra2}
\includegraphics[width=0.47\linewidth]{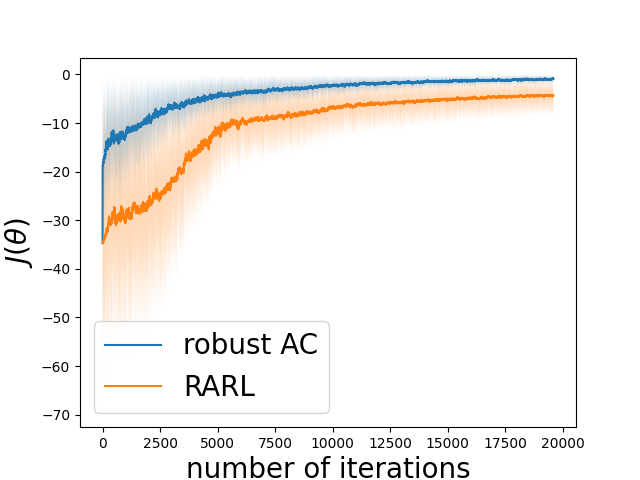}}
\caption{Robust Actor-Critic v.s. RARL on Taxi Environment.}
\label{Fig.ra}
\end{center}

\end{figure}
 
\section{Discussions}
In this paper, we develop direct policy search method for robust RL.  Our robust algorithms can be applied with arbitrary differentiable value function approximation and  policy parameterization, and thus is scalable to problems with large state and action spaces.
In this paper, the analysis is for the direct policy parameterization. Our approach can also be extended to establish the global optimality under other policy parameterizations, e.g., softmax. It is also of future interest to develop robust natural policy gradient approaches for robust RL.
We focus on the $R$-contamination model for the uncertainty set, which can be closely related to sets defined by total variation and Kullback-Leibler divergence using Pinsker’s inequality. It is also of future interest to investigate model-free approaches for other uncertainty sets defined via e.g., total variation, Kullback-Leibler divergence, Wasserstain distance. 

\newpage
\bibliographystyle{icml2022}
\bibliography{RPG.bib}

\begin{thebibliography}{93}
\providecommand{\natexlab}[1]{#1}
\providecommand{\url}[1]{\texttt{#1}}
\expandafter\ifx\csname urlstyle\endcsname\relax
  \providecommand{\doi}[1]{doi: #1}\else
  \providecommand{\doi}{doi: \begingroup \urlstyle{rm}\Url}\fi

\bibitem[Abdullah et~al.(2019)Abdullah, Ren, Ammar, Milenkovic, Luo, Zhang, and
  Wang]{abdullah2019wasserstein}
Abdullah, M.~A., Ren, H., Ammar, H.~B., Milenkovic, V., Luo, R., Zhang, M., and
  Wang, J.
\newblock Wasserstein robust reinforcement learning.
\newblock \emph{arXiv preprint arXiv:1907.13196}, 2019.

\bibitem[Achiam et~al.(2017)Achiam, Held, Tamar, and
  Abbeel]{achiam2017constrained}
Achiam, J., Held, D., Tamar, A., and Abbeel, P.
\newblock Constrained policy optimization.
\newblock In \emph{Proc. International Conference on Machine Learning (ICML)},
  pp.\  22--31. PMLR, 2017.

\bibitem[Agarwal et~al.(2021)Agarwal, Kakade, Lee, and
  Mahajan]{agarwal2021theory}
Agarwal, A., Kakade, S.~M., Lee, J.~D., and Mahajan, G.
\newblock On the theory of policy gradient methods: Optimality, approximation,
  and distribution shift.
\newblock \emph{Journal of Machine Learning Research}, 22\penalty0
  (98):\penalty0 1--76, 2021.

\bibitem[Archibald et~al.(1995)Archibald, McKinnon, and
  Thomas]{archibald1995generation}
Archibald, T., McKinnon, K., and Thomas, L.
\newblock {On the generation of {M}arkov decision processes}.
\newblock \emph{Journal of the Operational Research Society}, 46\penalty0
  (3):\penalty0 354--361, 1995.

\bibitem[Asadi \& Littman(2017)Asadi and Littman]{Asadi2016}
Asadi, K. and Littman, M.~L.
\newblock An alternative softmax operator for reinforcement learning.
\newblock In \emph{Proc. International Conference on Machine Learning (ICML)},
  volume~70, pp.\  243--252. JMLR, 2017.

\bibitem[Atkeson \& Morimoto(2003)Atkeson and
  Morimoto]{atkeson2003nonparametric}
Atkeson, C.~G. and Morimoto, J.
\newblock Nonparametric representation of policies and value functions: A
  trajectory-based approach.
\newblock In \emph{Proc. Advances in Neural Information Processing Systems
  (NIPS)}, pp.\  1643--1650, 2003.

\bibitem[Badrinath \& Kalathil(2021)Badrinath and
  Kalathil]{badrinath2021robust}
Badrinath, K.~P. and Kalathil, D.
\newblock Robust reinforcement learning using least squares policy iteration
  with provable performance guarantees.
\newblock In \emph{Proc. International Conference on Machine Learning (ICML)},
  pp.\  511--520. PMLR, 2021.

\bibitem[Bagnell et~al.(2001)Bagnell, Ng, and Schneider]{bagnell2001solving}
Bagnell, J.~A., Ng, A.~Y., and Schneider, J.~G.
\newblock Solving uncertain {M}arkov decision processes.
\newblock 09 2001.

\bibitem[Beck(2017)]{beck2017first}
Beck, A.
\newblock \emph{First-order methods in optimization}.
\newblock SIAM, 2017.

\bibitem[Bhandari \& Russo(2019)Bhandari and Russo]{bhandari2019global}
Bhandari, J. and Russo, D.
\newblock Global optimality guarantees for policy gradient methods.
\newblock \emph{arXiv preprint arXiv:1906.01786}, 2019.

\bibitem[Bhandari \& Russo(2021)Bhandari and Russo]{bhandari2021linear}
Bhandari, J. and Russo, D.
\newblock On the linear convergence of policy gradient methods for finite mdps.
\newblock In \emph{Proc. International Conference on Artifical Intelligence and
  Statistics (AISTATS)}, pp.\  2386--2394. PMLR, 2021.

\bibitem[Bhandari et~al.(2018)Bhandari, Russo, and Singal]{bhandari2018finite}
Bhandari, J., Russo, D., and Singal, R.
\newblock A finite time analysis of temporal difference learning with linear
  function approximation.
\newblock In \emph{Proc. Annual Conference on Learning Theory (CoLT)}, pp.\
  1691--1692. PMLR, 2018.

\bibitem[Bolte et~al.(2007{\natexlab{a}})Bolte, Daniilidis, and
  Lewis]{bolte2007lojasiewicz}
Bolte, J., Daniilidis, A., and Lewis, A.
\newblock The {\l}ojasiewicz inequality for nonsmooth subanalytic functions
  with applications to subgradient dynamical systems.
\newblock \emph{SIAM Journal on Optimization}, 17\penalty0 (4):\penalty0
  1205--1223, 2007{\natexlab{a}}.

\bibitem[Bolte et~al.(2007{\natexlab{b}})Bolte, Daniilidis, Lewis, and
  Shiota]{bolte2007clarke}
Bolte, J., Daniilidis, A., Lewis, A., and Shiota, M.
\newblock Clarke subgradients of stratifiable functions.
\newblock \emph{SIAM Journal on Optimization}, 18\penalty0 (2):\penalty0
  556--572, 2007{\natexlab{b}}.

\bibitem[Borkar(2009)]{borkar2009stochastic}
Borkar, V.~S.
\newblock \emph{Stochastic approximation: a dynamical systems viewpoint},
  volume~48.
\newblock Springer, 2009.

\bibitem[Borkar \& Meyn(2000)Borkar and Meyn]{borkar2000ode}
Borkar, V.~S. and Meyn, S.~P.
\newblock The ode method for convergence of stochastic approximation and
  reinforcement learning.
\newblock \emph{SIAM Journal on Control and Optimization}, 38\penalty0
  (2):\penalty0 447--469, 2000.

\bibitem[Brockman et~al.(2016)Brockman, Cheung, Pettersson, Schneider,
  Schulman, Tang, and Zaremba]{brockman2016openai}
Brockman, G., Cheung, V., Pettersson, L., Schneider, J., Schulman, J., Tang,
  J., and Zaremba, W.
\newblock {OpenAI Gym}.
\newblock \emph{arXiv preprint arXiv:1606.01540}, 2016.

\bibitem[Cai et~al.(2019)Cai, Yang, Lee, and Wang]{cai2019neural}
Cai, Q., Yang, Z., Lee, J.~D., and Wang, Z.
\newblock Neural temporal-difference learning converges to global optima.
\newblock In \emph{Proc. Advances in Neural Information Processing Systems
  (NeurIPS)}, pp.\  11312--11322, 2019.

\bibitem[Cen et~al.(2021)Cen, Cheng, Chen, Wei, and Chi]{cen2021fast}
Cen, S., Cheng, C., Chen, Y., Wei, Y., and Chi, Y.
\newblock Fast global convergence of natural policy gradient methods with
  entropy regularization.
\newblock \emph{Operations Research}, 2021.

\bibitem[Clarke(1990)]{clarke1990optimization}
Clarke, F.~H.
\newblock \emph{Optimization and nonsmooth analysis}.
\newblock SIAM, 1990.

\bibitem[Dalal et~al.(2018)Dalal, Sz{\"o}r{\'e}nyi, Thoppe, and
  Mannor]{Dalal2018a}
Dalal, G., Sz{\"o}r{\'e}nyi, B., Thoppe, G., and Mannor, S.
\newblock Finite sample analyses for {TD}(0) with function approximation.
\newblock In \emph{Proc. AAAI Conference on Artificial Intelligence (AAAI)},
  pp.\  6144--6160, 2018.

\bibitem[Derman et~al.(2021)Derman, Geist, and Mannor]{derman2021twice}
Derman, E., Geist, M., and Mannor, S.
\newblock Twice regularized {MDP}s and the equivalence between robustness and
  regularization.
\newblock In \emph{Proc. Advances in Neural Information Processing Systems
  (NeurIPS)}, 2021.

\bibitem[Du et~al.(2019)Du, Lee, Li, Wang, and Zhai]{du2019gradient}
Du, S., Lee, J., Li, H., Wang, L., and Zhai, X.
\newblock Gradient descent finds global minima of deep neural networks.
\newblock In \emph{Proc. International Conference on Machine Learning (ICML)},
  pp.\  1675--1685. PMLR, 2019.

\bibitem[Du et~al.(2018)Du, Wang, Balakrishnan, Ravikumar, and
  Singh]{du2018robust}
Du, S.~S., Wang, Y., Balakrishnan, S., Ravikumar, P., and Singh, A.
\newblock Robust nonparametric regression under huber's
  $\epsilon$-contamination model.
\newblock \emph{arXiv preprint arXiv:1805.10406}, 2018.

\bibitem[Eysenbach \& Levine(2021)Eysenbach and Levine]{eysenbach2021maximum}
Eysenbach, B. and Levine, S.
\newblock Maximum entropy {RL} (provably) solves some robust {RL} problems.
\newblock \emph{arXiv preprint arXiv:2103.06257}, 2021.

\bibitem[Federer(2014)]{federer2014geometric}
Federer, H.
\newblock \emph{Geometric measure theory}.
\newblock Springer, 2014.

\bibitem[Ghadimi \& Lan(2016)Ghadimi and Lan]{ghadimi2016accelerated}
Ghadimi, S. and Lan, G.
\newblock Accelerated gradient methods for nonconvex nonlinear and stochastic
  programming.
\newblock \emph{Mathematical Programming}, 156\penalty0 (1-2):\penalty0 59--99,
  2016.

\bibitem[Ghadimi et~al.(2016)Ghadimi, Lan, and Zhang]{ghadimi2016mini}
Ghadimi, S., Lan, G., and Zhang, H.
\newblock Mini-batch stochastic approximation methods for nonconvex stochastic
  composite optimization.
\newblock \emph{Mathematical Programming}, 155\penalty0 (1-2):\penalty0
  267--305, 2016.

\bibitem[Goyal \& Grand-Clement(2018)Goyal and Grand-Clement]{goyal2018robust}
Goyal, V. and Grand-Clement, J.
\newblock Robust markov decision process: Beyond rectangularity.
\newblock \emph{arXiv preprint arXiv:1811.00215}, 2018.

\bibitem[Ho et~al.(2018)Ho, Petrik, and Wiesemann]{ho2018fast}
Ho, C.~P., Petrik, M., and Wiesemann, W.
\newblock Fast bellman updates for robust mdps.
\newblock In \emph{International Conference on Machine Learning}, pp.\
  1979--1988. PMLR, 2018.

\bibitem[Ho et~al.(2021)Ho, Petrik, and Wiesemann]{ho2021partial}
Ho, C.~P., Petrik, M., and Wiesemann, W.
\newblock Partial policy iteration for l1-robust markov decision processes.
\newblock \emph{Journal of Machine Learning Research}, 22\penalty0
  (275):\penalty0 1--46, 2021.

\bibitem[Hou et~al.(2020)Hou, Pang, Hong, Lan, Ma, and Yin]{hou2020robust}
Hou, L., Pang, L., Hong, X., Lan, Y., Ma, Z., and Yin, D.
\newblock Robust reinforcement learning with wasserstein constraint.
\newblock \emph{arXiv preprint arXiv:2006.00945}, 2020.

\bibitem[Huang et~al.(2017)Huang, Papernot, Goodfellow, Duan, and
  Abbeel]{huang2017adversarial}
Huang, S., Papernot, N., Goodfellow, I., Duan, Y., and Abbeel, P.
\newblock Adversarial attacks on neural network policies.
\newblock In \emph{Proc. International Conference on Learning Representations
  (ICLR)}, 2017.

\bibitem[Huber \& Ronchetti(2009)Huber and Ronchetti]{huber2009robust}
Huber, P. and Ronchetti, E.
\newblock \emph{Robust Statistics}.
\newblock John Wiley \& Sons, Inc, 2009.

\bibitem[Huber(1965)]{hub65}
Huber, P.~J.
\newblock A robust version of the probability ratio test.
\newblock \emph{Ann. Math. Statist.}, 36:\penalty0 1753--1758, 1965.

\bibitem[Iyengar(2005)]{iyengar2005robust}
Iyengar, G.~N.
\newblock Robust dynamic programming.
\newblock \emph{Mathematics of Operations Research}, 30\penalty0 (2):\penalty0
  257--280, 2005.

\bibitem[Kakade(2001)]{kakade2001natural}
Kakade, S.~M.
\newblock A natural policy gradient.
\newblock In \emph{Proc. Advances in Neural Information Processing Systems
  (NIPS)}, volume~14, pp.\  1531--1538, 2001.

\bibitem[Karimi et~al.(2016)Karimi, Nutini, and Schmidt]{karimi2016linear}
Karimi, H., Nutini, J., and Schmidt, M.
\newblock Linear convergence of gradient and proximal-gradient methods under
  the {P}olyak-{\l}ojasiewicz condition.
\newblock In \emph{Joint European Conference on Machine Learning and Knowledge
  Discovery in Databases}, pp.\  795--811. Springer, 2016.

\bibitem[Kaufman \& Schaefer(2013)Kaufman and Schaefer]{kaufman2013robust}
Kaufman, D.~L. and Schaefer, A.~J.
\newblock Robust modified policy iteration.
\newblock \emph{INFORMS Journal on Computing}, 25\penalty0 (3):\penalty0
  396--410, 2013.

\bibitem[Konda \& Tsitsiklis(2000)Konda and Tsitsiklis]{konda2000actor}
Konda, V.~R. and Tsitsiklis, J.~N.
\newblock Actor-critic algorithms.
\newblock In \emph{Proc. Advances in Neural Information Processing Systems
  (NIPS)}, pp.\  1008--1014, 2000.

\bibitem[Korda \& La(2015)Korda and La]{korda2015td}
Korda, N. and La, P.
\newblock On {TD}(0) with function approximation: Concentration bounds and a
  centered variant with exponential convergence.
\newblock In \emph{Proc. International Conference on Machine Learning (ICML)},
  pp.\  626--634, 2015.

\bibitem[Kos \& Song(2017)Kos and Song]{kos2017delving}
Kos, J. and Song, D.
\newblock Delving into adversarial attacks on deep policies.
\newblock In \emph{Proc. International Conference on Learning Representations
  (ICLR)}, 2017.

\bibitem[Kruger(2003)]{kruger2003frechet}
Kruger, A.~Y.
\newblock On {F}r{\'e}chet subdifferentials.
\newblock \emph{Journal of Mathematical Sciences}, 116\penalty0 (3):\penalty0
  3325--3358, 2003.

\bibitem[Laroche \& des Combes(2021)Laroche and des Combes]{laroche2021dr}
Laroche, R. and des Combes, R.~T.
\newblock Dr jekyll \& mr hyde: the strange case of off-policy policy updates.
\newblock In \emph{Proc. Advances in Neural Information Processing Systems
  (NeurIPS)}, 2021.

\bibitem[Li et~al.(2021)Li, Wei, Chi, Gu, and Chen]{li2021softmax}
Li, G., Wei, Y., Chi, Y., Gu, Y., and Chen, Y.
\newblock Softmax policy gradient methods can take exponential time to
  converge.
\newblock \emph{arXiv preprint arXiv:2102.11270}, 2021.

\bibitem[Lim \& Autef(2019)Lim and Autef]{lim2019kernel}
Lim, S.~H. and Autef, A.
\newblock Kernel-based reinforcement learning in robust markov decision
  processes.
\newblock In \emph{Proc. International Conference on Machine Learning (ICML)},
  pp.\  3973--3981. PMLR, 2019.

\bibitem[Lim et~al.(2013)Lim, Xu, and Mannor]{lim2013reinforcement}
Lim, S.~H., Xu, H., and Mannor, S.
\newblock Reinforcement learning in robust {M}arkov decision processes.
\newblock In \emph{Proc. Advances in Neural Information Processing Systems
  (NIPS)}, pp.\  701--709, 2013.

\bibitem[Lin(2022)]{xiao2022On}
Lin, X.
\newblock On the convergence rates of policy gradient methods.
\newblock \emph{arXiv preprint arXiv:2201.07443}, 2022.

\bibitem[Lin et~al.(2017)Lin, Hong, Liao, Shih, Liu, and Sun]{lin2017tactics}
Lin, Y.-C., Hong, Z.-W., Liao, Y.-H., Shih, M.-L., Liu, M.-Y., and Sun, M.
\newblock Tactics of adversarial attack on deep reinforcement learning agents.
\newblock In \emph{Proc. International Joint Conferences on Artificial
  Intelligence (IJCAI)}, pp.\  3756--3762, 2017.

\bibitem[Lojasiewicz(1963)]{lojasiewicz1963topological}
Lojasiewicz, S.
\newblock A topological property of real analytic subsets.
\newblock \emph{Coll. du CNRS, Les {\'e}quations aux d{\'e}riv{\'e}es
  partielles}, 117\penalty0 (87-89):\penalty0 2, 1963.

\bibitem[Majewski et~al.(2018)Majewski, Miasojedow, and
  Moulines]{majewski2018analysis}
Majewski, S., Miasojedow, B., and Moulines, E.
\newblock Analysis of nonsmooth stochastic approximation: the differential
  inclusion approach.
\newblock \emph{arXiv preprint arXiv:1805.01916}, 2018.

\bibitem[Mandlekar et~al.(2017)Mandlekar, Zhu, Garg, Fei-Fei, and
  Savarese]{mandlekar2017adversarially}
Mandlekar, A., Zhu, Y., Garg, A., Fei-Fei, L., and Savarese, S.
\newblock Adversarially robust policy learning: Active construction of
  physically-plausible perturbations.
\newblock In \emph{2017 IEEE/RSJ International Conference on Intelligent Robots
  and Systems (IROS)}, pp.\  3932--3939. IEEE, 2017.

\bibitem[Mei et~al.(2020)Mei, Xiao, Szepesvari, and Schuurmans]{mei2020global}
Mei, J., Xiao, C., Szepesvari, C., and Schuurmans, D.
\newblock On the global convergence rates of softmax policy gradient methods.
\newblock In \emph{Proc. International Conference on Machine Learning (ICML)},
  pp.\  6820--6829. PMLR, 2020.

\bibitem[Miyato et~al.(2018)Miyato, Kataoka, Koyama, and
  Yoshida]{miyato2018spectral}
Miyato, T., Kataoka, T., Koyama, M., and Yoshida, Y.
\newblock Spectral normalization for generative adversarial networks.
\newblock In \emph{Proc. International Conference on Learning Representations
  (ICLR)}, 2018.

\bibitem[Morimoto \& Doya(2005)Morimoto and Doya]{morimoto2005robust}
Morimoto, J. and Doya, K.
\newblock Robust reinforcement learning.
\newblock \emph{Neural computation}, 17\penalty0 (2):\penalty0 335--359, 2005.

\bibitem[Neyshabur(2017)]{neyshabur2017implicit}
Neyshabur, B.
\newblock Implicit regularization in deep learning.
\newblock \emph{arXiv preprint arXiv:1709.01953}, 2017.

\bibitem[Nilim \& El~Ghaoui(2004)Nilim and El~Ghaoui]{nilim2004robustness}
Nilim, A. and El~Ghaoui, L.
\newblock Robustness in {Markov} decision problems with uncertain transition
  matrices.
\newblock In \emph{Proc. Advances in Neural Information Processing Systems
  (NIPS)}, pp.\  839--846, 2004.

\bibitem[Nishimura \& Ozaki(2004)Nishimura and Ozaki]{nishimura2004search}
Nishimura, K.~G. and Ozaki, H.
\newblock Search and knightian uncertainty.
\newblock \emph{Journal of Economic Theory}, 119\penalty0 (2):\penalty0
  299--333, 2004.

\bibitem[Nishimura \& Ozaki(2006)Nishimura and Ozaki]{Kiyohiko2006}
Nishimura, K.~G. and Ozaki, H.
\newblock An axiomatic approach to $\epsilon$-contamination.
\newblock \emph{Economic Theory}, 27\penalty0 (2):\penalty0 333--340, 2006.

\bibitem[Panaganti \& Kalathil(2021)Panaganti and
  Kalathil]{panaganti2021sample}
Panaganti, K. and Kalathil, D.
\newblock Sample complexity of robust reinforcement learning with a generative
  model.
\newblock \emph{arXiv preprint arXiv:2112.01506}, 2021.

\bibitem[Pattanaik et~al.(2018)Pattanaik, Tang, Liu, Bommannan, and
  Chowdhary]{pattanaik2018robust}
Pattanaik, A., Tang, Z., Liu, S., Bommannan, G., and Chowdhary, G.
\newblock Robust deep reinforcement learning with adversarial attacks.
\newblock In \emph{Proc. International Conference on Autonomous Agents and
  MultiAgent Systems}, pp.\  2040--2042, 2018.

\bibitem[Pinto et~al.(2017)Pinto, Davidson, Sukthankar, and
  Gupta]{pinto2017robust}
Pinto, L., Davidson, J., Sukthankar, R., and Gupta, A.
\newblock Robust adversarial reinforcement learning.
\newblock In \emph{Proc. International Conference on Machine Learning (ICML)},
  pp.\  2817--2826. PMLR, 2017.

\bibitem[Polyak(1963)]{polyak1963gradient}
Polyak, B.~T.
\newblock Gradient methods for minimizing functionals.
\newblock \emph{Zhurnal vychislitel'noi matematiki i matematicheskoi fiziki},
  3\penalty0 (4):\penalty0 643--653, 1963.

\bibitem[Prasad et~al.(2020{\natexlab{a}})Prasad, Srinivasan, Balakrishnan, and
  Ravikumar]{prasad2020learning}
Prasad, A., Srinivasan, V., Balakrishnan, S., and Ravikumar, P.
\newblock On learning ising models under huber's contamination model.
\newblock \emph{Proc. Advances in Neural Information Processing Systems
  (NeurIPS)}, 33, 2020{\natexlab{a}}.

\bibitem[Prasad et~al.(2020{\natexlab{b}})Prasad, Suggala, Balakrishnan, and
  Ravikumar]{prasad2020robust}
Prasad, A., Suggala, A.~S., Balakrishnan, S., and Ravikumar, P.
\newblock Robust estimation via robust gradient estimation.
\newblock \emph{Journal of the Royal Statistical Society: Series B (Statistical
  Methodology)}, 82\penalty0 (3):\penalty0 601--627, 2020{\natexlab{b}}.

\bibitem[Puterman(2014)]{puterman2014markov}
Puterman, M.~L.
\newblock \emph{Markov decision processes: discrete stochastic dynamic
  programming}.
\newblock John Wiley \& Sons, 2014.

\bibitem[Rajeswaran et~al.(2017)Rajeswaran, Ghotra, Ravindran, and
  Levine]{rajeswaran2017epopt}
Rajeswaran, A., Ghotra, S., Ravindran, B., and Levine, S.
\newblock Epopt: Learning robust neural network policies using model ensembles.
\newblock In \emph{Proc. International Conference on Learning Representations
  (ICLR)}, 2017.

\bibitem[Roy et~al.(2017)Roy, Xu, and Pokutta]{roy2017reinforcement}
Roy, A., Xu, H., and Pokutta, S.
\newblock Reinforcement learning under model mismatch.
\newblock In \emph{Proc. Advances in Neural Information Processing Systems
  (NIPS)}, pp.\  3046--3055, 2017.

\bibitem[Russel et~al.(2020)Russel, Benosman, and Van~Baar]{russel2020robust}
Russel, R.~H., Benosman, M., and Van~Baar, J.
\newblock Robust constrained-{MDP}s: Soft-constrained robust policy
  optimization under model uncertainty.
\newblock \emph{arXiv preprint arXiv:2010.04870}, 2020.

\bibitem[Ruszczy{\'n}ski(2020)]{ruszczynski2020convergence}
Ruszczy{\'n}ski, A.
\newblock Convergence of a stochastic subgradient method with averaging for
  nonsmooth nonconvex constrained optimization.
\newblock \emph{Optimization Letters}, pp.\  1--11, 2020.

\bibitem[Satia \& Lave~Jr(1973)Satia and Lave~Jr]{satia1973markovian}
Satia, J.~K. and Lave~Jr, R.~E.
\newblock {M}arkovian decision processes with uncertain transition
  probabilities.
\newblock \emph{Operations Research}, 21\penalty0 (3):\penalty0 728--740, 1973.

\bibitem[Schulman et~al.(2015)Schulman, Levine, Abbeel, Jordan, and
  Moritz]{schulman2015trust}
Schulman, J., Levine, S., Abbeel, P., Jordan, M., and Moritz, P.
\newblock Trust region policy optimization.
\newblock In \emph{Proc. International Conference on Machine Learning (ICML)},
  pp.\  1889--1897. PMLR, 2015.

\bibitem[Schulman et~al.(2017)Schulman, Wolski, Dhariwal, Radford, and
  Klimov]{schulman2017proximal}
Schulman, J., Wolski, F., Dhariwal, P., Radford, A., and Klimov, O.
\newblock Proximal policy optimization algorithms.
\newblock \emph{arXiv preprint arXiv:1707.06347}, 2017.

\bibitem[Si et~al.(2020)Si, Zhang, Zhou, and Blanchet]{si2020distributionally}
Si, N., Zhang, F., Zhou, Z., and Blanchet, J.
\newblock Distributionally robust policy evaluation and learning in offline
  contextual bandits.
\newblock In \emph{Proc. International Conference on Machine Learning (ICML)},
  pp.\  8884--8894. PMLR, 2020.

\bibitem[Sutton \& Barto(2018)Sutton and Barto]{sutton2018reinforcement}
Sutton, R.~S. and Barto, A.~G.
\newblock \emph{Reinforcement Learning: An Introduction}.
\newblock The MIT Press, Cambridge, Massachusetts, 2018.

\bibitem[Sutton et~al.(1999)Sutton, McAllester, Singh, Mansour,
  et~al.]{sutton1999policy}
Sutton, R.~S., McAllester, D.~A., Singh, S.~P., Mansour, Y., et~al.
\newblock Policy gradient methods for reinforcement learning with function
  approximation.
\newblock In \emph{Proc. Advances in Neural Information Processing Systems
  (NIPS)}, volume~99, pp.\  1057--1063. Citeseer, 1999.

\bibitem[Tamar et~al.(2014)Tamar, Mannor, and Xu]{tamar2014scaling}
Tamar, A., Mannor, S., and Xu, H.
\newblock Scaling up robust mdps using function approximation.
\newblock In \emph{Proc. International Conference on Machine Learning (ICML)},
  pp.\  181--189. PMLR, 2014.

\bibitem[Tessler et~al.(2019)Tessler, Efroni, and Mannor]{tessler2019action}
Tessler, C., Efroni, Y., and Mannor, S.
\newblock Action robust reinforcement learning and applications in continuous
  control.
\newblock In \emph{International Conference on Machine Learning}, pp.\
  6215--6224. PMLR, 2019.

\bibitem[Touati et~al.(2020)Touati, Zhang, Pineau, and
  Vincent]{touati2020stable}
Touati, A., Zhang, A., Pineau, J., and Vincent, P.
\newblock Stable policy optimization via off-policy divergence regularization.
\newblock In \emph{Proc. International Conference on Uncertainty in Artificial
  Intelligence (UAI)}, pp.\  1328--1337. PMLR, 2020.

\bibitem[Vinitsky et~al.(2020)Vinitsky, Du, Parvate, Jang, Abbeel, and
  Bayen]{vinitsky2020robust}
Vinitsky, E., Du, Y., Parvate, K., Jang, K., Abbeel, P., and Bayen, A.
\newblock Robust reinforcement learning using adversarial populations.
\newblock \emph{arXiv preprint arXiv:2008.01825}, 2020.

\bibitem[Wang \& Zou(2020)Wang and Zou]{wang2020finite}
Wang, Y. and Zou, S.
\newblock Finite-sample analysis of {Greedy-GQ} with linear function
  approximation under {M}arkovian noise.
\newblock In \emph{Proc. International Conference on Uncertainty in Artificial
  Intelligence (UAI)}, pp.\  11--20. PMLR, 2020.

\bibitem[Wang \& Zou(2021)Wang and Zou]{wang2021online}
Wang, Y. and Zou, S.
\newblock Online robust reinforcement learning with model uncertainty.
\newblock In \emph{Proc. Advances in Neural Information Processing Systems
  (NeurIPS)}, 2021.

\bibitem[Wiesemann et~al.(2013)Wiesemann, Kuhn, and
  Rustem]{wiesemann2013robust}
Wiesemann, W., Kuhn, D., and Rustem, B.
\newblock Robust {M}arkov decision processes.
\newblock \emph{Mathematics of Operations Research}, 38\penalty0 (1):\penalty0
  153--183, 2013.

\bibitem[Williams(1992)]{williams1992simple}
Williams, R.~J.
\newblock Simple statistical gradient-following algorithms for connectionist
  reinforcement learning.
\newblock \emph{Machine learning}, 8\penalty0 (3):\penalty0 229--256, 1992.

\bibitem[Xu \& Mannor(2010)Xu and Mannor]{xu2010distributionally}
Xu, H. and Mannor, S.
\newblock Distributionally robust markov decision processes.
\newblock In \emph{Proc. Advances in Neural Information Processing Systems
  (NIPS)}, pp.\  2505--2513, 2010.

\bibitem[Yang et~al.(2021)Yang, Zhang, and Zhang]{yang2021towards}
Yang, W., Zhang, L., and Zhang, Z.
\newblock Towards theoretical understandings of robust {M}arkov decision
  processes: Sample complexity and asymptotics.
\newblock \emph{arXiv preprint arXiv:2105.03863}, 2021.

\bibitem[Yu \& Xu(2015)Yu and Xu]{yu2015distributionally}
Yu, P. and Xu, H.
\newblock Distributionally robust counterpart in markov decision processes.
\newblock \emph{IEEE Transactions on Automatic Control}, 61\penalty0
  (9):\penalty0 2538--2543, 2015.

\bibitem[Zhang et~al.(2020{\natexlab{a}})Zhang, Koppel, Bedi, Szepesvari, and
  Wang]{zhang2020variational}
Zhang, J., Koppel, A., Bedi, A.~S., Szepesvari, C., and Wang, M.
\newblock Variational policy gradient method for reinforcement learning with
  general utilities.
\newblock In \emph{Proc. Advances in Neural Information Processing Systems
  (NeurIPS)}, volume~33, pp.\  4572--4583, 2020{\natexlab{a}}.

\bibitem[Zhang et~al.(2020{\natexlab{b}})Zhang, Lin, Jegelka, Jadbabaie, and
  Sra]{zhang2020complexity}
Zhang, J., Lin, H., Jegelka, S., Jadbabaie, A., and Sra, S.
\newblock Complexity of finding stationary points of nonsmooth nonconvex
  functions.
\newblock \emph{arXiv preprint arXiv:2002.04130}, 2020{\natexlab{b}}.

\bibitem[Zhang et~al.(2021{\natexlab{a}})Zhang, Tachet, and
  Laroche]{zhang2021global}
Zhang, S., Tachet, R., and Laroche, R.
\newblock Global optimality and finite sample analysis of softmax off-policy
  actor critic under state distribution mismatch.
\newblock \emph{arXiv preprint arXiv:2111.02997}, 2021{\natexlab{a}}.

\bibitem[Zhang et~al.(2021{\natexlab{b}})Zhang, Chen, Zhu, and
  Sun]{zhang2021corruption}
Zhang, X., Chen, Y., Zhu, J., and Sun, W.
\newblock Corruption-robust offline reinforcement learning.
\newblock \emph{arXiv preprint arXiv:2106.06630}, 2021{\natexlab{b}}.

\bibitem[Zhang et~al.(2021{\natexlab{c}})Zhang, Chen, Zhu, and
  Sun]{zhang2021robust}
Zhang, X., Chen, Y., Zhu, X., and Sun, W.
\newblock Robust policy gradient against strong data corruption.
\newblock \emph{arXiv preprint arXiv:2102.05800}, 2021{\natexlab{c}}.

\bibitem[Zhou et~al.(2021)Zhou, Bai, Zhou, Qiu, Blanchet, and
  Glynn]{zhou2021finite}
Zhou, Z., Bai, Q., Zhou, Z., Qiu, L., Blanchet, J., and Glynn, P.
\newblock Finite-sample regret bound for distributionally robust offline
  tabular reinforcement learning.
\newblock In \emph{Proc. International Conference on Artifical Intelligence and
  Statistics (AISTATS)}, pp.\  3331--3339. PMLR, 2021.

\end{thebibliography}

\newpage
\appendix
\onecolumn
\addcontentsline{toc}{section}{Appendix} 
\part{Appendix} 
\parttoc

In the following proofs, for a vector $v\in\mathbb{R}^d$, $\|v\|$ denotes its $l_2$ norm, and $\max v\triangleq \max_i v(i)$ denotes its largest entry. For a matrix $A$, $\|A\|$ denotes its operator norm. For a policy $\pit\in\Pi_\Theta$, denote $s_\theta\triangleq \arg\max_s V^{\pit}(s)$, where $V^{\pit}$ is the robust value function. 


\section{Robust Policy Gradient}\label{sec:lip}
\subsection{Robust Value Function Is Lipschitz}
In this section, we show that the robust value function is Lipschitz.
\begin{lemma}\label{thm:lip}
Under Assumption \ref{ass:lip}, the robust value function is Lipschitz in $\theta$, i.e., for any $s\in\mcs$, any $\theta_1,\theta_2\in\Theta$,
\begin{align}\label{eq:lipV}
     |V^{\pi_{\theta_1}}(s)-V^{\pi_{\theta_2}}(s)| \leq \frac{k_\pi |\mca|}{(1-\gamma)^2}\|\theta_1-\theta_2\|\triangleq L_V\|\theta_1-\theta_2\|,
\end{align}
where $L_V=\frac{k_\pi |\mca|}{(1-\gamma)^2}$.
\end{lemma}
\begin{proof}
For any $\theta_1,\theta_2\in\Theta$, we have that
\begin{align}\label{eq:28}
    &V^{\pi_{\theta_1}}(s)-V^{\pi_{\theta_2}}(s)\nn\\
    &=\sum_{a\in\mca} (\pi_{\theta_1}(a|s)Q^{\pi_{\theta_1}}(s,a))-\sum_{a\in\mca} (\pi_{\theta_2}(a|s)Q^{\pi_{\theta_2}}(s,a))\nn\\
    &=\sum_{a\in\mca} (\pi_{\theta_1}(a|s)Q^{\pi_{\theta_1}}(s,a))-\sum_{a\in\mca} (\pi_{\theta_2}(a|s)Q^{\pi_{\theta_1}}(s,a)) + \sum_{a\in\mca} (\pi_{\theta_2}(a|s)Q^{\pi_{\theta_1}}(s,a))-\sum_{a\in\mca} (\pi_{\theta_2}(a|s)Q^{\pi_{\theta_2}}(s,a))\nn\\
    &=\sum_{a\in\mca} (\pi_{\theta_1}(a|s)-\pi_{\theta_2}(a|s))Q^{\pi_{\theta_1}}(s,a)+\sum_{a\in\mca} \pi_{\theta_2}(a|s)(Q^{\pi_{\theta_1}}(s,a)-Q^{\pi_{\theta_2}}(s,a))\nn\\
    &\overset{(a)}{=}\sum_{a\in\mca} (\pi_{\theta_1}(a|s)-\pi_{\theta_2}(a|s))Q^{\pi_{\theta_1}}(s,a)\nn\\
    &\quad+\sum_{a\in\mca} \pi_{\theta_2}(a|s)\left(\gamma (1-R) \sum_{s'\in\mcs} p^a_{s,{s'}} \left(V^{\pi_{\theta_1}}({s'})-V^{\pi_{\theta_2}}({s'})\right)+\gamma R\left(\max V^{\pi_{\theta_1}}-\max V^{\pi_{\theta_2}}\right) \right)\nn\\
    &\overset{(b)}{=}\sum_{a\in\mca} (\pi_{\theta_1}(a|s)-\pi_{\theta_2}(a|s))Q^{\pi_{\theta_1}}(s,a)\nn\\
    &\quad+\gamma (1-R)\sum_{s'\in\mcs} \mathbb{P}(S_1={s'}|S_0=s,\pi_{\theta_2}) \underbrace{(V^{\pi_{\theta_1}}({s'})-V^{\pi_{\theta_2}}({s'}))}_{(c)}+ \gamma R \left(\max V^{\pi_{\theta_1}}-\max V^{\pi_{\theta_2}}\right),
\end{align}
where the equation $(a)$ is from the fact that $Q^\pi$ is the fixed point of the robust Bellman operator $ \mathbf T_\pi$, i.e., $Q^\pi(s,a)=c(s,a)+\gamma (1-R) \sum_{s'\in\mcs} p^a_{s,s'} V^\pi(s')+\gamma R \max V^\pi$, where $V^\pi(s)=\sum_{a\in\mca} \pi(a|s)Q^\pi(s,a)$, and equation $(b)$ is because $\mathbb{P}(S_1={s'}|S_0=s,\pi_{\theta_2})=\sum_{a\in\mca} \pi_{\theta_2}(a|s) p^a_{s,{s'}}$. Note that the term $(c)$   $V^{\pi_{\theta_1}}({s'})-V^{\pi_{\theta_2}}({s'})$ is again the difference between robust value functions, hence apply \eqref{eq:28} recursively, and we obtain that 
\begin{align}
    &V^{\pi_{\theta_1}}(s)-V^{\pi_{\theta_2}}(s)\nn\\
    &=\sum_{a\in\mca} (\pi_{\theta_1}(a|s)-\pi_{\theta_2}(a|s))Q^{\pi_{\theta_1}}(s,a)\nn\\
    &\quad+\gamma (1-R)\sum_{s'\in\mcs} \mathbb{P}(S_1={s'}|S_0=s,\pi_{\theta_2}) {(V^{\pi_{\theta_1}}({s'})-V^{\pi_{\theta_2}}({s'}))}+ \gamma R \left(\max V^{\pi_{\theta_1}}-\max V^{\pi_{\theta_2}}\right)\nn\\
    &=\sum_{a\in\mca} (\pi_{\theta_1}(a|s)-\pi_{\theta_2}(a|s))Q^{\pi_{\theta_1}}(s,a)+\gamma R \left(\max V^{\pi_{\theta_1}}-\max V^{\pi_{\theta_2}}\right)\nn\\
    &\quad+ \gamma (1-R)\sum_{s'\in\mcs} \mathbb{P}(S_1={s'}|S_0=s,\pi_{\theta_2}) \bigg(\sum_{a'} (\pi_{\theta_1}(a'|{s'})-\pi_{\theta_2}(a'|{s'})) Q^{\pi_{\theta_1}}({s'},a') \nn\\
    &\quad+\gamma R \left(\max V^{\pi_{\theta_1}}-\max V^{\pi_{\theta_2}}\right)+\gamma (1-R)\sum_{s''\in\mcs} \mathbb{P}(S_1={s''}|S_0={s'},\pi_{\theta_2}) (V^{\pi_{\theta_1}}({s''})-V^{\pi_{\theta_2}}({s''})) \bigg)\nn\\
    &=\sum_{a\in\mca} (\pi_{\theta_1}(a|s)-\pi_{\theta_2}(a|s))Q^{\pi_{\theta_1}}(s,a)+\gamma R \left(\max V^{\pi_{\theta_1}}-\max V^{\pi_{\theta_2}}\right)\nn\\
    &\quad+\gamma (1-R) \sum_{s'\in\mcs} \mathbb{P}(S_1={s'}|S_0=s,\pi_{\theta_2}) \sum_{a'} (\pi_{\theta_1}(a'|{s'})-\pi_{\theta_2}(a'|{s'})) Q^{\pi_{\theta_1}}({s'},a')\nn\\
    &\quad+\gamma R \gamma (1-R)\left(\max V^{\pi_{\theta_1}}-\max V^{\pi_{\theta_2}}\right)\nn\\
    &\quad+\gamma^2(1-R)^2\sum_{s'\in\mcs} \mathbb{P}(S_1={s'}|S_0=s,\pi_{\theta_2}) \sum_{s''\in\mcs} \mathbb{P}(S_1={s''}|S_0={s'},\pi_{\theta_2}) (V^{\pi_{\theta_1}}({s''})-V^{\pi_{\theta_2}}({s''}))\nn\\
    &=\sum_{{s'}\in\mcs} \left(\mathbb{P}(S_0={s'}|S_0=s,\pi_{\theta_2}) + \gamma (1-R) \mathbb{P}(S_1={s'}|S_0=s,\pi_{\theta_2})\right) \left(\sum_{a\in\mca} (\pi_{\theta_1}(a|{s'})-\pi_{\theta_2}(a|{s'})) Q^{\pi_{\theta_1}}(x,a)\right)\nn\\
    &\quad + \gamma R (1+\gamma (1-R)) \left(\max V^{\pi_{\theta_1}}-\max V^{\pi_{\theta_2}}\right)\nn\\
    &\quad +\gamma^2(1-R)^2\sum_{s'\in\mcs} \mathbb{P}(S_1={s'}|S_0=s,\pi_{\theta_2}) \sum_{s''\in\mcs} \mathbb{P}(S_1={s''}|S_0={s'},\pi_{\theta_2}) (V^{\pi_{\theta_1}}({s''})-V^{\pi_{\theta_2}}({s''})).
\end{align}
Note that in the last term, $\sum_{s'\in\mcs} \mathbb{P}(S_1={s'}|S_0=s,\pi_{\theta_2})\mathbb{P}(S_1={s''}|S_0={s'},\pi_{\theta_2})=\mathbb{P}(S_2={s''}|S_0=s,\pi_{\theta_2})$, hence the last term can be written as $\gamma^2(1-R)^2\sum_{s'\in\mcs} \mathbb{P}(S_2={s'}|S_0=s,\pi_{\theta_2}) (V^{\pi_{\theta_1}}({s'})-V^{\pi_{\theta_2}}({s'}))$. Then the difference between robust value functions can be further written as
\begin{align}
     &V^{\pi_{\theta_1}}(s)-V^{\pi_{\theta_2}}(s)\nn\\
     &=\sum_{{s'}\in\mcs} \left(\mathbb{P}(S_0={s'}|S_0=s,\pi_{\theta_2}) + \gamma (1-R) \mathbb{P}(S_1={s'}|S_0=s,\pi_{\theta_2})\right) \left(\sum_{a\in\mca} (\pi_{\theta_1}(a|{s'})-\pi_{\theta_2}(a|{s'})) Q^{\pi_{\theta_1}}({s'},a)\right)\nn\\
    &\quad + \gamma R (1+\gamma (1-R)) \left(\max V^{\pi_{\theta_1}}-\max V^{\pi_{\theta_2}}\right)\nn\\
    &\quad +\gamma^2(1-R)^2\sum_{s'\in\mcs} \mathbb{P}(S_2={s'}|S_0=s,\pi_{\theta_2}) (V^{\pi_{\theta_1}}({s'})-V^{\pi_{\theta_2}}({s'})).
\end{align}
Recursively applying this equation, we  have that
\begin{align}
    &V^{\pi_{\theta_1}}(s)-V^{\pi_{\theta_2}}(s)\nn\\
    &=\sum_{{s'}\in\mcs}  \left(\sum_{k=0}^{\infty} (\gamma(1-R))^k \mathbb{P}(S_k={s'}|S_0=s,\pi_{\theta_2}) \left(\sum_{a\in\mca} (\pi_{\theta_1}(a|{s'})-\pi_{\theta_2}(a|{s'}))Q^{\pi_{\theta_1}}({s'},a) \right) \right)\nn\\
    &\quad+\gamma R(\sum^{\infty}_{k=0} (\gamma(1-R))^k) \left(\max V^{\pi_{\theta_1}}-\max V^{\pi_{\theta_2}}\right)\nn\\
    &=\sum_{{s'}\in\mcs}  \left(\sum_{k=0}^{\infty} (\gamma(1-R))^k \mathbb{P}(S_k={s'}|S_0=s,\pi_{\theta_2}) \left(\sum_{a\in\mca} (\pi_{\theta_1}(a|{s'})-\pi_{\theta_2}(a|{s'}))Q^{\pi_{\theta_1}}({s'},a) \right) \right)\nn\\
    &\quad+\frac{\gamma R}{1-\gamma +\gamma R} \left(\max V^{\pi_{\theta_1}}-\max V^{\pi_{\theta_2}}\right).
\end{align}
 The first term can be treated as an expectation under the discounted visitation distribution:
\begin{align}\label{eq:diff1}
   V^{\pi_{\theta_1}}(s)-V^{\pi_{\theta_2}}(s) &=\frac{1}{1-\gamma(1-R)}\sum_{{s'}\in\mcs}  \left(d^{\pi_{\theta_2}}_{s}({s'}) \left(\sum_{a\in\mca} (\pi_{\theta_1}(a|{s'})-\pi_{\theta_2}(a|{s'}))Q^{\pi_{\theta_1}}({s'},a) \right) \right)\nn\\
    &\quad+\frac{\gamma R}{1-\gamma +\gamma R} \left(\max V^{\pi_{\theta_1}}-\max V^{\pi_{\theta_2}}\right),
\end{align}
where $d^{\pi_{\theta_2}}_{s}({s'})=(1-\gamma(1-R)) \sum_{k=0}^{\infty} \gamma^k(1-R)^k \left(\mathbb{P}(S_k={s'}|S_0=s_0,\pi) \right)$.
 
Taking absolute value on both sides, we then have that
\begin{align}\label{eq:difV1}
     |V^{\pi_{\theta_1}}(s)-V^{\pi_{\theta_2}}(s)| &=\bigg|\frac{1}{1-\gamma(1-R)}\sum_{{s'}\in\mcs}  \left(d^{\pi_{\theta_2}}_{s}({s'}) \left(\sum_{a\in\mca} (\pi_{\theta_1}(a|{s'})-\pi_{\theta_2}(a|{s'}))Q^{\pi_{\theta_1}}({s'},a) \right) \right)\nn\\
     &\quad+\frac{\gamma R}{1-\gamma +\gamma R} \left(\max V^{\pi_{\theta_1}}-\max V^{\pi_{\theta_2}}\right)\bigg|.
\end{align}

We then consider the second term $\max V^{\pi_{\theta_1}}-\max V^{\pi_{\theta_2}}$ in \eqref{eq:difV1}. Without loss of generality, assume that $V^{\pi_{\theta_1}}(s_{\theta_1})\geq V^{\pi_{\theta_2}}(s_{\theta_2})$.
Then $|\max V^{\pi_{\theta_1}}-\max V^{\pi_{\theta_2}}|=V^{\pi_{\theta_1}}(s_{\theta_1})-V^{\pi_{\theta_2}}(s_{\theta_2})=V^{\pi_{\theta_1}}(s_{\theta_1})-V^{\pi_{\theta_2}}(s_{\theta_1})+V^{\pi_{\theta_2}}(s_{\theta_1})-V^{\pi_{\theta_2}}(s_{\theta_2})\leq V^{\pi_{\theta_1}}(s_{\theta_1})-V^{\pi_{\theta_2}}(s_{\theta_1})$ because $V^{\pi_{\theta_2}}(s_{\theta_1})\leq V^{\pi_{\theta_2}}(s_{\theta_2})$. Then \eqref{eq:difV1} can be written as 
\begin{align}\label{eq:difV2}
   |V^{\pi_{\theta_1}}(s)-V^{\pi_{\theta_2}}(s)| &\leq \left|\frac{1}{1-\gamma(1-R)}\sum_{{s'}\in\mcs}  \left(d^{\pi_{\theta_2}}_{s}({s'}) \left(\sum_{a\in\mca} (\pi_{\theta_1}(a|{s'})-\pi_{\theta_2}(a|{s'}))Q^{\pi_{\theta_1}}({s'},a) \right) \right)\right|\nn\\
    &\quad\quad+\frac{\gamma R}{1-\gamma +\gamma R} \left|\left(V^{\pi_{\theta_1}}(s_{\theta_1})-V^{\pi_{\theta_2}}(s_{\theta_1})\right)\right|.
\end{align}
Note that this inequality holds for any state $s$. By setting $s=s_{\theta_1}$ we have
\begin{align}
     |V^{\pi_{\theta_1}}(s_{\theta_1})-V^{\pi_{\theta_2}}(s_{\theta_1})|
    &\leq \left|\frac{1}{1-\gamma(1-R)}\sum_{{s'}\in\mcs}  \left(d^{\pi_{\theta_2}}_{s_{\theta_1}}({s'}) \left(\sum_{a\in\mca} (\pi_{\theta_1}(a|{s'})-\pi_{\theta_2}(a|{s'}))Q^{\pi_{\theta_1}}({s'},a) \right) \right)\right|\nn\\
    &\quad+\frac{\gamma R}{1-\gamma +\gamma R} |\left(V^{\pi_{\theta_1}}(s_{\theta_1})-V^{\pi_{\theta_2}}(s_{\theta_1})\right)|,
\end{align}
which implies that 
\begin{align}
     |V^{\pi_{\theta_1}}(s_{\theta_1})-V^{\pi_{\theta_2}}(s_{\theta_1})| \leq \left|\frac{1}{1-\gamma}\sum_{{s'}\in\mcs}  \left(d^{\pi_{\theta_2}}_{s_{\theta_1}}({s'}) \left(\sum_{a\in\mca} (\pi_{\theta_1}(a|{s'})-\pi_{\theta_2}(a|{s'}))Q^{\pi_{\theta_1}}({s'},a) \right) \right)\right|.
\end{align}
Plugging this inequality in \eqref{eq:difV2} implies that
\begin{align}\label{eq:difV3}
    &|V^{\pi_{\theta_1}}(s)-V^{\pi_{\theta_2}}(s)|\nn\\
    &\leq \left|\frac{1}{1-\gamma(1-R)}\sum_{{s'}\in\mcs}  \left(d^{\pi_{\theta_2}}_{s}({s'}) \left(\sum_{a\in\mca} (\pi_{\theta_1}(a|{s'})-\pi_{\theta_2}(a|{s'}))Q^{\pi_{\theta_1}}({s'},a) \right) \right)\right|\nn\\
    &\quad+\left|\frac{\gamma R}{(1-\gamma)(1-\gamma+\gamma R)}\sum_{{s'}\in\mcs}  \left(d^{\pi_{\theta_2}}_{s_{\theta_1}}({s'}) \left(\sum_{a\in\mca} (\pi_{\theta_1}(a|{s'})-\pi_{\theta_2}(a|{s'}))Q^{\pi_{\theta_1}}({s'},a) \right) \right)\right|.
\end{align}
Note that from $c(s,a)\in[0,1]$, $Q^\pi(s,a)=\max_\kappa \mathbb{E}_\kappa [\sum^\infty_{t=0} \gamma^t c(S_t,A_t)|S_0=s,A_0=a,\pi]\leq \frac{1}{1-\gamma}$ for any $\pi,s,a$; and  from Assumption \ref{ass:lip}, $ |\pi_{\theta_1}(a|{s'})-\pi_{\theta_2}(a|{s'})|\leq   k_\pi\|\theta_1-\theta_2\|$, hence we have that 
\begin{align} 
     |V^{\pi_{\theta_1}}(s)-V^{\pi_{\theta_2}}(s)| \leq &\left(\frac{1}{1-\gamma(1-R)}+\frac{\gamma R}{(1-\gamma)(1-\gamma+\gamma R)}\right)\frac{k_\pi|\mca|}{1-\gamma} \|\theta_1-\theta_2\|\nn\\
     &=\frac{k_\pi |\mca|}{(1-\gamma)^2}\|\theta_1-\theta_2\|,
\end{align}
which completes the proof.
\end{proof}
The following corollary is straightforward hence the proof is omitted:
\begin{corollary}
The robust action-value functions $Q^{\pi_{\theta}}(s,a)$ for any $s \in
\mcs, a\in \mca$  and the objective function $J_\rho(\theta)=\sum_{s\in\mcs} \rho(s) V^{\pi_{\theta}}(s)$ are Lipschitz in $\theta$ with constant $L_V$, i.e., for any $\theta_1,\theta_2\in\Theta$,
\begin{align}
    |Q^{\pi_{\theta_1}}(s,a)-Q^{\pi_{\theta_2}}(s,a)|\leq L_V\|\theta_1-\theta_2\|,\\
    |J_\rho(\theta_1)-J_\rho(\theta_2)|\leq L_V\|\theta_1-\theta_2\|.
\end{align}
\end{corollary}

\subsection{Proof of Theorem \ref{coro:subJ}: Sub-gradient of Robust Value Function}\label{sec:subg}
In this section, we derive the sub-gradient of robust value function $J_\rho(\theta)$ and prove Theorem \ref{coro:subJ}.  

Denote $\st=\arg\max_s V^{\pit}(s)$, and define 
\begin{align}
    \hat{\phi}_s(\theta)&\triangleq\frac{\gamma R}{(1-\gamma)(1-\gamma+\gamma R)}\sum_{{s'}\in\mcs} d^{\pit}_{\st}({s'}) \sum_{a\in\mca} \nabla \pit(a|{s'}) Q^{\pit}({s'},a)\nn\\
    & \quad\quad+\frac{1}{1-\gamma+\gamma R} \sum_{{s'}\in\mcs} d^{\pit}_{s}({s'}) \sum_{a\in\mca} \nabla \pit(a|{s'}) Q^{\pit}({s'},a),
\end{align}
Recall from \eqref{eq:subgra} that $\psi_\rho(\theta)$ is the average of $\hat{\phi}_s(\theta)$ under distribution $\rho$. In the following, we show that $\hat{\phi}_s(\theta)\in\partial V^{\pit}(s)$, and hence naturally, $\psi_\rho(\theta)\in\partial J_\rho(\theta)$. In Section \ref{sec:phi}, we will demonstrate how we derive the expression of $\hat{\phi}_s(\theta)$.

From \eqref{eq:diff1} we can show that
\begin{align} \label{eq:40}
   V^{\pi_{\theta_1}}(s)-V^{\pi_{\theta_2}}(s) &=\frac{1}{1-\gamma(1-R)}\sum_{{s'}\in\mcs}  \left(d^{\pi_{\theta_2}}_{s}({s'}) \left(\sum_{a\in\mca} (\pi_{\theta_1}(a|{s'})-\pi_{\theta_2}(a|{s'}))Q^{\pi_{\theta_1}}({s'},a) \right) \right)\nn\\
    &\quad+\frac{\gamma R}{1-\gamma +\gamma R} \left(V^{\pi_{\theta_1}}(s_{\theta_1})-V^{\pi_{\theta_2}}(s_{\theta_2})\right)\nn\\
    &= \frac{1}{1-\gamma(1-R)}\sum_{{s'}\in\mcs}  \left(d^{\pi_{\theta_2}}_{s}({s'}) \left(\sum_{a\in\mca} (\pi_{\theta_1}(a|{s'})-\pi_{\theta_2}(a|{s'}))Q^{\pi_{\theta_1}}({s'},a) \right) \right)\nn\\
    &\quad+\frac{\gamma R}{1-\gamma +\gamma R} \left(V^{\pi_{\theta_1}}(s_{\theta_1})-V^{\pi_{\theta_1}}(s_{\theta_2})+V^{\pi_{\theta_1}}(s_{\theta_2})-V^{\pi_{\theta_2}}(s_{\theta_2})\right)\nn\\
    &\geq \frac{1}{1-\gamma(1-R)}\sum_{{s'}\in\mcs}  \left(d^{\pi_{\theta_2}}_{s}({s'}) \left(\sum_{a\in\mca} (\pi_{\theta_1}(a|{s'})-\pi_{\theta_2}(a|{s'}))Q^{\pi_{\theta_1}}({s'},a) \right) \right)\nn\\
    &\quad+\frac{\gamma R}{1-\gamma +\gamma R} \left(V^{\pi_{\theta_1}}(s_{\theta_2})-V^{\pi_{\theta_2}}(s_{\theta_2})\right),
\end{align}
where the last inequality is from $V^{\pi_{\theta_1}}(s_{\theta_1})-V^{\pi_{\theta_1}}(s_{\theta_2})\geq 0$. Set $s=s_{\theta_2}$ in the LHS of \eqref{eq:40}, we further have that
\begin{align} \label{eq:41}
   &V^{\pi_{\theta_1}}(s_{\theta_2})-V^{\pi_{\theta_2}}(s_{\theta_2}) \geq \frac{1}{  (1-\gamma)  }\sum_{{s'}\in\mcs}  \left(d^{\pi_{\theta_2}}_{s_{\theta_2}}({s'}) \left(\sum_{a\in\mca} (\pi_{\theta_1}(a|{s'})-\pi_{\theta_2}(a|{s'}))Q^{\pi_{\theta_1}}({s'},a) \right) \right).
\end{align}
Plug \eqref{eq:41} back into \eqref{eq:40}, it follows that
\begin{align}
    V^{\pi_{\theta_1}}(s)-V^{\pi_{\theta_2}}(s)&\geq \frac{1}{1-\gamma(1-R)}\sum_{{s'}\in\mcs}  \left(d^{\pi_{\theta_2}}_{s}({s'}) \left(\sum_{a\in\mca} (\pi_{\theta_1}(a|{s'})-\pi_{\theta_2}(a|{s'}))Q^{\pi_{\theta_1}}({s'},a) \right) \right)\nn\\
    &\quad+\frac{\gamma R}{ (1-\gamma+\gamma R) (1-\gamma)  }\sum_{{s'}\in\mcs}  \left(d^{\pi_{\theta_2}}_{s_{\theta_2}}({s'}) \left(\sum_{a\in\mca} (\pi_{\theta_1}(a|{s'})-\pi_{\theta_2}(a|{s'}))Q^{\pi_{\theta_1}}({s'},a) \right) \right).
\end{align}
Hence for any $\theta, \theta+h \in\Theta$, we have that 
\begin{align}
    V^{\pi_{\theta+h}}(s)-V^{\pit}(s)&\geq \frac{1}{1-\gamma(1-R)}\sum_{s'\in\mcs} d^{\pit}_s({s'})\sum_{a\in\mca} \left(\pi_{\theta+h}(a|{s'})-\pit(a|{s'}) \right)Q^{\pi_{\theta+h}}({s'},a)\nn\\
    &\quad+\frac{\gamma R}{(1-\gamma)(1-\gamma+\gamma R)}\sum_{s'\in\mcs} d^{\pit}_{s_\theta}({s'})\sum_{a\in\mca} \left(\pi_{\theta+h}(a|{s'})-\pit(a|{s'}) \right)Q^{\pi_{\theta+h}}({s'},a).
\end{align}
We then can show that
\begin{align}\label{eq:V}
     V^{\pi_{\theta+h}}(s)-V^{\pit}(s)-\langle \hat{\phi}_s(\theta), h\rangle&\geq \frac{1}{1-\gamma(1-R)}\sum_{s'\in\mcs} d^{\pit}_s({s'})\sum_{a\in\mca} \left(\pi_{\theta+h}(a|{s'})-\pit(a|{s'}) \right)Q^{\pi_{\theta+h}}({s'},a)\nn\\
    &\quad-\frac{1}{1-\gamma+\gamma R} \sum_{{s'}\in\mcs} d^{\pit}_{s}({s'}) \sum_{a\in\mca} \langle h, \nabla \pit(a|{s'}) \rangle Q^{\pit}({s'},a)\nn\\
    &\quad+\frac{\gamma R}{(1-\gamma)(1-\gamma+\gamma R)}\sum_{s'\in\mcs} d^{\pit}_{s_\theta}({s'})\sum_{a\in\mca} \left(\pi_{\theta+h}(a|{s'})-\pit(a|{s'}) \right)Q^{\pi_{\theta+h}}({s'},a)\nn\\
    &\quad - \frac{\gamma R}{1-\gamma+\gamma R}\frac{1}{1-\gamma}\sum_{{s'}\in\mcs} d^{\pit}_{\st}({s'}) \sum_{a\in\mca} \langle h, \nabla \pit(a|{s'}) \rangle Q^{\pit}({s'},a).
\end{align}
The first two terms in \eqref{eq:V} can be rewritten as follows:
\begin{align}
    &\frac{1}{1-\gamma(1-R)}\sum_{s'\in\mcs} d^{\pit}_s({s'})\sum_{a\in\mca} \left(\pi_{\theta+h}(a|{s'})-\pit(a|{s'}) \right)Q^{\pi_{\theta+h}}({s'},a)\nn\\
    &\quad-\frac{1}{1-\gamma+\gamma R} \sum_{{s'}\in\mcs} d^{\pit}_{s}({s'}) \sum_{a\in\mca} \langle h, \nabla \pit(a|{s'}) \rangle Q^{\pit}({s'},a)\nn\\
    &=\frac{1}{1-\gamma(1-R)}\sum_{s'\in\mcs} d^{\pit}_s({s'})\sum_{a\in\mca}\left( \left(\pi_{\theta+h}(a|{s'})-\pit(a|{s'}) \right)Q^{\pi_{\theta+h}}({s'},a)-\langle h, \nabla \pit(a|{s'}) \rangle Q^{\pit}({s'},a)\right).
\end{align}
Moreover note that
\begin{align}\label{eq:term1}
    &\left(\pi_{\theta+h}(a|{s'})-\pit(a|{s'}) \right)Q^{\pi_{\theta+h}}({s'},a)-\langle h, \nabla \pit(a|{s'}) \rangle Q^{\pit}({s'},a)\nn\\
    &=\left(\pi_{\theta+h}(a|{s'})-\pit(a|{s'}) \right)Q^{\pi_{\theta+h}}({s'},a)-\langle h, \nabla \pit(a|{s'}) \rangle Q^{\pi_{\theta+h}}({s'},a)\nn\\
    &\quad +\langle h, \nabla \pit(a|{s'}) \rangle Q^{\pi_{\theta+h}}({s'},a)-\langle h, \nabla \pit(a|{s'}) \rangle Q^{\pit}({s'},a).
\end{align}
Since $\pit$ is differentiable, hence
\begin{align}
    \lim_{\|h\|\to 0} \frac{\pi_{\theta+h}(a|{s'})-\pit(a|{s'})-\langle h, \nabla \pit(a|{s'}) \rangle }{\|h\|}=0.
\end{align}
Because $Q^{\pit}\geq 0$ for any $\theta$, we further have that
\begin{align}\label{eq:sub1}
    \lim_{\|h\|\to 0} \frac{\left(\pi_{\theta+h}(a|{s'})-\pit(a|{s'}) \right)Q^{\pi_{\theta+h}}({s'},a)-\langle h, \nabla \pit(a|{s'}) \rangle Q^{\pi_{\theta+h}}({s'},a)}{\|h\|}\geq 0.
\end{align}
For the remaining term in \eqref{eq:term1}, note that $\lim_{h\to 0} Q^{\pi_{\theta+h}}({s'},a)-Q^{\pit}({s'},a)=0$ as $Q^{\pit}$ is Lipschitz in $\theta$, thus 
\begin{align}\label{eq:sub2}
    &\lim_{\|h\|\to 0} \frac{\langle h, \nabla \pit(a|{s'}) (Q^{\pi_{\theta+h}}({s'},a)-Q^{\pit}({s'},a)) \rangle}{\|h\|}\nn\\
    &=\lim_{\|h\|\to 0} {\left\langle \frac{h}{\|h\|}, \nabla \pit(a|{s'}) (Q^{\pi_{\theta+h}}({s'},a)-Q^{\pit}({s'},a)) \right\rangle }\nn\\
    &=\lim_{\|h\|\to 0} \langle e_h, \nabla \pit(a|{s'}) (Q^{\pi_{\theta+h}}({s'},a)-Q^{\pit}({s'},a)) \rangle,
\end{align}
where $e_h$ is the normalized vector of $h$, and $\nabla \pit(a|{s'}) (Q^{\pi_{\theta+h}}({s'},a)-Q^{\pit}({s'},a)) $ is a vector of  dimension $|\mcs|\times|\mca|$. 

Clearly, $\lim_{\|h\|\to 0} \langle e_h, \nabla \pit(a|{s'}) (Q^{\pi_{\theta+h}}({s'},a)-Q^{\pit}({s'},a)) \rangle= 0$, which is also due to the Lipschitz of  $Q^{\pit}$. 
Hence combining all these inequalities in \eqref{eq:sub1} and \eqref{eq:sub2}, we have that
\begin{align}
    \lim_{\|h\|\to 0}\inf_{h\neq 0} \frac{\sum_{s'\in\mcs} d^{\pit}_s({s'}) \sum_{a\in\mca}\left( \left(\pi_{\theta+h}(a|{s'})-\pit(a|{s'}) \right)Q^{\pi_{\theta+h}}({s'},a)-\langle h, \nabla \pit(a|{s'}) \rangle Q^{\pit}({s'},a)\right)}{\|h\|}\geq 0. 
\end{align}
Similarly, we can also show that for the remaining terms in \eqref{eq:V},
\begin{align}
    \lim_{\|h\|\to 0}\inf_{h\neq 0} \frac{\sum_{s'\in\mcs} d^{\pit}_{\st}({s'}) \sum_{a\in\mca}\left( \left(\pi_{\theta+h}(a|{s'})-\pit(a|{s'}) \right)Q^{\pi_{\theta+h}}({s'},a)-\langle h, \nabla \pit(a|{s'}) \rangle Q^{\pit}({s'},a)\right)}{\|h\|}\geq 0. 
\end{align}
Hence 
\begin{align}
    \lim_{\|h\|\to 0}\inf_{h\neq 0} \frac{V^{\pi_{\theta+h}}(s)-V^{\pit}(s)-\langle \hat{\phi}_s(\theta), h\rangle}{\|h\|}\geq 0,
\end{align}
and this implies that $\hat{\phi}_s(\theta)\in\partial V^{\pit}(s)$ for any $\theta$. 

Now we consider $J_\rho$. From the definition, $J_\rho(\theta)=\sum_{{s}\in\mcs} \rho({s})V^{\pit}({s})$. Hence we have that
\begin{align}
    \partial J_\rho(\theta)& \supseteq \sum_{{s}\in\mcs} \rho({s}) \partial V^{\pit}({s})\nn\\
    &\ni \frac{\gamma R}{1-\gamma+\gamma R}\frac{1}{1-\gamma}\sum_{{s'}\in\mcs} d^{\pit}_{\st}({s'}) \sum_{a\in\mca} \nabla \pit(a|{s'}) Q^{\pit}({s'},a)\nn\\
    &\quad+\frac{1}{1-\gamma+\gamma R}\sum_{s\in\mcs} \rho(s) \sum_{{s'}\in\mcs} d^{\pit}_{s}({s'}) \sum_{a\in\mca} \nabla \pit(a|{s'}) Q^{\pit}({s'},a)\nn\\
    &=\frac{\gamma R}{(1-\gamma)(1-\gamma+\gamma R)}\sum_{{s'}\in\mcs} d^{\pit}_{\st}({s'}) \sum_{a\in\mca} \nabla \pit(a|{s'}) Q^{\pit}({s'},a)\nn\\
    &\quad+\frac{1}{1-\gamma+\gamma R} \sum_{{s'}\in\mcs} d^{\pit}_{\rho}({s'}) \sum_{a\in\mca} \nabla \pit(a|{s'}) Q^{\pit}({s'},a)\nn\\
    &\triangleq \psi_\rho(\theta),
\end{align}
which implies  $\psi_\rho(\theta)\in\partial J_\rho(\theta)$.

\subsubsection{Derivation of The Sub-gradient $\hat{\phi}_s(\theta)$}\label{sec:phi}
In this section we show how we derive the expression of the sub-gradient $\hat{\phi}_s$. 

Recall that $\max_s V^{\pit}(s)$ is Lipschitz in $\theta$ as shown in Section \ref{sec:lip}. From the Rademacher's theorem \cite{federer2014geometric}, we know that $\max_s V^{\pit}(s)$ is differentiable almost everywhere. Hence at those differentiable $\theta$, we define $\phi(\theta)=\nabla \max_s V^{\pit}(s)$, which can be also viewed as a sub-gradient of $ \max_s V^{\pit}(s)$, i.e., $\phi(\theta)\in\partial \max_s V^{\pit}(s)$. 

It is known from \cite{kruger2003frechet} that $  \partial(f)+\partial (g)  \subseteq \partial (f+g)$, hence we have that for any $s\in\mcs$,
\begin{align}\label{eq:54}
     \partial V^{\pit}(s)& \supseteq \sum_{a\in\mca} \nabla \pi_{\theta}(a|s) Q^{\pit}(s,a)+\sum_{a\in\mca} \pi_{\theta}(a|s) \partial Q^{\pit}(s,a)\nn\\
     & \supseteq \sum_{a\in\mca} \nabla \pi_{\theta}(a|s) Q^{\pit}(s,a)+\sum_{a\in\mca} \pi_{\theta}(a|s) \partial \left(c(s,a)+\gamma (1-R)\sum_{s'\in\mcs} p^a_{s,{s'}} V^{\pit}({s'})+\gamma R \max_s V^{\pi_{\theta}}(s)  \right)\nn\\
     & \supseteq\sum_{a\in\mca} \nabla \pi_{\theta}(a|s) Q^{\pit}(s,a)+\sum_{a\in\mca} \pi_{\theta}(a|s)  \left(\gamma (1-R)\sum_{s'\in\mcs} p^a_{s,{s'}} \partial V^{\pit}({s'})+\gamma R \phi(\theta)  \right)\nn\\
     &=\sum_{a\in\mca} \nabla \pi_{\theta}(a|s) Q^{\pit}(s,a)+\sum_{a\in\mca} \pi_{\theta}(a|s)  \left(\gamma (1-R)\sum_{s'\in\mcs} p^a_{s,{s'}} \partial V^{\pit}({s'})\right)+\gamma R \phi(\theta)\nn\\
     &=\gamma R \phi(\theta)+\sum_{a\in\mca} \nabla \pi_{\theta}(a|s) Q^{\pit}(s,a)+\gamma (1-R)\sum_{s'\in\mcs} \mathbb{P}(S_1={s'}|S_0=s,\pit) \partial V^{\pit}({s'}).
\end{align}
Recursively applying \eqref{eq:54}, we have that 
\begin{align}\label{eq:subb1}
     \partial V^{\pit}(s)& \supseteq \gamma R \phi(\theta)+\sum_{a\in\mca} \nabla \pi_{\theta}(a|s) Q^{\pit}(s,a)+\gamma (1-R)\sum_{s'\in\mcs} \mathbb{P}(S_1={s'}|S_0=s,\pit) \partial V^{\pit}({s'})\nn\\
     & \supseteq\gamma R \phi(\theta)+\sum_{a\in\mca} \nabla \pi_{\theta}(a|s) Q^{\pit}(s,a)\nn\\
     &\quad+\gamma (1-R)\sum_{s'\in\mcs}\mathbb{P}(S_1={s'}|S_0=s,\pit) \left(\sum_{a\in\mca} \nabla \pit(a|{s'})Q^{\pit}({s'},a)+\sum_{a\in\mca} \pit(a|{s'}) \partial Q^{\pit}({s'},a) \right)\nn\\
     & \supseteq\gamma R \phi(\theta)+\sum_{a\in\mca} \nabla \pi_{\theta}(a|s) Q^{\pit}(s,a) +\gamma (1-R)\sum_{s'\in\mcs}\mathbb{P}(S_1={s'}|S_0=s,\pit)\sum_{a\in\mca} \nabla \pit(a|{s'})Q^{\pit}({s'},a)\nn\\
     &\quad+\gamma(1-R)\sum_{s'\in\mcs}\mathbb{P}(S_1={s'}|S_0=s,\pit)\sum_{a\in\mca}\pit(a|{s'}) \partial Q^{\pit}({s'},a)\nn\\
     & \supseteq\gamma R \phi(\theta)+\sum_{a\in\mca} \nabla \pi_{\theta}(a|s) Q^{\pit}(s,a) +\gamma (1-R)\sum_{s'\in\mcs}\mathbb{P}(S_1={s'}|S_0=s,\pit)\sum_{a\in\mca} \nabla \pit(a|{s'})Q^{\pit}({s'},a)\nn\\
     &\quad+\gamma(1-R)\sum_{s'\in\mcs}\mathbb{P}(S_1={s'}|S_0=s,\pit)\sum_{a\in\mca}\pit(a|{s'}) \left(\gamma(1-R)\sum_{s''\in\mcs}p^a_{{s'},{s''}} \partial V^{\pit}({s''}) +\gamma R \phi(\theta) \right)\nn\\
     &=\gamma R \phi(\theta)+\sum_{a\in\mca} \nabla \pi_{\theta}(a|s) Q^{\pit}(s,a) +\gamma (1-R)\sum_{s'\in\mcs}\mathbb{P}(S_1={s'}|S_0=s,\pit)\sum_{a\in\mca} \nabla \pit(a|{s'})Q^{\pit}({s'},a)\nn\\
     &\quad+\gamma^2(1-R)^2\sum_{s''\in\mcs}\mathbb{P}(S_2={s''}|S_0=s,\pit) \partial V^{\pit}({s''}) +\gamma^2 R (1-R) \phi(\theta)\nn\\
     & \supseteq ...\nn\\
     &\ni \gamma R \left(\sum_{k=0}^\infty \gamma^k(1-R)^k\right) \phi(\theta)+\sum_{k}\left(\gamma^k(1-R)^k \sum_{{s'}\in\mcs} \mathbb{P}(S_k={s'}|S_0=s,\pit) \sum_{a\in\mca} \nabla \pit(a|{s'}) Q^{\pit}({s'},a) \right) \nn\\
     &=\frac{\gamma R}{1-\gamma+\gamma R} \phi(\theta)+\frac{1}{1-\gamma+\gamma R} \sum_{{s'}\in\mcs} d^{\pit}_{s}({s'}) \sum_{a\in\mca} \nabla \pit(a|{s'}) Q^{\pit}({s'},a).
\end{align}

From Lemma \ref{lemma:sub_max}, we have that $\left\{\partial V^\theta(s)|_{s=\st} : \st\in \arg\max_s V^{\pit}(s)\right\}  \subseteq \partial \max V^\theta(s)$. Set $s=\st$ in \eqref{eq:subb1}, we have that 
\begin{align} 
    \partial \max_s V^{\pit}(s)=\left\{ \phi(\theta)\right\} \supseteq\partial V^{\pit}|_{\st} &\ni\frac{\gamma R}{1-\gamma+\gamma R} \phi(\theta)+\frac{1}{1-\gamma+\gamma R} \sum_{{s'}\in\mcs} d^{\pit}_{\st}({s'}) \sum_{a\in\mca} \nabla \pit(a|{s'}) Q^{\pit}({s'},a)\nn\\
   &\triangleq c_1 \phi(\theta)+ c_2,
\end{align}
where $c_1=\frac{\gamma R}{1-\gamma+\gamma R}$ and $c_2=\frac{1}{1-\gamma+\gamma R} \sum_{{s'}\in\mcs} d^{\pit}_{\st}({s'}) \sum_{a\in\mca} \nabla \pit(a|{s'}) Q^{\pit}({s'},a)$.
Hence we have that 
\begin{align}
    \phi(\theta)=c_1\phi(\theta)+c_2,
\end{align}
and 
\begin{align}
    \phi(\theta)=\frac{c_2}{1-c_1}=\frac{1}{1-\gamma}\sum_{{s'}\in\mcs} d^{\pit}_{\st}({s'}) \sum_{a\in\mca} \nabla \pit(a|{s'}) Q^{\pit}({s'},a).
\end{align}

Hence we get an explicit expression of  the gradient of $\max_s V^{\pi_{\theta}}(s)$. We then plug it in \eqref{eq:subb1}, we further have that
\begin{align}\label{eq:59}
    &\frac{\gamma R}{1-\gamma+\gamma R}\frac{1}{1-\gamma}\sum_{{s'}\in\mcs} d^{\pit}_{\st}({s'}) \sum_{a\in\mca} \nabla \pit(a|{s'}) Q^{\pit}({s'},a) +\frac{1}{1-\gamma+\gamma R} \sum_{{s'}\in\mcs} d^{\pit}_{s}({s'}) \sum_{a\in\mca} \nabla \pit(a|{s'}) Q^{\pit}({s'},a)\nn\\
    &\in\partial V^{\pit}(s),
\end{align}
for any $\theta$ such that $\max_s V^{\pit}(s)$ is differentiable. 

 As we showed in the last section, at any non-differentiable $\theta$, \eqref{eq:59} is also a sub-gradient of $ V^{\pit}(s)$.
  
\subsection{Proof of Theorem \ref{thm:PL}: Global Optimality: PL-Condition under Direct Parametrization}\label{sec:PL}
In this section, we show that  the sub-gradient in Theorem \ref{coro:subJ} satisfies the PL-condition under the direct parametrization. 
 
Note that each entry of $\psi_\mu(\theta)$ can be written as 
\begin{align} 
    \psi_\mu(\theta)_{s,b}&=\frac{\gamma R}{(1-\gamma)(1-\gamma+\gamma R)}\sum_{s'\in\mcs} d^{\pit}_{\st}(s') \sum_{a\in\mca} (\nabla \pit(a|s'))_{s,b} Q^{\pit}(s',a)\nn\\
    &\quad +\frac{1}{1-\gamma+\gamma R} \sum_{s'\in\mcs} d^{\pit}_{\mu}(s') \sum_{a\in\mca} (\nabla \pit(a|s'))_{s,b} Q^{\pit}(s',a).
\end{align}
When $\pit$ is directly parameterized, i.e., $\pit(a|s)=\theta_{s,a}$, $(\nabla \pit(a|s'))_{s,b}=\mathbbm{1}_{(s,b)=(s',a)}$. Hence, $\psi(\theta)_{s,b}$ for any $s\in\mcs,b\in\mca$ can be further written as
\begin{align}\label{eq:subgrad}
    \psi_\mu(\theta)_{s,b}&=\frac{\gamma R}{(1-\gamma)(1-\gamma+\gamma R)}  d^{\pit}_{\st}(s) Q^{\pit}(s,b) +\frac{1}{1-\gamma+\gamma R}  d^{\pit}_{\mu}(s)  Q^{\pit}(s,b).
\end{align}
Similar to \eqref{eq:difV1} to \eqref{eq:difV3} without taking the absolute value, we have that
\begin{align} 
    &V^{\pi_{\theta_1}}(s)-V^{\pi_{\theta_2}}(s)\nn\\
    &\leq \frac{1}{1-\gamma(1-R)}\sum_{s'\in\mcs}  \left(d^{\pi_{\theta_2}}_{s}(s') \left(\sum_{a\in\mca} (\pi_{\theta_1}(a|s')-\pi_{\theta_2}(a|s'))Q^{\pi_{\theta_1}}(s',a) \right) \right)\nn\\
    &\quad+\frac{\gamma R}{(1-\gamma)(1-\gamma+\gamma R)}\sum_{s'\in\mcs}  \left(d^{\pi_{\theta_2}}_{s_{\theta_1}}(s') \left(\sum_{a\in\mca} (\pi_{\theta_1}(a|s')-\pi_{\theta_2}(a|s'))Q^{\pi_{\theta_1}}(s',a) \right) \right).
\end{align}
Set $\theta_1=\theta$ and $\theta_2=\theta^*$, where $\theta^*\in\arg\min_{(\Delta(\mca))^{|\mcs|}} J_\rho(\theta)$, then we have
\begin{align} 
    &V^{\pit}(s)-V^{\pi_{\theta^*}}(s)\nn\\
    &\leq \frac{1}{1-\gamma(1-R)}\sum_{s'\in\mcs}  \left(d^{\pi_{\theta*}}_{s}(s') \left(\sum_{a\in\mca} (\pit(a|s')-\pi_{\theta*}(a|s'))Q^{\pit}(s',a) \right) \right)\nn\\
    &\quad+\frac{\gamma R}{(1-\gamma)(1-\gamma+\gamma R)}\sum_{s'\in\mcs}  \left(d^{\pi_{\theta*}}_{s_{\theta}}(s') \left(\sum_{a\in\mca} (\pit(a|s')-\pi_{\theta*}(a|s'))Q^{\pit}(s',a) \right) \right).
\end{align}
Note that $J_\rho(\theta)=\sum_{s\in\mcs} \rho(s)V^{\pit}(s)$, hence
\begin{align}\label{eq:PL1}
    J_\rho(\theta)-J^*&=\sum_{s\in\mcs} \rho(s)(V^{\pit}(s)-V^{\pi_{\theta^*}}(s))\nn\\
    &\leq \sum_{s\in\mcs} \rho(s) \Bigg(\frac{1}{1-\gamma(1-R)}\sum_{s'\in\mcs}  \left(d^{\pi_{\theta*}}_{s}(s') \left(\sum_{a\in\mca} (\pit(a|s')-\pi_{\theta*}(a|s'))Q^{\pit}(s',a) \right) \right)\nn\\
    &\quad+\frac{\gamma R}{(1-\gamma)(1-\gamma+\gamma R)}\sum_{s'\in\mcs}  \left(d^{\pi_{\theta*}}_{s_\theta}(s') \left(\sum_{a\in\mca} (\pit(a|s')-\pi_{\theta*}(a|s'))Q^{\pit}(s',a) \right) \right) \Bigg)\nn\\
    &=\sum_{s'\in\mcs}\left(\frac{1}{1-\gamma(1-R)} d^{\pi_{\theta*}}_{\rho}(s')+ \frac{\gamma R}{(1-\gamma)(1-\gamma+\gamma R)}d^{\pi_{\theta*}}_{s_\theta}(s')\right)\left(\sum_{a\in\mca} (\pit(a|s')-\pi_{\theta*}(a|s'))Q^{\pit}(s',a) \right)\nn\\
    &=\sum_{s'\in\mcs}\frac{\frac{1}{1-\gamma(1-R)} d^{\pi_{\theta*}}_{\rho}(s')+ \frac{\gamma R}{(1-\gamma)(1-\gamma+\gamma R)}d^{\pi_{\theta*}}_{s_\theta}(s')}{\frac{1}{1-\gamma+\gamma R}d^{\pit}_{\mu}(s')+\frac{\gamma R}{(1-\gamma)(1-\gamma+\gamma R)}d^{\pit}_{\st}(s') }\left(\frac{1}{1-\gamma+\gamma R}d^{\pit}_{\mu}(s')+\frac{\gamma R}{(1-\gamma)(1-\gamma+\gamma R)}d^{\pit}_{\st}(s')  \right)\nn\\
    &\quad \cdot\left(\sum_{a\in\mca} (\pit(a|s')-\pi_{\theta*}(a|s'))Q^{\pit}(s',a) \right)\nn\\
    &= \sum_{s'\in\mcs} \frac{l^*_\theta(s')}{l_\theta(s')} l_\theta(s') \cdot\left(\sum_{a\in\mca} (\pit(a|s')-\pi_{\theta*}(a|s'))Q^{\pit}(s',a) \right),
\end{align}
where $l_\theta(s)\triangleq \left(\frac{1}{1-\gamma+\gamma R}d^{\pit}_{\mu}(s)+\frac{\gamma R}{(1-\gamma)(1-\gamma+\gamma R)}d^{\pit}_{\st}(s)  \right)$ and $l^*_\theta(s)\triangleq \frac{1}{1-\gamma(1-R)} d^{\pi_{\theta*}}_{\rho}(s)+ \frac{\gamma R}{(1-\gamma)(1-\gamma+\gamma R)}d^{\pi_{\theta*}}_{s_\theta}(s)$. Recall \eqref{eq:subgrad}, then $\psi_\mu(\theta)_{s,b}=l_\theta(s)Q^{\pit}(s,b)$.

The ratio of distribution $\frac{l_\theta^*(s)}{l_\theta(s)}$ can be bounded as follows. Note that 
\begin{align}
    \frac{1}{1-\gamma+\gamma R}d^{\pit}_{\mu}(s)+\frac{\gamma R}{(1-\gamma)(1-\gamma+\gamma R)}d^{\pit}_{\st}(s)&\geq \frac{1}{1-\gamma+\gamma R}d^{\pit}_{\mu}(s)\nn\\
    &\geq \frac{1}{1-\gamma+\gamma R}(1-\gamma+\gamma R) \mu(s)\nn\\
    &\geq \mu_{\min},
\end{align}
which is from $d^{\pit}_{\mu}(s)\geq (1-\gamma+\gamma R) \mu(s)$. Hence the ratio can be bounded as
\begin{align}
    \frac{\frac{1}{1-\gamma(1-R)} d^{\pi_{\theta*}}_{\rho}(s)+ \frac{\gamma R}{(1-\gamma)(1-\gamma+\gamma R)}d^{\pi_{\theta*}}_{s_\theta}(s)}{\frac{1}{1-\gamma+\gamma R}d^{\pit}_{\mu}(s)+\frac{\gamma R}{(1-\gamma)(1-\gamma+\gamma R)}d^{\pit}_{\st}(s) }\leq \frac{1}{(1-\gamma)\mu_{\min}}\triangleq C_{PL}.
\end{align}
Note that  \eqref{eq:PL1} can be further bounded as 
\begin{align}\label{eq:PL2}
    &\sum_{s'\in\mcs} \frac{l^*_\theta(s')}{l_\theta(s')} l_\theta(s') \cdot\left(\sum_{a\in\mca} (\pit(a|s')-\pi_{\theta*}(a|s'))Q^{\pit}(s',a) \right)\nn\\
    &=\sum_{s'\in\mcs} l^*_\theta(s')\left(\sum_{a\in\mca} (\pit(a|s')-\pi_{\theta*}(a|s'))Q^{\pit}(s',a) \right)\nn\\
    &=\sum_{s'\in\mcs} l^*_\theta(s') \langle \pi_\theta(\cdot|s')-\pi_{\theta*}(\cdot|s'), Q^{\pit}(s',\cdot)\rangle\nn\\
    &\leq \sum_{s'\in\mcs} l^*_\theta(s') \max_{\Bar{\pi}_{s'}(\cdot|s')\in \Delta(\mca)} \langle \pi_\theta(\cdot|s')-\Bar{\pi}_{s'}(\cdot|s'), Q^{\pit}(s',\cdot)\rangle,
\end{align}
where $\pi(\cdot|s')\in\Delta(\mca)$ and $Q^{\pit}(s',\cdot)=(Q^{\pit}(s',a_1),Q^{\pit}(s',a_2),...,Q^{\pit}(s',a_{|\mca|}))\in\mathbb{R}^{|\mca|}$. Note that 
\begin{align}
    \max_{\Bar{\pi}_{s'}(\cdot|s')\in \Delta(\mca)} \langle \pi_\theta(\cdot|s')-\Bar{\pi}_{s'}(\cdot|s'), Q^{\pit}(s',\cdot)\rangle\geq 0,
\end{align}
which is because $\max_{\Bar{\pi}_{s'}(\cdot|s')\in \Delta(\mca)} \langle \pi_\theta(\cdot|s')-\Bar{\pi}_{s'}(\cdot|s'), Q^{\pit}(s',\cdot)\rangle\geq \langle \pi_\theta(\cdot|s')-\pi_\theta(\cdot|s'), Q^{\pit}(s',\cdot)\rangle=0$. Hence \eqref{eq:PL2} can be further bounded as
\begin{align}
     &\sum_{s'\in\mcs} \frac{l^*_\theta(s')}{l_\theta(s')} l_\theta(s') \cdot\left(\sum_{a\in\mca} (\pit(a|s')-\pi_{\theta*}(a|s'))Q^{\pit}(s',a) \right)\nn\\
    &\leq \sum_{s'\in\mcs} \frac{l^*_\theta(s')}{l_\theta(s')} l_\theta(s') \max_{\Bar{\pi}_{s}'(\cdot|s')\in \Delta(\mca)} \langle \pi_\theta(\cdot|s')-\Bar{\pi}_{s'}(\cdot|s'),
    Q^{\pit}(s',\cdot)\rangle\nn\\
    &\leq C_{PL} \sum_{s'\in\mcs} l_\theta(s') \max_{\Bar{\pi}_{s'}(\cdot|s')\in \Delta(\mca)} \langle \pi_\theta(\cdot|s')-\Bar{\pi}_{s'}(\cdot|s'),
    Q^{\pit}(s',\cdot)\rangle.
\end{align}
If we denote $\pi_{s'}(\cdot|{s'})\triangleq \arg\max_{\Bar{\pi}_{s'}(\cdot|{s'})\in \Delta(\mca)}\langle \pi_\theta(\cdot|s')-\Bar{\pi}_{s'}(\cdot|s'),
    Q^{\pit}(s',\cdot)\rangle$, then 
\begin{align}
    &\sum_{{s'}\in\mcs} l_\theta({s'}) \max_{\Bar{\pi}_{s'}(\cdot|{s'})\in \Delta(\mca)} \langle \pi_\theta(\cdot|{s'})-\Bar{\pi}_{s'}(\cdot|{s'}), Q^{\pit}({s'},\cdot)\rangle\nn\\
    &=\sum_{{s'}\in\mcs} l_\theta({s'}) \langle \pi_\theta(\cdot|{s'})-{\pi}_{s'}(\cdot|{s'}), Q^{\pit}({s'},\cdot)\rangle\nn\\
    &=\sum_{{s'}\in\mcs} l_\theta({s'}) \sum_{a\in\mca} (\pi_\theta(a|{s'})-{\pi}_{s'}(a|{s'}))Q^{\pit}({s'},a)\nn\\
    &=\sum_{{s'}\in\mcs}  \sum_{a\in\mca} (\pi_\theta(a|{s'})-{\pi}_{s'}(a|{s'}))l_\theta({s'})Q^{\pit}({s'},a)\nn\\
    &=\left\langle \pi_\theta-\Bar{\pi}, l_\theta(\cdot)Q^{\pit}(\cdot,\cdot)  \right\rangle,
\end{align}
where $\pi_\theta=\left(\pi_\theta(\cdot|s_1),...,\pi_\theta(\cdot|s_{|\mcs|}) \right)^\top\in(\Delta(\mca))^{|\mcs|}$, $\Bar{\pi}=\left(\pi_{s_1}(\cdot|s_1),...,\pi_{s_{|\mcs|}}(\cdot|s_{|\mcs|}) \right)^\top\in(\Delta(\mca))^{|\mcs|}$,  $l_\theta(\cdot)Q^{\pit}(\cdot,\cdot)\in \mathbb{R}^{|\mcs|\times|\mca|}$ and $\left(l_\theta(\cdot)Q^{\pit}(\cdot,\cdot)\right)_{{s'},a}=l_\theta({s'})Q^{\pit}({s'},a)$. Clearly, we have that 
\begin{align}
    \left\langle \pi_\theta-\Bar{\pi}, l_\theta(\cdot)Q^{\pit}(\cdot,\cdot)  \right\rangle \leq \max_{\hat{\pi}\in (\Delta(\mca))^{|\mcs|}} \left\langle \pi_\theta-\hat{\pi}, l_\theta(\cdot)Q^{\pit}(\cdot,\cdot)  \right\rangle.
\end{align}
Note that $\psi_\mu(\theta)_{s,b}=l_\theta(s)Q^{\pit}(s,b)$, hence combining all the inequalities above implies that
\begin{align}
    J_\rho(\theta)-J_\rho^*\leq C_{PL} \max_{\hat{\pi}\in (\Delta(\mca))^{|\mcs|}} \left\langle \pi_\theta-\hat{\pi}, \psi_\mu(\theta)  \right\rangle,
\end{align}
which completes the proof.
 
\subsection{Proof of Theorem \ref{thm:conv}: Convergence and Global Optimality}\label{sec:glo}
The theorem can be proved following the standard results of stochastic differential inclusion in \cite{ruszczynski2020convergence,majewski2018analysis,borkar2009stochastic,borkar2000ode}. 
\begin{lemma}(Theorem 4.1 in \cite{majewski2018analysis})\label{lemma:sub_conv}
If the sequence $\left\{\theta_t \right\}$ is generated by a stochastic algorithm $\theta_{t+1}\leftarrow \mathbf{ \prod}_K (\theta_t-\alpha_t Y_t)$, where $Y_t\in F(\theta_t)$ for some set-value function $F$, where  $F(\theta)=-\partial_C f(\theta)$ for some function $f$, where $\partial_C$ denotes the Clarke sub-gradient \cite{clarke1990optimization}. Denote the normal cone $N_K(\theta)\triangleq \left\{g\in\mathbb{R}^d: \langle g,\theta'-\theta\rangle\leq 0 \text{ for any }\theta'\in K \right\}$.
If 

(1). the set-valued map $\theta\mapsto F(\theta)$ is upper hemicontinuous, convex-compact valued and locally bounded;

(2). the step-sizes satisfy $\alpha_t>0$, $\sum^{\infty}_t \alpha_t=\infty$ and $\sum^{\infty}_t \alpha_t^2<\infty$;

(3). denote the stationary points $\mathcal{X}^*\triangleq \left\{ x: 0\in\partial_C f(x)+N_K(x) \right\}$, and  $f(\mathcal{X}^*)$ has no interior;

then  $\theta_t$ convergence to $\mathcal{X}^*$ almost surely.
\end{lemma}

Note that  according to \cite{clarke1990optimization,bolte2007clarke}, $\partial J_\mu({\theta}) \subseteq\partial_C J_\mu({\theta})$, hence a Fr\'{e}chet sub-gradient is also a Clarke sub-gradient, and our robust policy gradient method can be viewed as a Clarke sub-gradient descent method. 

It is shown in Proposition 2.1.5 and Proposition 2.1.2 in \cite{clarke1990optimization} that the Clarke sub-gradient of a Lipschitz function is convex-compact valued, locally bounded, and lower hemicontinuous. Hence the map $-\partial J_C^\mu(\theta)$ satisfies assumption (1). The assumption (2) is easy to satisfy, e.g., $\alpha_t=\frac{1}{t+1}$.

Before verifying Assumption (3), we first show a lemma.
\begin{lemma}\label{lemma:sta}
For any $\hat{\theta}\in\left\{\theta: 0\in\partial J_C^\mu(\theta) \right\}$, $J_\rho(\hat{\theta})=J_\rho^*$.
\end{lemma}
\begin{proof}
From Rademacher's theorem \cite{federer2014geometric}, we know that $J_\mu(\theta)$ is differentiable almost everywhere. And from \cite{kruger2003frechet}, we have that if $J_\mu(\theta)$ is differentiable at $\theta$, then $\partial J_\mu(\theta)=\left\{ \nabla J_\mu(\theta) \right\}$ and hence $\psi_\mu(\theta)=\nabla J_\mu(\theta)$.  

Now consider any $\hat{\theta}\in\left\{\theta: 0\in\partial J_C^\mu(\theta) \right\}$. From \cite{clarke1990optimization,zhang2020complexity}, the Clarke sub-gradient can be viewed as a convex hull of limit gradients:
\begin{align}
    \partial_C J_\mu(\hat{\theta})=\text{conv}\left(\left\{g: \exists \theta_n\to \hat{\theta}, \text{ s.t. }\nabla J_\mu(\theta_n)\to g  \right\}\right).
\end{align} 
Denote  $S\triangleq \left\{g: \exists \theta_n\to \hat{\theta}, \text{ s.t. }\nabla J_\mu(\theta_n)\to g  \right\}$. Note that if $0\in \text{conv}(S)$, then $\exists \lambda_i\in[0,1] \text{ and } g_i\in S, i=1,...,k$, such that $\sum_i^k \lambda_i=1$ and $0=\sum_i^k \lambda_i g_i$. Note that $g_i\in S$ is a limit of a sequence of gradient $\left\{ \psi_\mu(\theta_i^n) \right\}_n$ where $\theta^n_i \to \hat{\theta}$.  However, note that from \eqref{eq:subgrad}, every entry of $\psi_\mu(\theta)$ is non-negative for any $\theta$,  and hence $g_i$ also has non-negative entries. 

This implies that  $\exists j\in\left\{ 1,...,k\right\}$, such that $g_j=0$. This further implies that there exists a sequence ${\theta}^t_j \to \hat{\theta}$, with $\psi_\mu(\theta^t_j)=\nabla J_\mu({\theta}^t_j)\to 0$, i.e., $\|\nabla J_\mu(\theta^t_j)\|\to 0$. 

From the PL-condition in Theorem \ref{thm:PL}, we know that $J_\rho(\theta^t_j)-J_\rho^*\leq C_{PL}|\mcs||\mca|\|\nabla J_\mu(\theta^t_j)\| \to 0$ and hence $J_\rho(\theta^t_j) \to J_\rho^*$. As $J_\rho(\theta)$ is a continuous function, $J_\rho(\theta^t_j)\to J_\rho(\hat{\theta})$ and hence $J_\rho(\hat{\theta})=J_\rho^*$. This means  any $\hat{\theta}\in\left\{\theta: 0\in\partial_C J_\mu(\theta) \right\}$ is a global optimal point, i.e., $J_\rho(\hat{\theta})=J_\rho^*$.
\end{proof}

We then verify the last assumption (3). First note that from the definition of Clarke sub-differential, it can be verified that $\mathcal{X}^*\subseteq \left\{\theta: 0\in\partial_C f(\theta)  \right\}$. Hence from Lemma \ref{lemma:sta}, $\mathcal{X}^*$ is a subset of global optimal points of $J_\mu$ (set $\rho=\mu$ in Lemma \ref{lemma:sta}), hence $J_\mu(\mathcal{X}^*)=\left\{ J_\mu^*\right\}$ is a singleton and  contains no interior.  

Hence the lemma \ref{lemma:sub_conv} holds for our robust policy gradient algorithm. Then, by the lemma, any sequence $\left\{\theta_t\right\}$ generated by Algorithm \ref{alg:rpgfull} converges to a stationary point a.s., i.e., $\lim_t {\theta_t}\in \left\{\theta: 0\in\partial J_C^\mu(\theta)+N_{(\Delta(\mca))^{|\mcs|}}(\theta) \right\}\subseteq \left\{\theta: 0\in\partial J_C^\mu(\theta) \right\}$ almost surely.

Note that Lemma \ref{lemma:sta} shows that any $\theta\in \left\{\theta: 0\in\partial J_C^\mu(\theta) \right\}$ is a global optimal point of $J_\rho$, and this implies that the sequence $\left\{\theta_t \right\}$ generated by the Algorithm \ref{alg:rpgfull} converges to the global optimum of $J_\rho$, i.e.,  $J_\rho(\theta_t)\to J_\rho^*$ almost surely. And hence this completes the proof of Theorem \ref{thm:conv}.

\section{Smoothed Robust Policy Gradient}
In the remaining parts, if not specified, we omit $\sigma$ in $\lse$ function, i.e., denote $\lse(V)\triangleq\lse(\sigma,V)$.

\subsection{Proof of Theorem \ref{thm:gradJsigma}: Gradient of  $J_\sigma$}
In this section we derive the gradient of $J_\sigma(\theta)\triangleq\sum_{{s'}\in\mcs} \rho({s'})V^{\pi_{\theta}}_{\sigma}({s'})$, and prove Theorem \ref{thm:gradJsigma}. 
 
From the fact that $V_\sigma^{\pit}(s)=\sum_{a\in\mca} \pit(a|s) Q^{\pit}_{\sigma}(s,a)$, we have that
\begin{align}\label{eq:grdV1}
    \nabla \sv^{\pit}(s)&=\nabla \left(\sum_{a\in\mca} \pi_\theta(a|s) \sq^{\pit}(s,a) \right)\nn\\
    &=\sum_{a\in\mca} \nabla \pit(a|s)\sq^{\pit}(s,a)+\sum_{a\in\mca} \pit(a|s)\nabla \sq^{\pit}(s,a)\nn\\
    &\overset{(a)}{=}\sum_{a\in\mca} \nabla \pit(a|s)\sq^{\pit}(s,a)+\sum_{a\in\mca} \pit(a|s)\nabla \left(c(s,a)+\gamma (1-R)\sum_{s'\in\mcs}p^a_{s,{s'}} \sv^{\pit}({s'})+\gamma R\cdot \lse(\sv^{\pit}) \right)\nn\\
    &=\sum_{a\in\mca} \nabla \pit(a|s)\sq^{\pit}(s,a)+\sum_{a\in\mca} \pit(a|s) \left(\gamma (1-R)\sum_{s'\in\mcs}p^a_{s,{s'}} \nabla\sv^{\pit}({s'})+\gamma R\cdot \nabla\lse(\sv^{\pit}) \right)\nn\\
    &=\sum_{a\in\mca} \nabla \pit(a|s)\sq^{\pit}(s,a)+\gamma R\cdot \nabla\lse(\sv^{\pit})+\sum_{a\in\mca} \pit(a|s) \left(\gamma (1-R)\sum_{s'\in\mcs}p^a_{s,{s'}} \nabla\sv^{\pit}({s'})\right) \nn\\
    &=\sum_{a\in\mca} \nabla \pit(a|s)\sq^{\pit}(s,a)+\gamma R\cdot \nabla\lse(\sv^{\pit})+\sum_{a\in\mca} \pit(a|s) \left(\gamma (1-R)\sum_{s'\in\mcs}p^a_{s,{s'}} \nabla\left(\sum_{a'\in\mca} \pi_\theta(a'|{s'}) \sq^{\pit}({s'},a') \right)\right) \nn\\
    &=\sum_{a\in\mca} \nabla \pit(a|s)\sq^{\pit}(s,a)+\gamma R\cdot \nabla\lse(\sv^{\pit})+\gamma (1-R) \sum_{a\in\mca} \pit(a|s)\sum_{s'\in\mcs}p^a_{s,{s'}}\sum_{a'}\nabla \pit(a'|{s'}) \sq^{\pit}({s'},a')\nn\\
    &\quad+\gamma (1-R) \sum_{a\in\mca} \pit(a|s)\sum_{s'\in\mcs}p^a_{s,{s'}}\sum_{a'}\pit(a'|{s'})\nabla \sq^{\pit}({s'},a')\nn\\
    &=\sum_{a\in\mca} \nabla \pit(a|s)\sq^{\pit}(s,a)+\gamma R\cdot \nabla\lse(\sv^{\pit})+\gamma (1-R)\sum_{s'\in\mcs} \mathbb{P}(S_1={s'}|S_0=s,\pit)\sum_{a'}\nabla \pit(a'|{s'}) \sq^{\pit}({s'},a')\nn\\
    &\quad+\gamma (1-R)\sum_{s'\in\mcs} \mathbb{P}(S_1={s'}|S_0=s,\pit)\sum_{a'}\pit(a'|{s'})\nabla \sq^{\pit}({s'},a')\nn\\
    &=...\nn\\
    &\overset{(b)}{=}\sum_{k=0}^\infty \gamma^k(1-R)^k \sum_{{s'}\in\mcs} \mathbb{P}(S_k={s'}|S_0=s,\pit)\sum_{a\in\mca} \nabla \pit(a|s') \sq^{\pit}({s'},a)+\frac{\gamma R}{1-\gamma+\gamma R} \nabla \lse(\sv^{\pit})\nn\\
    &=\frac{1}{1-\gamma+\gamma R} \sum_{{s'}\in\mcs} d^{\pi}_s({s'})\sum_{a\in\mca} \nabla \pit(a|s') \sq^{\pit}({s'},a)+\frac{\gamma R}{1-\gamma+\gamma R} \nabla \lse(\sv^{\pit}),
\end{align}
where $(a)$ is from the definition of $Q_\sigma^{\pit}$, and $(b)$ is to recursively apply the previous steps.

To simplify the notations, we denote $B(s,\theta)\triangleq\frac{1}{1-\gamma+\gamma R} \sum_{{s'}\in\mcs} d^{\pit}_s({s'})\sum_{a\in\mca} \nabla \pit(a|{s'}) \sq^{\pit}({s'},a)$. We then study the second term $\nabla \lse(V^{\pit}_\sigma)$ in \eqref{eq:grdV1}. Note that 
\begin{align}\label{eq:grdV2}
    \nabla \lse(V_\sigma^{\pit})&=\nabla \left(\frac{1}{\sigma}\log\left(\sum_{s\in\mcs} e^{\sigma \sv^{\pit}(s)}\right) \right)\nn\\
    &=\frac{1}{\sigma} \frac{\nabla \left(\sum_{s\in\mcs} e^{\sigma \sv^{\pit}(s)} \right)}{\sum_{s\in\mcs} e^{\sigma \sv^{\pit}(s)}}\nn\\
    &=\frac{1}{\sigma}\frac{\sum_{s\in\mcs} e^{\sigma \sv^{\pit}(s)} \sigma \nabla \sv^{\pit}(s)}{\sum_{s\in\mcs} e^{\sigma \sv^{\pit}(s)}}\nn\\
    &=\frac{\sum_{s\in\mcs} e^{\sigma \sv^{\pit}(s)}\nabla \sv^{\pit}(s)}{\sum_{s\in\mcs} e^{\sigma \sv^{\pit}(s)}}.
\end{align}
Then we plug \eqref{eq:grdV1} in \eqref{eq:grdV2}, we have that
\begin{align}
    \nabla \lse(\sv^{\pit})&=\frac{1}{\sum_{s\in\mcs} e^{\sigma \sv^{\pit}(s)}}\sum_{s\in\mcs} \left( e^{\sigma \sv^{\pit}(s)} \left( B(s,\theta)+\frac{\gamma R}{1-\gamma+\gamma R} \nabla \lse(\sv^{\pit})\right) \right)\nn\\
    &=\frac{\sum_{s\in\mcs} e^{\sigma \sv^{\pit}(s)} B(s,\theta)}{\sum_{s\in\mcs} e^{\sigma \sv^{\pit}(s)}}+\frac{\gamma R}{1-\gamma+\gamma R} \nabla \lse(\sv^{\pit}).
\end{align}
This implies that
\begin{align}
    \nabla \lse(\sv^{\pit})=\frac{1-\gamma+\gamma R}{1-\gamma}\frac{\sum_{s\in\mcs} e^{\sigma \sv^{\pit}(s)} B(s,\theta)}{\sum_{s\in\mcs} e^{\sigma \sv^{\pit}(s)}}.
\end{align}
Hence plugging it in \eqref{eq:grdV1} implies that
\begin{align}\label{eq:gradV}
    \nabla \sv^{\pit}(s)=B(s,\theta)+\frac{\gamma R}{1-\gamma}\frac{\sum_{s\in\mcs} e^{\sigma \sv^{\pit}(s)} B(s,\theta)}{\sum_{s\in\mcs} e^{\sigma \sv^{\pit}(s)}}.
\end{align}
And finally it is easy to see that
\begin{align} 
    \nabla J_\sigma(\theta)=B(\rho,\theta)+\frac{\gamma R}{1-\gamma}\frac{\sum_{s\in\mcs} e^{\sigma \sv^{\pit}(s)} B(s,\theta)}{\sum_{s\in\mcs} e^{\sigma \sv^{\pit}(s)}},
\end{align}
where $B(\rho,\theta)\triangleq\frac{1}{1-\gamma+\gamma R} \sum_{{s'}\in\mcs} d^{\pit}_\rho({s'})\sum_{a\in\mca} \nabla \pit(a|{s'}) \sq^{\pit}({s'},a)$.

\subsection{Proof of Lemma \ref{lemma:lipJsigma}: Smoothness of $J_\sigma$}
In this section we prove Lemma \ref{lemma:lipJsigma} that the smoothed robust value function $J_\sigma(\theta)$ is $L_\sigma$-smooth in $\theta$. 
It has been shown in \cite{wang2021online} that $\|V_{\sigma}^{\pi}-V^{\pi}\|_{\infty}\leq \frac{\gamma R}{1-\gamma} \frac{\log |\mcs|}{\sigma}$. Moreover we have that 
\begin{align}
\|Q^\pi_\sigma-Q^\pi\|_\infty&=\max_{s,a} |\sum_{s'\in\mcs} \gamma (1-R) p^a_{s,s'} (V^\pi_\sigma(s')-V^\pi(s')) +\gamma R (\lse(V^\pi_\sigma)-\max V^\pi)  |\nn\\
&\leq \gamma (1-R) \|V_{\sigma}^{\pi}-V^{\pi}\|_{\infty}+ |\gamma R (\lse(V^\pi_\sigma)-\max V^\pi)|\nn\\
&\leq \gamma (1-R)\frac{\gamma R}{1-\gamma} \frac{\log |\mcs|}{\sigma} +\gamma R|\lse(V^\pi_\sigma)-\lse(V^\pi)|+\gamma R |\lse(V^\pi)-\max V^\pi|     \nn\\
&\overset{(a)}{\leq} \gamma (1-R)\frac{\gamma R}{1-\gamma} \frac{\log |\mcs|}{\sigma} +\gamma R \|V_{\sigma}^{\pi}-V^{\pi}\|_{\infty} +\gamma R |\lse(V^\pi)-\max V^\pi| \nn\\
&\overset{(b)}{\leq} \frac{\gamma^2 R}{1-\gamma} \frac{\log |\mcs|}{\sigma}+\gamma R \frac{\log|\mcs|}{\sigma}\nn\\
&=\frac{\gamma R}{1-\gamma} \frac{\log|\mcs|}{\sigma},
\end{align}
where $(a)$ is from the fact that $|\lse(V_1)-\lse(V_2)|\leq \|V_1-V_2\|_\infty$ ((59) in \cite{wang2021online}), and $(b)$ is because $\lse(V)-\max V \leq  \frac{\log|S|}{\sigma}$ ((61) in \cite{wang2021online}).

Note that if we define a new cost function $c'(s,a)\triangleq c(s,a)+ \gamma R\frac{\log |\mcs|}{\sigma}$, the robust action-value function w.r.t. new cost function $c'$ is  $Q^{\pi}_{c'}(s,a)=\max_{\kappa} \mathbb{E}_{\kappa}[\sum^\infty_{t=0} \gamma^t c'(S_t,A_t)^t | S_0=s,A_0=a,\pi]\geq \frac{\gamma R}{1-\gamma} \frac{\log |\mcs|}{\sigma}$ for any $s\in\mcs, a\in\mca, \pi$. Hence the smoothed robust value function w.r.t. $c'$ is non-negative:  $Q^\pi_{c',\sigma}(s,a)\geq Q_{c'}^\pi(s,a)-\frac{\gamma R}{1-\gamma} \frac{\log|\mcs|}{\sigma}\geq 0$ for any $s\in\mcs, a\in\mca, \pi$.

This means by define a new cost function $c'$, $Q^\pi_{c',\sigma}$ and $V^\pi_{c',\sigma}$ are non-negative. In the remaining parts, we omit the subscript $c'$ and denote the smoothed robust value functions w.r.t. $c'$ by $Q^\pi_\sigma$ and $V^\pi_\sigma$.

On the other had, the upper bounds on them can be easily derived:
\begin{align}\label{eq:csigma}
    Q^\pi_\sigma(s,a)&\leq Q_{c'}^\pi(s,a)+\frac{\gamma R}{1-\gamma} \frac{\log|\mcs|}{\sigma} \leq \frac{1}{1-\gamma}(1+\gamma R \frac{\log|\mcs|}{\sigma})+ \frac{\gamma R}{1-\gamma} \frac{\log|\mcs|}{\sigma}\leq \frac{1}{1-\gamma}(1+2\gamma R \frac{\log|\mcs|}{\sigma}),\\
    V^\pi_\sigma(s)&=\sum_{a\in\mca} \pi(a|s)Q^\pi_\sigma(s,a) \leq \frac{1}{1-\gamma}(1+2\gamma R \frac{\log|\mcs|}{\sigma}).
\end{align}
We denote these upper bounds by $C_\sigma$. 
Hence we have showed that $0\leq Q^\pi_\sigma(s,a)\leq C_\sigma$ and $0\leq V^\pi_\sigma(s)\leq C_\sigma$ for any $s\in\mcs, a\in\mca, \pi$.

We then prove Lemma \ref{lemma:lipJsigma}.
We first show that $B(s,\theta)$ is Lipschitz. From \eqref{eq:gradV}, we know that
\begin{align}
    \|\nabla \sv^{\pit}(s)\|\leq \max_{s\in\mcs}\|B(s,\theta)\|\left(1+\frac{\gamma R}{1-\gamma}\right)\leq \left(\frac{1-\gamma+\gamma R}{1-\gamma} \right)\frac{1}{1-\gamma+\gamma R}|\mca|k_\pi\sup_{s,a}|\sq^{\pit}(s,a)|,
\end{align}
which is from $\|B(s,\theta)\|\leq \frac{1}{1-\gamma+\gamma R} |\mca| k_\pi \sup_{s,a}|\sq^{\pit}(s,a)|$ for any $\theta\in\Theta, s\in\mcs$. 

As we showed in \eqref{eq:csigma}, $\|\sq^{\pit}\|_\infty\leq C_\sigma$. Hence the gradient of $\sv^{\pit}$ can be bounded as
\begin{align}\label{eq:gradVbound}
    \|\nabla \sv^{\pit}(s)\|\leq\frac{1}{1-\gamma} |\mca|k_\pi C_\sigma\triangleq C^V_\sigma,
\end{align}
and this means $\sv^{\pit}(s) $ is $C^V_\sigma$-Lipschitz for any $s\in\mcs$. Moreover we have that
\begin{align}
    |\sq^{\pi_{\theta_1}}(s,a)-\sq^{\pi_{\theta_2}}(s,a)|&=|\sum_{s'\in\mcs} p^a_{s,s'} \gamma (1-R) (\sv^{\pi_{\theta_1}}(s')-\sv^{\pi_{\theta_2}}(s')) +\gamma R (\lse(\sv^{\pi_{\theta_1}})-\lse(\sv^{\pi_{\theta_2}})) |\nn\\
    &\overset{(a)}{\leq} \gamma (1-R) C^V_\sigma\|\theta_1-\theta_2\| +\gamma R C^V_\sigma\|\theta_1-\theta_2\|\nn\\
    &= C^V_\sigma\|\theta_1-\theta_2\|,
\end{align}
where inequality $(a)$ is from  $|\lse(V_1)-\lse(V_2)|\leq \|V_1-V_2\|_\infty$ and $\sv^{\pit}(s) $ is $C^V_\sigma$-Lipschitz. 

It was showed in \cite{achiam2017constrained,touati2020stable} that 
\begin{align}
    d_{\text{TV}}(d^{\pi_1}_s,d^{\pi_2}_s)\leq \frac{2\gamma (1-R)}{1-\gamma+\gamma R} \mE_{S\sim d^{\pi_2}_s}\left[d_{\text{TV}}(\pi_1(\cdot|S),\pi_2(\cdot|S))\right],
\end{align}
hence we have that for any function $f(s,\theta)$ defined on $\mcs\times\Theta$, 
\begin{align}\label{eq:91}
    \|\mE_{S\sim d_s^{\pi_{\theta_1}}}[f(S,\theta)]-\mE_{S\sim d_s^{\pi_{\theta_2}}}[f(S,\theta)]\|& \leq 2d_{\text{TV}}(d_s^{\pi_{\theta_1}},d_s^{\pi_{\theta_1}}) \sup_{s\in\mcs,\theta\in\Theta}\|f(s,\theta)\|\nn\\
    &\leq \frac{4\gamma (1-R)}{1-\gamma+\gamma R} \frac{|\mca|}{2}k_\pi\|\theta_1-\theta_2\| \sup_{s\in\mcs,\theta\in\Theta}\|f(s,\theta)\|,
\end{align}
where the last inequality is from 
\begin{align}
    d_{\text{TV}}(\pone(\cdot|s),\ptwo(\cdot|s))=\frac{1}{2}\sum_{a\in\mca} |\pone(a|s)-\ptwo(a|s)| \leq \frac{|\mca|}{2}k_\pi \|\theta_1-\theta_2\|.
\end{align}
This implies that
\begin{align}\label{eq:B_diff}
   \| B(s,\theta_1)-B(s,\theta_2)\|&=\frac{1}{1-\gamma+\gamma R} \|\mE_{S\sim d_s^{\pi_{\theta_1}}}[f(S,\theta_1)]-\mE_{S\sim d_s^{\pi_{\theta_2}}}[f(S,\theta_2)] \|\nn\\
   &\leq \frac{1}{1-\gamma+\gamma R} \|\mE_{S\sim d_s^{\pi_{\theta_1}}}[f(S,\theta_1)]-\mE_{S\sim d_s^{\pi_{\theta_1}}}[f(S,\theta_2)]+\mE_{S\sim d_s^{\pi_{\theta_1}}}[f(S,\theta_2)]-\mE_{S\sim d_s^{\pi_{\theta_2}}}[f(S,\theta_2)] \|\nn\\
   &\leq \frac{1}{1-\gamma +\gamma R} \left(\mE_{S\sim d^{\theta_1}_s} \|f(S,\theta_1)-f(S,\theta_2)\|+\frac{2k_\pi|\mca|\gamma (1-R)}{1-\gamma+\gamma R}\sup_{x\in\mcs,\theta\in\Theta}\|f(x,\theta)\| \|\theta_1-\theta_2\|\right),
\end{align}
where $f(x,\theta)=\sum_{a\in\mca} \nabla \pit(a|x)\sq^{\pit}(x,a)$, and the last inequality is from \eqref{eq:91}.

Note that for any $s\in\mcs$ and $\theta\in\Theta$, 
\begin{align}\label{eq:fbound}
\|f(s,\theta)\|=\left\|\sum_{a\in\mca} \nabla \pit(a|x)\sq^{\pit}(x,a)\right\|\leq |\mca|k_\pi C_\sigma,
\end{align}
and
\begin{align}\label{eq:flip}
    \|f(x,\theta_1)-f(x,\theta_2)\|& = \left\|\sum_{a\in\mca}\left(\nabla\pone(a|x)\sq^{\pone}(x,a)- \nabla\ptwo(a|x)\sq^{\ptwo}(x,a)\right)\right\| \nn\\
    &\leq \left\|\sum_{a\in\mca}\left(\nabla\pone(a|x)\sq^{\pone}(x,a)- \nabla\ptwo(a|x)\sq^{\pone}(x,a)\right)\right\| \nn\\
    &\quad+ \left\|\sum_{a\in\mca}\left(\nabla\ptwo(a|x)\sq^{\pone}(x,a)- \nabla\ptwo(a|x)\sq^{\ptwo}(x,a)\right)\right\|\nn\\
    &\leq |\mca|C_\sigma l_\pi \|\theta_1-\theta_2\|+|\mca|k_\pi C^V_\sigma \|\theta_1-\theta_2\|\nn\\
    &=\left( |\mca|C_\sigma l_\pi+|\mca|k_\pi C^V_\sigma\right)\|\theta_1-\theta_2\|,
\end{align}
where the last inequality is because $\nabla \pit(a|x)$ is $l_\pi$-Lipschitz, and $\sq^{\pit}(x,a)$ is $C^V_\sigma$-Lipschitz. 

Hence plugging \eqref{eq:fbound} and \eqref{eq:flip} in \eqref{eq:B_diff} implies that
\begin{align}
    \| B(s,\theta_1)-B(s,\theta_2)\| &\leq \frac{1}{1-\gamma +\gamma R} \left( |\mca|C_\sigma l_\pi+|\mca|k_\pi C^V_\sigma\right)\|\theta_1-\theta_2\|\nn\\
    &\quad+\frac{2|\mca|^2\gamma (1-R)}{(1-\gamma+\gamma R)^2}k_\pi^2C_\sigma\|\theta_1-\theta_2\|\nn\\
    &\triangleq k_B \|\theta_1-\theta_2\|,
\end{align}
Hence we showed that $B(s,\theta)$ and $B(\rho,\theta)$ are both $k_B$-Lipschitz.

Next we show that $\frac{\sum_{s\in\mcs} e^{\sigma \sv^{\pit}(s)} B(s,\theta)}{\sum_{s\in\mcs} e^{\sigma \sv^{\pit}(s)}}$ is also Lipschitz. Note that it is the 
inner production of $B(\theta)\triangleq (B(s_1,\theta),...,B(s_{|\mcs|},\theta))$ and $p^\theta_\sigma\triangleq\left(\frac{e^{\sigma \sv^{\pit}(s_1)}}{\sum_{s\in\mcs} e^{\sigma \sv^{\pit}(s)}},...,\frac{e^{\sigma \sv^{\pit}(s_{|\mcs|})}}{\sum_{s\in\mcs} e^{\sigma \sv^{\pit}(s)}}\right)$.  Hence 
\begin{align}
    \langle B(\theta_1),p_\sigma^{\theta_1}\rangle -\langle B(\theta_2),p_\sigma^{\theta_2}\rangle&=\langle B(\theta_1),p_\sigma^{\theta_1}\rangle-\langle B(\theta_2),p_\sigma^{\theta_1}\rangle+\langle B(\theta_2),p_\sigma^{\theta_1}\rangle-\langle B(\theta_2),p_\sigma^{\theta_2}\rangle\nn\\
    &= \langle B(\theta_1)-B(\theta_2),p^{\theta_1}_\sigma\rangle+\langle B(\theta_2), p^{\theta_1}_\sigma-p^{\theta_2}_\sigma \rangle\nn\\
    &\leq \|B(\theta_1)-B(\theta_2) \|+ \|B(\theta_2)\|\|p^{\theta_1}_\sigma-p^{\theta_2}_\sigma \|,
\end{align}
where the last inequality is because $\|p^{\theta}_\sigma\|\leq 1$ for any $\theta\in\Theta$.

The Lipschitz smoothness of $p^\theta_\sigma$ can be showed as follows. Note that
\begin{align}
    \nabla p^{\theta}_\sigma(j)&=\nabla \left(\frac{e^{\sigma \sv^{\pit}(j)}}{\sum_{s\in\mcs} e^{\sigma \sv^{\pit}(s)}} \right)\nn\\
    &=\frac{e^{\sigma \sv^{\pit}(j)}\sigma \nabla \sv^{\pit}(j)  \left(\sum_{s\in\mcs} e^{\sigma \sv^{\pit}(s)}\right)-\left(\sum_{s\in\mcs} e^{\sigma \sv^{\pit}(s)} \sigma\nabla \sv^{\pit}(s)\right) e^{\sv^{\pit}(j)}}{\left(\sum_{s\in\mcs} e^{\sigma \sv^{\pit}(s)} \right)^2}\nn\\
    &=\frac{\sigma e^{\sv^{\pit}(j)}\sum_{s\in\mcs}\left(e^{\sigma \sv^{\pit}(s)} \left( \nabla \sv^{\pit} (j)- \nabla \sv^{\pit}(s)\right) \right) }{\left(\sum_{s\in\mcs} e^{\sigma \sv^{\pit}(s)} \right)^2},
\end{align}
and hence 
\begin{align}
    \|\nabla p^{\theta}_\sigma(j)\|&\leq\frac{\sigma e^{\sv^{\pit}(j)} \cdot 2 \max_{s\in\mcs}\|\nabla \sv^{\pit}(s) \|}{\sum_{s\in\mcs}  e^{\sigma\sv^{\pit}(s)}}\nn\\
    &\leq 2\sigma C^V_\sigma,
\end{align}
which is from \eqref{eq:gradVbound}.

Thus we have $\|p^{\theta_1}_\sigma-p^{\theta_2}_\sigma \|= \sqrt{\sum_{s\in\mcs} \left(p^{\theta_1}_\sigma(s)-p^{\theta_2}_\sigma(s) \right)^2} \leq 2\sigma \sqrt{|\mcs|} C^V_\sigma \|\theta_1-\theta_2\|$. Moreover, note that 
\begin{align}
    \|B(\theta_1)-B(\theta_2) \|\leq \sqrt{|\mcs|}k_B\|\theta_1-\theta_2\|,
\end{align}
combining both inequality implies that
\begin{align}\label{eq:L-sigma}
    \left\| \frac{\sum_{s\in\mcs} e^{\sigma \sv^{\theta_1}(s)} B(s,\theta_1)}{\sum_{s\in\mcs} e^{\sigma \sv^{\theta_1}(s)}}- \frac{\sum_{s\in\mcs} e^{\sigma \sv^{\theta_2}(s)} B(s,\theta_2)}{\sum_{s\in\mcs} e^{\sigma \sv^{\theta_2}(s)}} \right\|\leq \left(\sqrt{|\mcs|}k_B+ 2\sigma|\mcs|C^V_\sigma \frac{1}{1-\gamma+\gamma R}k_\pi|\mca|C_\sigma\right)\|\theta_1-\theta_2\|,
\end{align}
where we use the fact that $\|B(s,\theta)\|\leq \frac{\sqrt{|\mcs|}}{1-\gamma+\gamma R} k_\pi|\mca|C_\sigma$ for any $s\in\mcs$ and $\theta\in\Theta$.

Hence recall the expression of $\nabla J_\sigma(\theta)$ in \eqref{eq:gradJ}, $\nabla J_\sigma(\theta)$ is Lipschitz with constant $k_B+\frac{\gamma R}{1-\gamma}\left(\sqrt{|\mcs|}k_B+ 2\sigma|\mcs|C^V_\sigma \frac{1}{1-\gamma+\gamma R}k_\pi|\mca|C_\sigma\right)\triangleq L_\sigma$. And this completes the proof.

\subsection{Proof of Theorem \ref{thm:smPL}: PL-Condition of Smoothed Robust Policy Gradient}\label{sec:smoothPL}
In the following, we will show that the smoothed objective function $J_\sigma(\theta)$ satisfies the PL-condition.

It can be shown that
\begin{align}\label{eq:spl1}
V_\sigma^{\pi_{\theta_1}}\left(s\right)-V_\sigma^{\pi_{\theta_2}}\left(s\right)&=\sum_{a\in\mca} \pone\left(a|s\right)Q_\sigma^{\pi_{\theta_1}}\left(s,a\right)-\sum_{a\in\mca} \ptwo\left(a|s\right)Q_\sigma^{\pi_{\theta_2}}\left(s,a\right)\nn\\
&=\sum_{a\in\mca} \left(\pone\left(a|s\right)-\ptwo\left(a|s\right)\right)Q_\sigma^{\pi_{\theta_1}}\left(s,a\right)+\sum_{a\in\mca} \ptwo\left(a|s\right) \left(Q_\sigma^{\pi_{\theta_1}}\left(s,a\right)-Q_\sigma^{\pi_{\theta_2}}\left(s,a\right)\right)\nn\\
&=\sum_{a\in\mca} \left(\pone\left(a|s\right)-\ptwo\left(a|s\right)\right)Q_\sigma^{\pi_{\theta_1}}\left(s,a\right)\nn\\
&\quad+\sum_{a\in\mca} \ptwo\left(a|s\right) \left(\gamma \left(1-R\right) \sum_{s'\in\mcs} p^a_{s,{s'}}\left(V_\sigma^{\pi_{\theta_1}}\left({s'}\right)-V_\sigma^{\pi_{\theta_2}}\left({s'}\right)\right)+\gamma R \left(\lse_1-\lse_2\right)\right)\nn\\
&\overset{(a)}{\leq}\sum_{a\in\mca} \left(\pone\left(a|s\right)-\ptwo\left(a|s\right)\right)Q_\sigma^{\pi_{\theta_1}}\left(s,a\right)\nn\\
&\quad+\gamma \left(1-R\right)\sum_{s'\in\mcs} \mathbb{P}\left(S_1={s'}|S_0=s,\ptwo\right) \left(V_\sigma^{\pi_{\theta_1}}\left({s'}\right)-V_\sigma^{\pi_{\theta_2}}\left({s'}\right)\right)+\gamma R \left(\lse_1-\lse_2\right)\nn\\
&\overset{(b)}{\leq}\frac{1}{1-\gamma+\gamma R} \sum_{{s'}\in\mcs} d^{\ptwo}_s\left({s'}\right) \sum_{a\in\mca} \left(\pone\left(a|{s'}\right)-\ptwo\left(a|{s'}\right)\right)Q_\sigma^{\pi_{\theta_1}}\left({s'},a\right)\nn\\
&\quad+\frac{\gamma R}{1-\gamma+\gamma R}\left(\lse_1-\lse_2\right),
\end{align}
where $(a)$ follows from the definition of $Q^\pi_\sigma$ and we denote $ \lse(V_\sigma^{\pi_{\theta_i}})$ by $\lse_i$; and $(b)$ follows by recursively apply the above steps. 

From $|\lse\left(V\right)-\max V|\leq \frac{\log|\mcs|}{\sigma}$, we further have 
\begin{align}
    \lse_1-\lse_2\leq \max  V_\sigma^{\pi_{\theta_1}}-\max V_\sigma^{\pi_{\theta_2}} +\frac{2\log |\mcs|}{\sigma}.
\end{align}
Hence \eqref{eq:spl1} can be further bounded as 
\begin{align}
V_\sigma^{\pi_{\theta_1}}\left(s\right)-V_\sigma^{\pi_{\theta_2}}\left(s\right)&\leq \frac{1}{1-\gamma+\gamma R} \sum_{{s'}\in\mcs} d^{\ptwo}_s\left({s'}\right) \sum_{a\in\mca} \left(\pone\left(a|{s'}\right)-\ptwo\left(a|{s'}\right)\right)Q_\sigma^{\pi_{\theta_1}}\left({s'},a\right)\nn\\
&\quad+\frac{\gamma R}{1-\gamma+\gamma R}\left(\frac{2\log |\mcs|}{\sigma} + V_\sigma^{\pi_{\theta_1}}\left(s^\sigma_{\theta_1}\right)-V_\sigma^{\pi_{\theta_2}}\left(s^\sigma_{\theta_2}\right)\right),
\end{align}
where $s^\sigma_{\theta_j}\in\arg\max_s \sv^{\pi_{\theta_j}}\left(s\right), j=1,2$. Note that $V_\sigma^{\pi_{\theta_1}}\left(s^\sigma_{\theta_1}\right)-V_\sigma^{\pi_{\theta_2}}\left(s^\sigma_{\theta_2}\right)=V_\sigma^{\pi_{\theta_1}}\left(s^\sigma_{\theta_1}\right)-V_\sigma^{\pi_{\theta_2}}\left(s^\sigma_{\theta_1}\right)+V_\sigma^{\pi_{\theta_2}}\left(s^\sigma_{\theta_1}\right)-V_\sigma^{\pi_{\theta_2}}\left(s^\sigma_{\theta_2}\right)\leq V_\sigma^{\pi_{\theta_1}}\left(s^\sigma_{\theta_1}\right)-V_\sigma^{\pi_{\theta_2}}\left(s^\sigma_{\theta_1}\right) $ and hence \eqref{eq:spl1} can be bounded as 
\begin{align}\label{eq:spl2}
V_\sigma^{\pi_{\theta_1}}\left(s\right)-V_\sigma^{\pi_{\theta_2}}\left(s\right)&\leq \frac{1}{1-\gamma+\gamma R} \sum_{{s'}\in\mcs} d^{\ptwo}_s\left({s'}\right) \sum_{a\in\mca} \left(\pone\left(a|{s'}\right)-\ptwo\left(a|{s'}\right)\right)Q_\sigma^{\pi_{\theta_1}}\left({s'},a\right)\nn\\
&\quad+\frac{\gamma R}{1-\gamma+\gamma R}\left(\frac{2\log |\mcs|}{\sigma} + V_\sigma^{\pi_{\theta_1}}\left(s^\sigma_{\theta_1}\right)-V_\sigma^{\pi_{\theta_2}}\left(s^\sigma_{\theta_1}\right)\right).
\end{align}
If we set $s=s^\sigma_{\theta_1}$ in this inequality, we have that
\begin{align}
    V_\sigma^{\pi_{\theta_1}}\left(s^\sigma_{\theta_1}\right)-V_\sigma^{\pi_{\theta_2}}\left(s^\sigma_{\theta_1}\right)&\leq \frac{1}{1-\gamma+\gamma R} \sum_{{s'}\in\mcs} d^{\ptwo}_{s^\sigma_{\theta_1}}\left({s'}\right) \sum_{a\in\mca} \left(\pone\left(a|{s'}\right)-\ptwo\left(a|{s'}\right)\right)Q_\sigma^{\pi_{\theta_1}}\left({s'},a\right)\nn\\
&\quad+\frac{\gamma R}{1-\gamma+\gamma R}\left(\frac{2\log |\mcs|}{\sigma} + V_\sigma^{\pi_{\theta_1}}\left(s^\sigma_{\theta_1}\right)-V_\sigma^{\pi_{\theta_2}}\left(s^\sigma_{\theta_1}\right)\right),
\end{align}
and hence
\begin{align}\label{eq:106}
    &V_\sigma^{\pi_{\theta_1}}\left(s^\sigma_{\theta_1}\right)-V_\sigma^{\pi_{\theta_2}}\left(s^\sigma_{\theta_1}\right)\nn\\
    &\leq \left(\frac{1-\gamma +\gamma R}{1-\gamma}\right)\left(\frac{\gamma R}{1-\gamma +\gamma R}\frac{2\log |\mcs|}{\sigma}+\frac{1}{1-\gamma+\gamma R} \sum_{{s'}\in\mcs} d^{\ptwo}_{s^\sigma_{\theta_1}}\left({s'}\right) \sum_{a\in\mca} \left(\pone\left(a|{s'}\right)-\ptwo\left(a|{s'}\right)\right)Q_\sigma^{\pi_{\theta_1}}\left({s'},a\right)\right).
\end{align}
Now plugging \eqref{eq:106} in \eqref{eq:spl2} implies that for any $s\in\mcs$
\begin{align}
    V_\sigma^{\pi_{\theta_1}}\left(s\right)-V_\sigma^{\pi_{\theta_2}}\left(s\right)&\leq \frac{1}{1-\gamma+\gamma R} \sum_{{s'}\in\mcs} d^{\ptwo}_s\left({s'}\right) \sum_{a\in\mca} \left(\pone\left(a|{s'}\right)-\ptwo\left(a|{s'}\right)\right)Q_\sigma^{\pi_{\theta_1}}\left({s'},a\right)\nn\\
    &\quad+\frac{\gamma R}{1-\gamma+\gamma R}\bigg(\frac{2\log |\mcs|}{\sigma} + \left(\frac{1-\gamma +\gamma R}{1-\gamma}\right)\cdot\bigg(\frac{\gamma R}{1-\gamma +\gamma R}\frac{2\log |\mcs|}{\sigma}\nn\\
    &\quad+\frac{1}{1-\gamma+\gamma R} \sum_{{s'}\in\mcs} d^{\ptwo}_{s^\sigma_{\theta_1}}\left({s'}\right) \sum_{a\in\mca} \left(\pone\left(a|{s'}\right)-\ptwo\left(a|{s'}\right)\right)Q_\sigma^{\pi_{\theta_1}}\left({s'},a\right)\bigg)\bigg)\nn\\
    &=\frac{1}{1-\gamma+\gamma R} \sum_{{s'}\in\mcs} d^{\ptwo}_s\left({s'}\right) \sum_{a\in\mca} \left(\pone\left(a|{s'}\right)-\ptwo\left(a|{s'}\right)\right)Q_\sigma^{\pi_{\theta_1}}\left({s'},a\right)\nn\\
    &\quad+\frac{\gamma R}{\left(1-\gamma\right)\left(1-\gamma+\gamma R\right)}\left(\sum_{{s'}\in\mcs} d^{\ptwo}_{s^\sigma_{\theta_1}}\left({s'}\right) \sum_{a\in\mca} \left(\pone\left(a|{s'}\right)-\ptwo\left(a|{s'}\right)\right)Q_\sigma^{\pi_{\theta_1}}\left({s'},a\right)\right)\nn\\
    &\quad+ \left(\frac{\gamma R}{1-\gamma}\right)\frac{2\log |\mcs|}{\sigma}.
\end{align}
Set $\theta_1=\theta$ and $\theta_2=\theta_\sigma^*\in\arg\max_{\theta\in(\Delta(\mca))^{|\mcs|}}J_\sigma(\theta)$, we have
\begin{align}
    \sv^{\pit}\left(s\right)-\sv^{\pi_{\theta_\sigma^*}}\left(s\right)&\leq \frac{1}{1-\gamma+\gamma R} \sum_{{s'}\in\mcs} d^{\pi_{\theta^*_\sigma}}_s\left({s'}\right) \sum_{a\in\mca} \left(\pi_\theta\left(a|{s'}\right)-\pi_{\theta^*_\sigma}\left(a|{s'}\right)\right)\sq^{\pit}\left({s'},a\right)\nn\\
&\quad+\frac{\gamma R}{\left(1-\gamma\right)\left(1-\gamma+\gamma R\right)}\left(\sum_{{s'}\in\mcs} d^{\pi_{\theta^*_\sigma}}_{s^\sigma_{\theta}}\left({s'}\right) \sum_{a\in\mca} \left(\pi_\theta\left(a|{s'}\right)-\pi_{\theta^*_\sigma}\left(a|{s'}\right)\right)\sq^{\pit}\left({s'},a\right)\right)\nn\\
    &\quad+ \left(\frac{\gamma R}{1-\gamma}\right)\frac{2\log |\mcs|}{\sigma}.
\end{align}
Recall the definition of $J_\sigma$, we can show that
\begin{align}
    J_\sigma\left(\theta\right)-J_\sigma^*&\leq \frac{1}{1-\gamma+\gamma R} \sum_{{s'}\in\mcs} d^{\pi_{\theta^*_\sigma}}_\rho\left({s'}\right) \sum_{a\in\mca} \left(\pi_\theta\left(a|{s'}\right)-\pi_{\theta^*_\sigma}\left(a|{s'}\right)\right)\sq^{\pit}\left({s'},a\right)\nn\\
&\quad+\frac{\gamma R}{\left(1-\gamma\right)\left(1-\gamma+\gamma R\right)}\left(\sum_{{s'}\in\mcs} d^{\pi_{\theta^*_\sigma}}_{s^\sigma_{\theta}}\left({s'}\right) \sum_{a\in\mca} \left(\pi_\theta\left(a|{s'}\right)-\pi_{\theta^*_\sigma}\left(a|{s'}\right)\right)\sq^{\pit}\left({s'},a\right)\right)\nn\\
    &\quad+ \left(\frac{\gamma R}{1-\gamma}\right)\frac{2\log |\mcs|}{\sigma}\nn\\
&=\sum_{{s'}\in\mcs} \left( \frac{1}{1-\gamma+\gamma R}d^{\pi_{\theta^*_\sigma}}_\rho\left({s'}\right)+\frac{\gamma R}{\left(1-\gamma\right)\left(1-\gamma+\gamma R\right)} d^{\pi_{\theta^*_\sigma}}_{s^\sigma_{\theta}}\left({s'}\right)\right)\sum_{a\in\mca} \left(\pi_\theta\left(a|{s'}\right)-\pi_{\theta^*_\sigma}\left(a|{s'}\right)\right)\sq^{\pit}\left({s'},a\right)\nn\\
&\quad+\left(\frac{\gamma R}{1-\gamma}\right)\frac{2\log |\mcs|}{\sigma}\nn\\
&\triangleq \sum_{{s'}\in\mcs} l^*_\theta\left({s'}\right)\sum_{a\in\mca} \left(\pi_\theta\left(a|{s'}\right)-\pi_{\theta^*_\sigma}\left(a|{s'}\right)\right)\sq^{\pit}\left({s'},a\right) +\left(\frac{\gamma R}{1-\gamma}\right)\frac{2\log |\mcs|}{\sigma},
\end{align}
where $l^*_\theta\left({s'}\right)=\frac{1}{1-\gamma+\gamma R}d^{\pi_{\theta^*_\sigma}}_\rho\left({s'}\right)+\frac{\gamma R}{\left(1-\gamma\right)\left(1-\gamma+\gamma R\right)} d^{\pi_{\theta^*_\sigma}}_{s^\sigma_{\theta}}\left({s'}\right)$. Recall that 
\begin{align}
    \left(\nabla J_\sigma\left(\theta\right)\right)_{s,a}&=B\left(\rho,\theta\right)_{s,a}+\frac{\gamma R}{1-\gamma}\frac{\sum_{s'\in\mcs} e^{\sigma \sv^{\pit}\left(s'\right)} B\left(s,\theta\right)_{s',a}}{\sum_{s\in\mcs} e^{\sigma \sv^{\pit}\left(s\right)}}\nn\\
    &=\frac{1}{1-\gamma+\gamma R} d^{\pi_\theta}_\rho\left(s\right)\sq^{\pit}\left(s,a\right)+ \frac{\gamma R}{1-\gamma}\frac{\sum_{{s'}\in\mcs} e^{\sigma \sv^{\pit}\left({s'}\right)}\frac{1}{1-\gamma+\gamma R} d^\pi_{s'}\left(s\right)\sq^{\pit}\left(s,a\right)}{\sum_{s\in\mcs} e^{\sigma \sv^{\pit}\left(s\right)}}\nn\\
    &=\left(\frac{1}{1-\gamma+\gamma R} d^{\pi_\theta}_\rho\left(s\right)+\frac{\gamma R}{1-\gamma}\frac{\sum_{{s'}\in\mcs} e^{\sigma \sv^{\pit}\left({s'}\right)}\frac{1}{1-\gamma+\gamma R} d^\pi_{s'}\left(s\right)}{\sum_{s\in\mcs} e^{\sigma \sv^{\pit}\left(s\right)}} \right)\sq^{\pit}\left(s,a\right)\nn\\
    &\triangleq l_\theta\left(s\right)\sq^{\pit}\left(s,a\right),
\end{align}
where $l_\theta\left(s\right)=\frac{1}{1-\gamma+\gamma R} d^{\pi_\theta}_\rho\left(s\right)+\frac{\gamma R}{1-\gamma}\frac{\sum_{{s'}\in\mcs} e^{\sigma \sv^{\pit}\left({s'}\right)}\frac{1}{1-\gamma+\gamma R} d^\pi_{s'}\left(s\right)}{\sum_{s\in\mcs} e^{\sigma \sv^{\pit}\left(s\right)}} $. It can be verified that
\begin{align}
   \sup_{s\in\mcs} \frac{l^*_\theta\left(s\right)}{l_\theta\left(s\right)}=\sup_{s\in\mcs}\frac{\frac{1}{1-\gamma+\gamma R}d^{\pi_{\theta^*_\sigma}}_\rho\left(s\right)+\frac{\gamma R}{\left(1-\gamma\right)\left(1-\gamma+\gamma R\right)} d^{\pi_{\theta^*_\sigma}}_{s^\sigma_{\theta}}\left(s\right)}{\frac{1}{1-\gamma+\gamma R} d^{\pi_\theta}_\rho\left(s\right)+\frac{\gamma R}{1-\gamma}\frac{\sum_{{s'}\in\mcs} e^{\sigma \sv^{\pit}\left({s'}\right)}\frac{1}{1-\gamma+\gamma R} d^\pi_{s'}\left(s\right)}{\sum_{{s'}\in\mcs} e^{\sigma \sv^{\pit}\left({s'}\right)}}}\leq \frac{1}{\left(1-\gamma\right)\rho_{\min}}=C_{PL},
\end{align}
where the last inequality is from $d^{\pi_\theta}_\rho\geq (1-\gamma+\gamma R)\rho_{\min}$, and $\frac{\sum_{{s'}\in\mcs} e^{\sigma \sv^{\pit}\left({s'}\right)}\frac{1}{1-\gamma+\gamma R} d^\pi_{s'}\left(s\right)}{\sum_{{s'}\in\mcs} e^{\sigma \sv^{\pit}\left({s'}\right)}}\geq 0$.

Hence similar to Theorem \ref{thm:PL}, we show that
\begin{align}
    J_\sigma\left(\theta\right)-J_\sigma^*\leq C_{PL} \max_{\hat{\pi}\in \left(\Delta\left(|\mca|\right)\right)^{|\mcs|}} \left\langle \pi_\theta-\hat{\pi}, \nabla J_\sigma\left(\theta\right) \right\rangle+\left(\frac{\gamma R}{1-\gamma}\right)\frac{2\log |\mcs|}{\sigma},
\end{align}
which completes the proof. 

\subsection{Proof of Theorem \ref{thm:srpg}: Convergence Rate of Smoothed Robust Policy Gradient}
In this section we prove Theorem \ref{thm:srpg} and show the convergence rate of our smoothed robust policy gradient algorithm. 

The following lemma can be derived directly using existing classic results:
\begin{lemma}(Theorem 10.15 in \cite{beck2017first})\label{thm:gm}
Set the step sizes $\alpha_t=\alpha=\frac{1}{L_\sigma}$, and define the gradient mapping as $G^\alpha(\theta)\triangleq \frac{1}{\alpha}\left(\theta-\proj(\theta-\alpha \nabla J_\sigma(\theta)) \right)$, then 
\begin{align}
    \min_{t=0,...,T-1}\|G^\alpha(\theta_t)\|\leq \sqrt{\frac{2L_\sigma (J_\sigma(\theta_0)-J_\sigma^*)}{T}}.
\end{align}
\end{lemma}

We then prove Theorem \ref{thm:srpg}.

It has been shown in Lemma 3 in \cite{ghadimi2016accelerated} that if $\|G^\alpha(\theta) \|\leq \epsilon$, then 
\begin{align}\label{eq:113}
    -\nabla J_\sigma(\theta^+)\in N_{\Delta(\mca)^{|\mcs|}}(\theta^+)+2\epsilon B_2,
\end{align}
where $N_{\Delta(\mca)^{|\mcs|}}(x)\triangleq \left\{ g\in \mathbb{R}^{|\mcs|\times|\mca|}: \langle g, y-x\rangle\leq 0 \text{ for any } y \in \Delta(\mca)^{|\mcs|} \right\} $ is the normal cone, $B_2$ is the unit $l_2$ ball and $\theta^+=\theta-\alpha G^\alpha(\theta) $.  From PL-condition in Theorem \ref{thm:smPL}, it can be shown that
\begin{align}
    \min_{t\leq T-1} J_{\sigma}(\theta_t)-J_{\sigma}^* &\leq C_{PL} \min_{t\leq T-1}\max_{\hat{\pi}\in \left(\Delta\left(|\mca|\right)\right)^{|\mcs|}} \left\langle \pi_{\theta_t}-\hat{\pi}, \nabla J_\sigma\left(\theta_t\right) \right\rangle+\left(\frac{\gamma R}{1-\gamma}\right)\frac{2\log |\mcs|}{\sigma}\nn\\
    &\leq C_{PL} \max_{\hat{\pi}\in \left(\Delta\left(|\mca|\right)\right)^{|\mcs|}} \left\langle \pi_{\theta_W}-\hat{\pi}, \nabla J_\sigma\left(\theta_W\right) \right\rangle+\left(\frac{\gamma R}{1-\gamma}\right)\frac{2\log |\mcs|}{\sigma}
\end{align} 
where $W\triangleq 1+\arg\min_{t\leq T-1} \|G^\alpha(\theta_t)\|$. Recall that in Lemma \ref{thm:gm}, we showed that 
\begin{align}
    \|G^\alpha(\theta_{W-1})\|\leq \sqrt{\frac{2L_\sigma (J_\sigma(\theta_0)-J_\sigma^*)}{T}}.
\end{align}
Hence if we set 
\begin{align}
    T&= \frac{64|\mcs|C^2_{PL}L_\sigma C_\sigma}{\epsilon^2},
\end{align}
then 
\begin{align}
\|G^\alpha(\theta_{W-1})\|\leq\sqrt{\frac{2L_\sigma (J_\sigma(\theta_0)-J_\sigma^*)}{T}}\leq \frac{\epsilon}{4\sqrt{|\mcs|}C_{PL}}, 
\end{align}
which is from $J_\sigma(\theta_0)-J_\sigma^*\leq 2 \sup_{\theta\in\Delta(\mca)^{|\mcs|}} J_\sigma(\theta)\leq 2C_\sigma$.

Hence from \eqref{eq:113}, we know that $-\nabla J_\sigma(\theta_W)=g_1+g_2$, where $g_1\in N_{\Delta(\mca)^{|\mcs|}}(\theta_W)$, and $g_2\in \frac{\epsilon}{2\sqrt{|\mcs|}C_{PL}} B_2$. Thus 
\begin{align}\label{eq:119}
    \max_{\hat{\pi}\in \left(\Delta\left(|\mca|\right)\right)^{|\mcs|}} \left\langle \pi_{\theta_W}-\hat{\pi}, \nabla J_\sigma\left(\theta_W\right) \right\rangle&=\max_{\hat{\pi}\in \left(\Delta\left(|\mca|\right)\right)^{|\mcs|}} \left\langle \hat{\pi}-\pi_{\theta_W} , -\nabla J_\sigma\left(\theta_W\right) \right\rangle\nn\\
    &=\max_{\hat{\pi}\in \left(\Delta\left(|\mca|\right)\right)^{|\mcs|}} \left\langle \hat{\pi}-\pi_{\theta_W} , g_1+g_2 \right\rangle\nn\\
    &\leq \sup_{\hat{\pi}\in \left(\Delta\left(|\mca|\right)\right)^{|\mcs|}}\|\hat{\pi}-\pi_{\theta_W}\| \frac{\epsilon}{2\sqrt{|\mcs|}C_{PL}},
\end{align}
where the last inequality is from $g_1\in N_{\Delta(\mca)^{|\mcs|}}(\theta_W)$ and $\langle \pi-\pi_{\theta_W}, g_1 \rangle \leq 0$ for any $\pi \in \left(\Delta\left(|\mca|\right)\right)^{|\mcs|}$.

Note that $\|\pone-\ptwo\|\leq 2\sqrt{|\mcs|}$, hence from \eqref{eq:119},

we have that
\begin{align}
    \max_{\hat{\pi}\in \left(\Delta\left(|\mca|\right)\right)^{|\mcs|}} \left\langle \pi_{\theta_W}-\hat{\pi}, \nabla J_\sigma\left(\theta_W\right) \right\rangle \leq \frac{\epsilon}{C_{PL}}, 
\end{align}
which further implies that
\begin{align}
 \min_{t\leq T-1} J_{\sigma}(\theta_t)-J_{\sigma}^* \leq \epsilon+\left(\frac{\gamma R}{1-\gamma}\right)\frac{2\log |\mcs|}{\sigma},
\end{align}
Hence if we set 
\begin{align}
    \sigma &=\frac{2\log |\mcs|\left(\frac{\gamma R}{1-\gamma}\right)}{\epsilon }=\mathcal{O}(\epsilon^{-1}),\\
    T&= \frac{64|\mcs|C^2_{PL}L_\sigma C_\sigma}{\epsilon^2}=\mathcal{O}(\epsilon^{-3}),
\end{align}
then 
\begin{align}
    \min_{t\leq T-1} J_{\sigma}(\theta_t)-J_{\sigma}^* \leq 2\epsilon,
\end{align}
which means Algorithm \ref{alg:srpgfull} finds a global $\epsilon$-optimum of $J_\sigma$ in $\mathcal{O}(\epsilon^{-3})$ steps. 

Further it can be shown that 
\begin{align}
\min_{t<T} J(\theta_t)-J^* \leq \min_{t<T} J_\sigma (\theta_t)-J^*_\sigma+\frac{\gamma R}{1-\gamma}\frac{2\log |\mcs|}{\sigma}\leq 3\epsilon.
\end{align}
This hence completes the proof.

\section{Robust Actor-Critic under Tabular Setting}
\subsection{Convergence of Robust TD under Tabular Setting}\label{sec:td}
In this section we prove that robust TD and smoothed robust TD  converge asymptotically under the tabular setting. Note that if we set $Q_\zeta=\zeta$ with $\zeta$ being the Q-table in Algorithm \ref{alg:NTD}, it reduces to the robust TD algorithm. For completeness, we also present the smoothed robust TD algorithm under the tabular setting in Algorithm \ref{alg:rtd}.  
\begin{algorithm}[!htb]
\caption{Smoothed Robust TD (Tabular Setting)}
\label{alg:rtd}
\textbf{Input}:   $T_c$, $\pi$, $\sigma$\\
\textbf{Initialization}: $Q_0$, $s_0$
\begin{algorithmic} 
\FOR {$t=0,1,...,T_c-1$}
\STATE Choose $a_t\sim\pi(\cdot|s_t)$ and observe $c_t, s_{t+1}$
\STATE $V_t(s)\leftarrow \sum_{a\in\mca} \pi(a|s)Q_t(s,a)$ \text{ for all }$s\in\mcs$
\STATE {$Q_{t+1}(s_t,a_t)\leftarrow Q_t(s_t,a_t)+\alpha_t (c_t+\gamma (1-R) \cdot V_t(s_{t+1})+\gamma R \cdot \text{LSE}(V_t)-Q_t(s_t,a_t))$}
\ENDFOR
\end{algorithmic}
\textbf{Output}: $Q_{T_c}$
\end{algorithm}

\begin{theorem}
If step-sizes $\alpha_t$ satisfy $\sum^{\infty}_{t=0}\alpha_t=\infty$, $\sum^{\infty}_{t=0}\alpha_t^2<\infty$, then Algorithm \ref{alg:rtd} (Algorithm \ref{alg:NTD}) converges to the smoothed robust action-value function $Q^\pi_\sigma$ (robust action action-value function $Q^\pi$) almost surely. 
\end{theorem}
\begin{proof}
We present the proof for smoothed robust TD here. The proof of non-smoothed one can be similarly derived.

Define the smoothed robust Bellman operator for the robust $Q$-function as follows:
\begin{align}
    \mathbf T^\pi_\sigma Q(s,a)=c(s,a)+\gamma (1-R) \sum_{s'\in\mcs}p^a_{s,{s'}} \left( \sum_b \pi(b|{s'})Q({s'},b)\right)+\gamma R \cdot \lse \left( \sum_b \pi(b|{s'})Q({s'},b)\right).
\end{align}
Note that
\begin{align}
    \|\mathbf T^\pi_\sigma Q_1-\mathbf T^\pi_\sigma Q_2\|_{\infty}&=\max_{s,a} \bigg|\gamma (1-R) \sum_{s'\in\mcs}p^a_{s,{s'}} \left( \sum_b \pi(b|{s'})(Q_1({s'},b)-Q_2({s'},b))\right)\nn\\
    &\quad+\gamma R \cdot  \lse \left( \sum_b \pi(b|{s'})Q_1({s'},b)\right)-\gamma R \cdot  \lse \left( \sum_b \pi(b|{s'})Q_2({s'},b)\right) \bigg|\nn\\
    &\overset{(a)}{\leq}  \gamma (1-R) \sum_{s'\in\mcs}p^a_{s,{s'}} \|Q_1-Q_2\|_{\infty}+\bigg|\gamma R \cdot \left( \sum_b \pi(b|{s'})(Q_1({s'},b)-Q_2({s'},b))\right)\bigg|\nn\\
    &\leq \gamma  \|Q_1-Q_2\|_{\infty},
\end{align}
where $(a)$ is from the fact that $\lse$ is $1$-Lipschitz for any $\sigma>0$, i.e., $|\lse(V_1)-\lse(V_2)|\leq \|V_1-V_2\|_{\infty}$. Therefore, $\mathbf T^\pi_\sigma$ is a contraction and  $Q^\pi_\sigma$ is its fixed point. Hence  following \cite{borkar2000ode}, smoothed robust TD converges to its fixed point $Q^\pi_\sigma$ almost surely.

The proof for the convergence of the non-smoothed robust TD (Algorithm \ref{alg:rtd}) follows similarly, and is omitted here.  
\end{proof}

\subsection{Robust Actor-Critic under Tabular Setting}\label{sec:ssrpg}
\begin{algorithm}[!htb]
\caption{Robust Smoothed Actor-Critic under Tabular Setting}
\label{alg:sac}
\textbf{Input}:   $T$, $T_c$, $\sigma$, M\\
\textbf{Initialization}: $\theta_0$
\begin{algorithmic} 
\FOR {$t=0,1,...,T-1$}
\STATE {Run Algorithm \ref{alg:rtd} for $T_c$ times until $\|Q_{T_c}-\sq^{\pit}\|\leq \epe$}
\STATE {$Q_t\leftarrow Q_{ {T_c}}$}
\STATE {$V_t(s)\leftarrow \sum_{a\in\mca} \pit(a|s)Q_t(s,a)$ \text{ for all }$s\in\mcs$}
\FOR {$i=1,...,M$}
\STATE Sample $T^i\sim\textbf{Geom}(1-\gamma+\gamma R)$
\STATE {Sample $s^i_0\sim\rho$}
\STATE Sample trajectory starting from $s^i_0$: $(s^i_0, a^i_0,...,s^i_{T^i})$ 
\STATE $B^i_t\leftarrow \frac{1}{1-\gamma+\gamma R}\sum_{a\in\mca} \nabla\pit(a|s^i_{T^i})Q_t(s^i_{T^i},a)$
\STATE Sample $x^i_0\sim \text{softmax}(\sigma, V^t)$
\STATE Sample trajectory starting from $x^i_0$: $(x^i_0, b^i_0,...,x^i_{T^i})$ 
\STATE $D^i_t\leftarrow \frac{1}{1-\gamma+\gamma R}\sum_{a\in\mca} \nabla\pit(a|x^i_{T^i})Q_t(x^i_{T^i},a)$
\STATE $g^i_t\leftarrow B^i_t+\frac{\gamma R}{1-\gamma}D^i_t$
\ENDFOR
\STATE $g_t\leftarrow \frac{\sum^M_{i=1}g^i_t }{M}$
\STATE $\theta_{t+1}\leftarrow \mathbf{\prod}_\Theta (\theta_t-\alpha_tg_t)$
\ENDFOR
\end{algorithmic}
\textbf{Output}: $\theta_T$
\end{algorithm}
We then show the convergence of Algorithm \ref{alg:sac}. First we show that the update $g_t$ is an unbiased estimate of the gradient $\nabla J_\sigma$ if $Q_t=Q^\pi_\sigma$.
\begin{lemma}
If $Q_t=\sq^{\theta_t}$, then $g_t$ is an unbiased estimate of $\nabla J_\sigma$, i.e., \begin{align}
    \mathbb{E}\left[  g_t |\mathcal{F}_t \right]=\nabla J_\sigma(\theta_t),
\end{align}
where $\mathcal{F}_t$ denotes the $\sigma$-field generated by all the randomness until the $t$-th iteration.
\end{lemma}
\begin{proof}
First note that $T^i$ is generated following the geometry distribution $\textbf{Geom}(1-\gamma+\gamma R)$, thus for any $s,{s'}\in\mcs$
\begin{align}
    \mathbb{P}(S_{T^i}={s'}|S_0=s)&=\sum_{k=0}^\infty \mathbb{P}(T^i=k)\mathbb{P}(S_k={s'}|S_0=s,\pit)\nn\\
    &=\sum_{k=0}^\infty (1-\gamma+\gamma R) (\gamma-\gamma R)^k\mathbb{P}(S_k={s'}|S_0=s,\pit)\nn\\
    &=(1-\gamma+\gamma R)\sum_{k=0}^\infty (\gamma-\gamma R)^k\mathbb{P}(S_k={s'}|S_0=s,\pit)\nn\\
    &=d^{\pit}_{s}({s'}).
\end{align}
Hence,
\begin{align}\label{eq:130}
    \mathbb{E}[B^i_t|\mathcal{F}_t]&=\mathbb{E}\left[\frac{1}{1-\gamma+\gamma R}\sum_{a\in\mca} \nabla\pi_{\theta_t}(a|S^i_{T^i})Q_t(S^i_{T^i},a)\bigg|\mathcal{F}_t \right]\nn\\
    &=\frac{1}{1-\gamma+\gamma R} \sum_{s\in\mcs}\rho(s)\sum_{s'\in\mcs} \mathbb{P}(S^i_{T^i}={s'}|S^i_0=s) \sum_{a\in\mca} \nabla\pi_{\theta_t}(a|{s'})Q_t({s'},a)\nn\\
    &=\frac{1}{1-\gamma+\gamma R} \sum_{s'\in\mcs} d^{\pi_{\theta_t}}_{\rho}({s'})\sum_{a\in\mca} \nabla\pi_{\theta_t}(a|{s'})Q_t({s'},a)\nn\\
    &=B(\rho,\theta_t).
\end{align}
According to the algorithm,
\begin{align}
    \mathbb{P}(x^i_0=s')=\frac{e^{\sigma V_t(s')}}{\sum_{s\in\mcs} e^{\sigma V^t(s)}},
\end{align}
hence we further have that 
\begin{align}\label{eq:133}
    \mathbb{E}[D^i_t|\mathcal{F}_t]&=\frac{1}{1-\gamma+\gamma R} \sum_{s'\in\mcs}  \mathbb{P}(x^i_0=s') \sum_{s''\in\mcs} \mathbb{P}(x^i_{T^i}=s''|x^i_0=s')\left(\sum_{a\in\mca} \nabla \pi_{\theta_t}(a|s'')Q_t(s'',a) \right)\nn\\
    &=\sum_{s'\in\mcs} \frac{e^{\sigma V_t(s')}}{\sum_{s\in\mcs} e^{\sigma V^t(s)}}\sum_{s''\in\mcs} d^{\pi_{\theta_t}}_{s'}(s'')\left(\sum_{a\in\mca} \nabla \pi_{\theta_t}(a|s'')Q_t(s'',a) \right)\nn\\
    &= \sum_{s'\in\mcs} \frac{e^{\sigma V_t(s')}}{\sum_{s\in\mcs} e^{\sigma V^t(s)}}B(s',\theta)\nn\\
    &= \sum_{s'\in\mcs} \frac{e^{\sigma V^{\pi_{\theta_t}}_\sigma(s')}}{\sum_{s\in\mcs} e^{\sigma V^{\pi_{\theta_t}}_\sigma(s')}}B(s',\theta).
\end{align}
Combining \eqref{eq:130} and \eqref{eq:133} thus implies that
\begin{align}
    \mathbb{E}\left[  g^i_t |\mathcal{F}_t \right]&=B(\rho,\theta_t)+\frac{\gamma R}{1-\gamma}\sum_{s'\in\mcs} \frac{e^{\sigma V^{\pi_{\theta_t}}_\sigma(s')}}{\sum_{s\in\mcs} e^{\sigma V^{\pi_{\theta_t}}_\sigma(s')}}B(s',\theta)\nn\\
    &=\nabla J_\sigma(\theta_t),
\end{align}
which completes the proof.
\end{proof}
However, in Algorithm \ref{alg:sac} we use an estimate $Q_t$ instead of $\sq^{\theta_t}$, and therefore the estimate $g_t$ may be biased. The next lemma develops an upper bound on this bias.
\begin{lemma}
Consider Algorithm \ref{alg:sac}. If $T_c$ is chosen such that $\|Q_{T_c}-\sq^{\theta_t}\|_\infty\leq \epe$,  then we have that for any $i,t$,
\begin{align}
     \|\mathbb{E}[g^i_t|\mathcal{F}_t]-\nabla J_\sigma(\theta_t)\|&\leq 2\sigma\epe e^{\sigma\epe}\frac{\gamma R}{1-\gamma} \frac{|\mca|(\epe+C_\sigma)}{1-\gamma+\gamma R}+\frac{\gamma R}{1-\gamma} \frac{|\mca|\epe}{1-\gamma+\gamma R}+\frac{|\mca| \epe}{1-\gamma+\gamma R}\nn\\
     &= \mathcal{O}(\epe+\sigma\epe e^{\sigma\epe}).
\end{align}
 \end{lemma}
\begin{proof}
We first have that 
\begin{align}
     &\mathbb{E}[B^i_t|\mathcal{F}_t]\nn\\
     &= \mathbb{E}\left[\frac{1}{1-\gamma+\gamma R}\sum_{a\in\mca} \nabla\pi_{\theta_t}(a|s^i_{T^i})Q_t(s^i_{T^i},a)\bigg|\mathcal{F}_t\right]\nn\\
     &=\mathbb{E}\left[\frac{1}{1-\gamma+\gamma R}\sum_{a\in\mca} \nabla\pi_{\theta_t}(a|s^i_{T^i})\sq^{\theta_t}(s^i_{T^i},a)\bigg|\mathcal{F}_t\right]\nn\\
     &\quad+\mathbb{E}\left[\frac{1}{1-\gamma+\gamma R}\sum_{a\in\mca} \nabla\pi_{\theta_t}(a|s^i_{T^i})(Q_t(s^i_{T^i},a)-\sq^{\theta_t}(s^i_{T^i},a))\bigg|\mathcal{F}_t\right]\nn\\
     &=B(\rho,\theta_t)+\mathbb{E}\left[\frac{1}{1-\gamma+\gamma R}\sum_{a\in\mca} \nabla\pi_{\theta_t}(a|s^i_{T^i})(Q_t(s^i_{T^i},a)-\sq^{\theta_t}(s^i_{T^i},a))\bigg|\mathcal{F}_t\right].
\end{align}
We can show that 
\begin{align}
\left\| \mathbb{E}\left[\frac{1}{1-\gamma+\gamma R}\sum_{a\in\mca} \nabla\pi_{\theta_t}(a|s^i_{T^i})(Q_t(s^i_{T^i},a)-\sq^{\theta_t}(s^i_{T^i},a))\bigg|\mathcal{F}_t\right]\right\|\leq \frac{|\mca| \epe}{1-\gamma+\gamma R},
\end{align}
which is from $\|\nabla \pi_{\theta_t}(a|s)\|=1$ and $\mE[|Q_t(s^i_{T^i},a)-\sq^{\theta_t}(s^i_{T^i},a)|]\leq\epe$. Hence 
\begin{align}\label{eq:biasB}
     \left\|\mathbb{E}[B^i_t|\mathcal{F}_t] -B(s,\theta_t)\right\|\leq \frac{|\mca| \epe}{1-\gamma+\gamma R},
\end{align}
which means the bias of $B^i_t$ is bounded by $\mathcal{O}({\epe})$. 

We then bound the bias of $D^i_t$. We first have that 
\begin{align}
    &\frac{\sum_{s'\in\mcs}B(s',\theta_t)e^{\sigma \sv^{\pi_{\theta_t}}(s')} }{\sum_{s'\in\mcs} e^{\sigma\sv^{\pi_{\theta_t}}(s')}}-\mE[D^i_t|\mathcal{F}_t]\nn\\
    &=\frac{\sum_{s'\in\mcs}B(s',\theta_t)e^{\sigma \sv^{\pi_{\theta_t}}(s')} }{\sum_{s'\in\mcs} e^{\sigma\sv^{\pi_{\theta_t}}(s')}}- \frac{1}{1-\gamma+\gamma R}\sum_{s'\in\mcs} \frac{e^{\sigma V_t(s')}}{\sum_{s'\in\mcs} e^{\sigma V_t(s')}}\sum_{s\in\mcs} d^{\pi_{\theta_t}}_{s'}(s)\sum_{a\in\mca} \nabla{\pi_{\theta_t}}(a|s)Q_t(s,a).
\end{align}
Define $B_t(s')\triangleq \frac{1}{1-\gamma+\gamma R}\sum_{s\in\mcs} d^{\pi_{\theta_t}}_{s'}(s)\sum_{a\in\mca} \nabla \pi_{\theta_t}(a|s)Q_t(s,a)$, then the bias of $D^i_t$ can be further written as
\begin{align}\label{eq:bias1}
    &\frac{\sum_{s'\in\mcs}B(s',\theta_t)e^{\sigma \sv^{\pi_{\theta_t}}(s')} }{\sum_{s'\in\mcs} e^{\sigma\sv^{\pi_{\theta_t}}(s')}}-\mE[D^i_t|\mathcal{F}_t]\nn\\
    &=\frac{\sum_{s'\in\mcs}B(s',\theta_t)e^{\sigma \sv^{\pi_{\theta_t}}(s')} }{\sum_{s'\in\mcs} e^{\sigma\sv^{\pi_{\theta_t}}(s')}}-\frac{\sum_{s'\in\mcs} e^{\sigma V_t(s')} B_t(s')}{\sum_{s'\in\mcs} e^{\sigma V_t(s')}}\nn\\
    &=\frac{\sum_{s,s'\in\mcs}B(s',\theta_t)e^{\sigma V_t(s)}e^{\sigma \sv^{\pi_{\theta_t}}(s')}-\sum_{s,s'\in\mcs}B_t(s')e^{\sigma \sv^{\pi_{\theta_t}}(s)}e^{\sigma V_t(s')}}{\left(\sum_{s'\in\mcs} e^{\sigma V_t(s')}\right)\left(\sum_{s'\in\mcs} e^{\sigma\sv^{\pi_{\theta_t}}(s')}\right)}\nn\\
    &=\frac{\sum_{s,s'\in\mcs}B(s',\theta_t)e^{\sigma V_t(s)}e^{\sigma \sv^{\pi_{\theta_t}}(s')}-\sum_{s,s'\in\mcs}B_t(s')e^{\sigma V_t(s)}e^{\sigma \sv^{\pi_{\theta_t}}(s')}}{\left(\sum_{s'\in\mcs} e^{\sigma V_t(s')}\right)\left(\sum_{s'\in\mcs} e^{\sigma\sv^{\pi_{\theta_t}}(s')}\right)}\nn\\
    &\quad+\frac{\sum_{s,s'\in\mcs}B_t(s')e^{\sigma V_t(s)}e^{\sigma \sv^{\pi_{\theta_t}}(s')}-\sum_{s,s'\in\mcs}B_t(s')e^{\sigma \sv^{\pi_{\theta_t}}(s)}e^{\sigma V_t(s')}}{\left(\sum_{s'\in\mcs} e^{\sigma V_t(s')}\right)\left(\sum_{s'\in\mcs} e^{\sigma\sv^{\pi_{\theta_t}}(s')}\right)}.
\end{align}

The first term can be bounded as
\begin{align}\label{eq:bound1}
    & \frac{\sum_{s,s'\in\mcs}B(s',\theta_t)e^{\sigma V_t(s)}e^{\sigma \sv^{\pi_{\theta_t}}(s')}-\sum_{s,s'\in\mcs}B_t(s')e^{\sigma V_t(s)}e^{\sigma \sv^{\pi_{\theta_t}}(s')}}{\left(\sum_{s'\in\mcs} e^{\sigma V_t(s')}\right)\left(\sum_{s'\in\mcs} e^{\sigma\sv^{\pi_{\theta_t}}(s')}\right)} \nn\\
    &=\frac{\sum_{s,s'\in\mcs}e^{\sigma V_t(s)}e^{\sigma \sv^{\pi_{\theta_t}}(s')}(B(s',\theta_t)-B_t(s'))}{\left(\sum_{s'\in\mcs} e^{\sigma V_t(s')}\right)\left(\sum_{s'\in\mcs} e^{\sigma\sv^{\pi_{\theta_t}}(s')}\right)} \nn\\
    &\leq \frac{|\mca|\epe}{1-\gamma+\gamma R},
\end{align}
where the last inequality is from \eqref{eq:biasB}. 

We then bound the second term in \eqref{eq:bias1}. Consider the numerator, 
\begin{align}
    &\sum_{s,s'\in\mcs}B_t(s')e^{\sigma V_t(s)}e^{\sigma \sv^{\pi_{\theta_t}}(s')}-\sum_{s,s'\in\mcs}B_t(s')e^{\sigma \sv^{\pi_{\theta_t}}(s)}e^{\sigma V_t(s')}\nn\\
    &=\left(\sum_{s\in\mcs} e^{\sigma V_t(s)}\right)\left(\sum_{s'\in\mcs} e^{\sigma \sv^{\pi_{\theta_t}}(s')}B_t(s')\right)-\left(\sum_{s\in\mcs} e^{\sigma \sv^{\pi_{\theta_t}}(s)}\right)\left(\sum_{s'\in\mcs} e^{\sigma V_t(s')}B_t(s')\right)\nn\\
    &=\left(\sum_{s\in\mcs} e^{\sigma V_t(s)}\right)\left(\sum_{s'\in\mcs} e^{\sigma \sv^{\pi_{\theta_t}}(s')}B_t(s')\right)- \left(\sum_{s\in\mcs} e^{\sigma V_t(s)}\right)\left(\sum_{s'\in\mcs} e^{\sigma V_t(s')}B_t(s')\right)\nn\\
    &\quad + \left(\sum_{s\in\mcs} e^{\sigma V_t(s)}\right)\left(\sum_{s'\in\mcs} e^{\sigma V_t(s')}B_t(s')\right) -\left(\sum_{s\in\mcs} e^{\sigma \sv^{\pi_{\theta_t}}(s)}\right)\left(\sum_{s'\in\mcs} e^{\sigma V_t(s')}B_t(s')\right)\nn\\
    &=\underbrace{\left(\sum_{s\in\mcs} e^{\sigma V_t(s)}\right)\left(\sum_{s'\in\mcs} \left(e^{\sigma \sv^{\pi_{\theta_t}}(s')}-e^{\sigma V_t(s')}\right)B_t(s'))\right)}_{(a)}+\underbrace{\left(\sum_{s\in\mcs} e^{\sigma V_t(s)}-\sum_{s\in\mcs} e^{\sigma \sv^{\pi_{\theta_t}}(s)}\right)\left(\sum_{s'\in\mcs} e^{\sigma V_t(s')}B_t(s')\right)}_{(b)}.
\end{align}
To bound term $(a)$, we first have that 
\begin{align}
    e^{\sigma \sv^{\pi_{\theta_t}}(s')}-e^{\sigma V_t(s')}=(\sv^{\pi_{\theta_t}}(s')-V_t(s'))\sigma e^{\sigma (\sv^{\pi_{\theta_t}}(s')+\lambda(V_t(s')-\sv^{\pi_{\theta_t}}(s')))}\leq \sigma\epe e^{\sigma \sv^{\pi_{\theta_t}}(s')} e^{\sigma \epe},
\end{align}
where the first equation is from the mean-value theorem for some $0\leq\lambda\leq 1$, and the last inequality is because $\sv^{\pi_{\theta_t}}(s')-V_t(s')\leq \epe$ and $1-\lambda\leq 1$. Hence term $(a)$ can be bounded as follows:
\begin{align}
    \left\|\left(\sum_{s\in\mcs} e^{\sigma V_t(s)}\right)\left(\sum_{s'\in\mcs} \left(e^{\sigma \sv^{\pi_{\theta_t}}(s')}-e^{\sigma V_t(s')}\right)B_t(s'))\right)\right\|\leq\left\| \left(\sum_{s\in\mcs} e^{\sigma V_t(s)}\right)\left( \sum_{s'\in\mcs} e^{ \sigma\sv^{\pi_{\theta_t}}(s')} \right)\sigma\epe e^{\sigma\epe} \right\|\sup_{s} \|B_t(s)\|.
\end{align}
To bound term $(b)$, for some $\lambda\in[0,1]$ and by the mean value theorem, we have that 
\begin{align}
    &\left\|\left(\sum_{s\in\mcs} e^{\sigma V_t(s)}-\sum_{s\in\mcs} e^{\sigma \sv^{\pi_{\theta_t}}(s)}\right)\left(\sum_{s'\in\mcs} e^{\sigma V_t(s')}B_t(s')\right)\right\|\nn\\
    &=\left\|\left(\sum_{s'\in\mcs} e^{\sigma V_t(s')}B_t(s')\right)\right\| \left|\left(\sum_{s\in\mcs}\left( \sigma e^{\sigma \sv^{\pi_{\theta_t}}(s) +\sigma (1-\lambda) (V_t(s)-\sv^{\pi_{\theta_t}}(s))} (V_t(s)-\sv^{\pi_{\theta_t}}(s))\right) \right)\right|\nn\\
    &\leq \sigma\epe e^{\sigma \epe}\sup_{s} \|B_t(s)\|\left(\sum_{s\in\mcs} e^{\sigma V_t(s)}\right)\left(\sum_{s\in\mcs} e^{\sigma \sv^{\pi_{\theta_t}}(s)}\right).
\end{align}
Hence combine the bounds on $(a)$ and $(b)$, we obtain a bound of the second term in \eqref{eq:bias1}:
\begin{align}
    \left\|\frac{\sum_{s,s'\in\mcs}B(s',\theta_t)e^{\sigma V_t(s)}e^{\sigma \sv^{\pi_{\theta_t}}(s')}-\sum_{s,s'\in\mcs}B_t(s')e^{\sigma V_t(s)}e^{\sigma \sv^{\pi_{\theta_t}}(s')}}{\left(\sum_{s'\in\mcs} e^{\sigma V_t(s')}\right)\left(\sum_{s'\in\mcs} e^{\sigma\sv^{\pi_{\theta_t}}(s')}\right)}\right\|\leq 2\sigma\epe e^{\sigma\epe}\sup_{s}\|B_t(s)\|.  
\end{align}
Note that $B_t(s')=\frac{1}{1-\gamma+\gamma R}\sum_{s\in\mcs}d^{\pi_{\theta_t}}_{s'}(s)\sum_{a\in\mca} \nabla \pi_{\theta_t}(a|s)Q_t(s,a)$, and $\|Q_t\|_\infty\leq \epe+\|\sq^{\theta_t}\|_\infty\leq \epe+C_\sigma$, hence
\begin{align}
    \sup_{s}\|B_t(s)\|\leq \frac{|\mca|(\epe+C_\sigma)}{1-\gamma+\gamma R},
\end{align}
and thus, 
\begin{align}\label{eq:bound2}
    \left\|\frac{\sum_{s,s'\in\mcs}B(s',\theta_t)e^{\sigma V_t(s)}e^{\sigma \sv^{\pi_{\theta_t}}(s')}-\sum_{s,s'\in\mcs}B_t(s')e^{\sigma V_t(s)}e^{\sigma \sv^{\pi_{\theta_t}}(s')}}{\left(\sum_{s'\in\mcs} e^{\sigma V_t(s')}\right)\left(\sum_{s'\in\mcs} e^{\sigma\sv^{\pi_{\theta_t}}(s')}\right)}\right\|\leq 2\sigma\epe e^{\sigma\epe}\frac{|\mca|(\epe+C_\sigma)}{1-\gamma+\gamma R}.
\end{align}
Finally combining the bounds in \eqref{eq:bound1} and \eqref{eq:bound2} and plugging them in \eqref{eq:bias1} implies that
\begin{align} 
    &\left\|\mathbb{E}\left[\frac{\sum_{s'\in\mcs}B(s',\theta_t)e^{\sigma \sv^{\pi_{\theta_t}}(s')} }{\sum_{s'\in\mcs} e^{\sigma\sv^{\pi_{\theta_t}}(s')}}-D^i_t\bigg|\mathcal{F}_t\right]\right\|\nn\\
    &\leq 2\sigma\epe e^{\sigma\epe}\frac{|\mca|(\epe+C_\sigma)}{1-\gamma+\gamma R} +\frac{|\mca|\epe}{1-\gamma+\gamma R}.
\end{align}
Hence the bias of $g^i_t$ can be bounded as follows
\begin{align}\label{eq:biasg}
    &\|\mathbb{E}[g^i_t|\mathcal{F}_t]-\nabla J_\sigma(\theta_t)\|\nn\\
    &\leq\|\mathbb{E}[B^i_t|\mathcal{F}_t]-B(\rho,\theta_t)\|+\frac{\gamma R}{1-\gamma} \left\| \mathbb{E}\left[\frac{\sum_{s'\in\mcs}B(s',\theta_t)e^{\sigma \sv^{\pi_{\theta_t}}(s')} }{\sum_{s'\in\mcs} e^{\sigma\sv^{\pi_{\theta_t}}(s')}}-D^i_t \bigg|\mathcal{F}_t\right]\right\|\nn\\
    &\leq 2\sigma\epe e^{\sigma\epe}\frac{\gamma R}{1-\gamma} \frac{|\mca|(\epe+C_\sigma)}{1-\gamma+\gamma R}+\frac{\gamma R}{1-\gamma} \frac{|\mca|\epe}{1-\gamma+\gamma R}+\frac{|\mca| \epe}{1-\gamma+\gamma R}\nn\\
    &\triangleq b_g =\mathcal{O}(\epe+\sigma\epe e^{\sigma\epe}).
\end{align}
\end{proof}
This theorem implies that the Algorithm \ref{alg:act} is actually a projected stochastic gradient descent with bias $b_g=\mathcal{O}(\epe+\sigma\epe e^{\sigma\epe})$.

\subsection{Proof of Theorem \ref{thm:tabularAC}: Global Convergence of Robust Actor-Critic under Tabular Setting}
Denote $\Omega_t=g_t-\nabla J_\sigma(\theta_t)$, and define the stochastic gradient map by $H_t=\frac{1}{\alpha_t}(\theta_t-\proj(\theta_t-\alpha_t g_t))$. Note that  $J_\sigma$ is $L_\sigma$-smooth, hence similar  to the proof of Theorem 1 in \cite{ghadimi2016mini}, we can show that
\begin{align}
    J_\sigma(\theta_{t+1})&\leq J_\sigma(\theta_t)-\left(\alpha_t-\frac{L_\sigma}{2}\alpha_t^2 \right)\|H_t\|^2+\alpha_t\langle \Omega_t, G_t \rangle +\alpha_t\|\Omega_t\|\|H_t-G_t\|\nn\\
    &\leq J_\sigma(\theta_t)-\left(\alpha_t-\frac{L_\sigma}{2}\alpha_t^2 \right)\|H_t\|^2+\alpha_t\langle \Omega_t, G_t \rangle +\alpha_t\|\Omega_t\|^2,
\end{align}
where $G_t\triangleq\frac{1}{\alpha_t}(\theta_t-\proj(\theta_t-\alpha_t \nabla J_\sigma(\theta_t))) $, and the last inequality is from Proposition 1 in \cite{ghadimi2016mini}. Summing up from $t=0$ to $T-1$, we have
\begin{align}\label{eq:eq1}
    \sum^{T-1}_{t=0}(\alpha_t-L_\sigma \alpha_t^2)\|H_t\|^2&\leq J_\sigma(\theta_0)-J_\sigma(\theta_{T+1})+\sum^{T-1}_{t=0}
\left( \alpha_t\langle \Omega_t, G_t\rangle+ \alpha_t \|\Omega_t\|^2\right)\nn\\
&\leq J_\sigma(\theta_0)-J_\sigma^*+\sum^{T-1}_{t=0}
\left( \alpha_t\langle \Omega_t, G_t\rangle+ \alpha_t \|\Omega_t\|^2\right).
\end{align}
Note that $G_t$ is deterministic given $\theta_t$, i.e., $\mathcal{F}_t$. Hence we have
\begin{align}\label{eq:152}
    \mathbb{E}[\langle \Omega_t, G_t\rangle|\mathcal{F}_t]=\langle \mathbb{E}[\Omega_t|\mathcal{F}_t], G_t\rangle \leq \|G_t\|b_g.
\end{align}
Define $\Omega_t^j=g^j_t-\nabla J_\sigma(\theta_t)$. Then we have that
\begin{align}
    \mE[\|\Omega_t\|^2 |\mathcal{F}_t]&= \mE\left[\left\| \frac{\sum^M_{j=1}\Omega_t^j }{M}\right\|^2\bigg|\mathcal{F_t}\right]\nn\\
    &=\mE\left[ \sum_{i,j }\left\langle \frac{\Omega_t^j }{M}, \frac{\Omega_t^i}{M}\right\rangle \bigg|\mathcal{F}_t\right]\nn\\
    &\overset{(a)}{=}\sum^M_{i=1} \mE\left[\left\| \frac{\Omega_t^i}{M}\right\|^2  + \sum_{i\neq j }\left\langle \frac{\Omega_t^j }{M}, \frac{\Omega_t^i}{M}\right\rangle \bigg|\mathcal{F}_t\right]\nn\\
    &\leq \frac{\sup_i \mE[\|\Omega^i_t\|^2]}{M} + \|\mE[\Omega^i_t|\mathcal{F}_t]\|^2\nn\\
    &\leq \frac{\sup_i \mE[\|\Omega^i_t\|^2]}{M} + b_g^2,
\end{align}
where $(a)$ is from the fact that $\Omega_t^i$ and $\Omega_t^j$ are independent for $i\neq j$. 

Note that  for any $j$, 
\begin{align}
    \mathbb{E}[\|\Omega^j_t\|^2]=\mathbb{E}[\|g^j_t-\nabla J_\sigma(\theta_t)^2 \|]\leq 2 \left( \sup_j \|g^j_t\|^2 +\sup_{\theta}\|\nabla J_\sigma(\theta)\|^2\right).
\end{align}
And we have shown that 
\begin{align}
    \sup_j \|g^j_t\|^2 \leq \left(\frac{\gamma R}{1-\gamma}+1\right)^2\frac{|\mca|^2}{(1-\gamma +\gamma R)^2} (\sup_t \|Q_t\|)^2\leq \left(\frac{\gamma R}{1-\gamma}+1\right)^2\frac{|\mca|^2}{(1-\gamma +\gamma R)^2} (C_\sigma+\epe)^2 \triangleq C^2_g,
\end{align}
and hence 
\begin{align} 
    \mathbb{E}[\|\Omega^j_t\|^2]\leq 2 \left(C_g^2 +(C^V_\sigma)^2\right),
\end{align}
where we use the fact $\|\nabla J_\sigma\|\leq C_\sigma^V$. 

Thus 
\begin{align}\label{eq:deltabound}
     \mathbb{E}[\|\Omega_t\|^2]\leq b_g^2+\frac{ 2 \left(C_g^2 +{C^V_\sigma}^2\right)}{M}\triangleq C_\Omega.
\end{align}
Thus plugging all the inequalities \eqref{eq:152} and \eqref{eq:deltabound} in \eqref{eq:eq1}, we have that
\begin{align} 
    \sum^{T-1}_{t=0}(\alpha_t-L_\sigma \alpha_t^2)\mathbb{E}\left[\|H_t\|^2|\mathcal{F}_t\right]\leq J_\sigma(\theta_0)-J_\sigma^*+ \sum^{T-1}_{t=0}\alpha_t\|G_t\|b_g+ \sum^{T-1}_{t=0}\alpha_t C_\Omega.
\end{align}
Note that $\|G_t\|\leq \| \nabla J_\sigma(\theta_t)\| \leq C^V_\sigma$, hence the last inequality becomes
\begin{align}
     \sum^{T-1}_{t=0}(\alpha_t-L_\sigma \alpha_t^2)\mathbb{E}\left[\|H_t\|^2|\mathcal{F}_t\right]\leq J_\sigma(\theta_0)-J_\sigma^*+ \sum^{T-1}_{t=0}\alpha_t C^V_\sigma b_g+ \sum^{T-1}_{t=0}\alpha_t C_\Omega.
\end{align}
Set $\alpha_t=\frac{1}{2L_\sigma}$, it then follows that 
\begin{align}
     \frac{1}{4L_\sigma}\sum^{T-1}_{t=0}\mathbb{E}\left[\|H_t\|^2|\mathcal{F}_t\right]\leq J_\sigma(\theta_0)-J_\sigma^*+\frac{TC^V_\sigma b_g}{2L_\sigma}+\frac{T}{2L_\sigma} C_\Omega,
\end{align}
and 
\begin{align}
     \mE[\|H_U\|^2]\leq \frac{4L_\sigma(J_\sigma(\theta_0)-J_\sigma^*)}{T}+2C^V_\sigma b_g+2C_\Omega,
\end{align}
where $U\sim\text{Uniform}(0,...,T-1)$. 
Similar to Corollary 3 in \cite{ghadimi2016mini}, we have that
\begin{align}
    \mathbb{E}\left[\|G_{U}\|^2\right]&\leq 2\mathbb{E}\left[\|H_{U}\|^2 \right]+2\mathbb{E}\left[\|H_{U}-\nabla J_\sigma(\theta_U)\|^2  \right]\nn\\
    &=2\mathbb{E}\left[\|H_U\|^2 \right]+2\mathbb{E}\left[\|\Omega_U\|^2  \right]\nn\\
    &\leq  \frac{8L_\sigma(J_\sigma(\theta_0)-J_\sigma^*)}{T}+4C^V_\sigma b_g+6 C_\Omega\nn\\
    &\triangleq \epsilon_G.
\end{align}

Note that $G_t$ is fully determined by $\theta_{t}$, hence we have that
\begin{align}
    \mathbb{E}\left[\|G_t\|^2\right]=\sum_{\theta\in(\Delta(\mca))^{|\mcs|}} \mathbb{P}(\theta_t=\theta) \|G(\theta)\|^2,
\end{align}
where $G(\theta)=\frac{1}{\alpha} \left(\theta-\proj(\theta-\alpha \nabla J_\sigma(\theta))\right)$, and we denote $\theta^+=\theta-\alpha G(\theta)$. 

If we denote $\|G(\theta)\|^2\triangleq \epsilon_{\theta}$, then $\int_{{\theta}\in(\Delta(\mca))^{|\mcs|}} \epsilon_{\theta} d\mathbb{P}(\theta_t={\theta})   =\mE[\|G_t\|^2]$. 
Following the proof of Theorem \ref{thm:srpg}, and because of $\|G({\theta})\|=\sqrt{\epsilon_{\theta}}$, we have
\begin{align}
    -\nabla J_\sigma({\theta}^+)\in N_{(\Delta(\mca))^{|\mcs|}}({\theta}^+)+\sqrt{\epsilon_{\theta}} B_2.
\end{align}
From the PL-condition in Theorem \ref{thm:smPL}, we further have that 
\begin{align}
    J_\sigma({\theta}^+)-J_\sigma^*&\leq C_{PL}\max_{\hat{\pi}\in(\Delta(\mca))^{|\mcs|}}\langle \pi_{{\theta}^+}-\hat{\pi}, \nabla J_\sigma({\theta}^+)\rangle+\left(\frac{\gamma R}{1-\gamma}\right)\frac{2\log |\mcs|}{\sigma}\nn\\
    &\leq 2C_{PL}\sqrt{ |\mcs|\epsilon_{\theta}}+\left(\frac{\gamma R}{1-\gamma}\right)\frac{2\log |\mcs|}{\sigma}.
\end{align}
Hence we have that for any $t$:
\begin{align}
    \mE[J_\sigma(\theta_t^+)-J_\sigma^*]&=\int_{{\theta}\in\Theta}   (J_\sigma({\theta}^+)-J_\sigma^*)d\mathbb{P}(\theta_t={\theta}) \nn\\
    &\leq \int_{{\theta}\in\Theta} 2C_{PL}\sqrt{ |\mcs|\epsilon_{\theta}}\mathbb{P}(\theta_t={\theta}) +\left(\frac{\gamma R}{1-\gamma}\right)\frac{2\log |\mcs|}{\sigma}\nn\\
    &\leq 2C_{PL}\sqrt{|\mcs|}\sqrt{\int_{{\theta}\in\Theta}  \epsilon_{\theta}d\mathbb{P}(\theta_t={\theta})}+\left(\frac{\gamma R}{1-\gamma}\right)\frac{2\log |\mcs|}{\sigma}\nn\\
    &= 2C_{PL}\sqrt{|\mcs|}\sqrt{\mE[\|G_t\|^2]}+\left(\frac{\gamma R}{1-\gamma}\right)\frac{2\log |\mcs|}{\sigma}.
\end{align}
And further we have that 
\begin{align}
     \mE[J_\sigma(\theta_U^+)-J_\sigma^*]&=\frac{\sum^{T-1}_{t=0} \mE[J_\sigma(\theta_t^+)-J_\sigma^*]}{T}\nn\\
     &=\frac{\sum^{T-1}_{t=0}2C_{PL}\sqrt{|\mcs|}\sqrt{\mE[\|G_t\|^2]} }{T}+\left(\frac{\gamma R}{1-\gamma}\right)\frac{2\log |\mcs|}{\sigma}\nn\\
     &\leq 2C_{PL}\sqrt{|\mcs|}\sqrt{\frac{\sum^{T-1}_{t=0} \mE[\|G_t\|^2]}{T}}+\left(\frac{\gamma R}{1-\gamma}\right)\frac{2\log |\mcs|}{\sigma}\nn\\
     &\leq 2C_{PL}\sqrt{|\mcs|\epsilon_G}+\left(\frac{\gamma R}{1-\gamma}\right)\frac{2\log |\mcs|}{\sigma}.
\end{align}

Now consider $J_\sigma(\theta_{U+1})$. Note that 
\begin{align}
    \mE[J_\sigma(\theta_{U+1})-J_\sigma(\theta_U^+)]&\leq \mE[C^V_\sigma\|\theta_U^+-\theta_{U+1}\|]\nn\\
     &\overset{(a)}{\leq} C^V_\sigma \alpha \mE[\|g_{U}-\nabla J_\sigma(\theta_{U}) \| ]\nn\\
    &=C^V_\sigma \alpha \mE[\|{\Omega}_{U} \| ]\nn\\
    &\leq C^V_\sigma  \alpha  \sqrt{C_\Omega},
\end{align}
where $(a)$ is from Lemma 2 in \cite{ghadimi2016mini}, and the last inequality is from \eqref{eq:deltabound} and the fact that $\mE[\|\Omega_t \|]\leq \sqrt{ \mE[\| {\Omega}_{t} \|^2]}\leq \sqrt{C_\Omega}$ for any $t$.

 Hence we have
\begin{align}
    \mE[J_\sigma(\theta_{U+1})-J^*_\sigma]&=\mE[J_\sigma(\theta_{U+1})-J_\sigma(\theta_U^+)+J_\sigma(\theta_U^+)-J^*_\sigma]\nn\\
    &\leq C^V_\sigma  \alpha   \sqrt{C_\Omega}+2C_{PL}\sqrt{|\mcs|\epsilon_G}+\left(\frac{\gamma R}{1-\gamma}\right)\frac{2\log |\mcs|}{\sigma}\nn\\
    &\leq 2C^V_\sigma\frac{\epsilon}{L_\sigma}+2C_{PL}\sqrt{|\mcs|\epsilon_G}+\epsilon.
\end{align}
Moreover, we have that 
\begin{align}
    |J_\sigma(\theta)-J(\theta)|&\leq \epsilon,\nn\\
    |J^*_\sigma-J^*|\leq \epsilon,
\end{align}
hence
\begin{align}
    \mE[J(\theta_{U+1})-J^*]&\leq 2\epsilon+2C^V_\sigma\frac{\epsilon}{L_\sigma}+2C_{PL}\sqrt{|\mcs|\epsilon_G}+\epsilon \nn\\
    &=2C_{PL}\sqrt{|\mcs|}\sqrt{\frac{8L_\sigma(J_\sigma(\theta_0)-J_\sigma^*)}{T}+4C^V_\sigma b_g+6 C_\Omega}   +3\epsilon+2C^V_\sigma\frac{\epsilon}{L_\sigma}.
\end{align}
Plug in the definition of $C_\Omega$, and we further have that 
\begin{align}\label{eq:acbound}
    \mE[J(\theta_{U+1})-J^*]&\leq 2\epsilon+2C^V_\sigma\frac{\epsilon}{L_\sigma}+2C_{PL}\sqrt{|\mcs|\epsilon_G}+\epsilon \nn\\
    &=2C_{PL}\sqrt{|\mcs|}\sqrt{\frac{8L_\sigma(J_\sigma(\theta_0)-J_\sigma^*)}{T}+4C^V_\sigma b_g+6 \left( b_g^2+\frac{ 2 \left(C_g^2 +{C^V_\sigma}^2\right)}{M}\right)}   +3\epsilon+2C^V_\sigma\frac{\epsilon}{L_\sigma}\nn\\
    &\leq 2C_{PL}\sqrt{|\mcs|}\left(\sqrt{\frac{16L_\sigma C_\sigma }{T}}+ \sqrt{4C^V_\sigma b_g +6b_g^2}+\sqrt{\frac{ 12 \left(C_g^2 +{C^V_\sigma}^2\right)}{M}} \right)
\end{align}

Hence if we set 
\begin{align}
    b_g&\leq \frac{\epsilon}{2\sqrt{6|\mcs|}C_{PL}},\\
    M&=\frac{48(C_g^2+{C^V_\sigma}^2)|\mcs|C_{PL}^2}{\epsilon^2},\\
    T&=\frac{64L_\sigma C_\sigma|\mcs|C_{PL}^2}{\epsilon^2},
\end{align}
we have that
\begin{align}
     \min_{1\leq t \leq T}\mE[J(\theta_t)]\leq  \mE[J(\theta_{U+1})-J^*]&\leq 6\epsilon+2C^V_\sigma\frac{\epsilon}{L_\sigma}\leq 7\epsilon.
\end{align}

\section{Useful Lemmas}
\begin{lemma}\label{lemma:sub_max}
Let $F(x)=\max\left\{ f_1(x),...,f_n(x)\right\}$, and for any $x$, denote $I(x)\triangleq\arg\max_i \left\{ f_i(x)\right\}$. Then $\left\{ \partial f_i(x): i\in I_x \right\} \subseteq \partial F(x)$.
\end{lemma}
\begin{proof}
From the definition, we know that for any $i\in I_x$ and $g\in\partial f_i(x)$, we have
\begin{align}
    \lim\inf_{y\to x} \frac{f_i(y)-f_i(x)-\langle g, y-x\rangle}{\|y-x\|}\geq 0.
\end{align}
It can be showed that $g\in\partial F(x)$ as follows:
\begin{align}
    \lim\inf_{y\to x} \frac{F(y)-F(x)-\langle g, y-x\rangle}{\|y-x\|}\geq \lim\inf_{y\to x} \frac{f_i(y)-f_i(x)-\langle g, y-x\rangle}{\|y-x\|}\geq 0,
\end{align}
which is from $F(y)\geq f_i(y)$ and $F(x)=f_i(x)$. And this completes the proof. 
\end{proof}

\section{Additional Experiments}
In this section we present some additional experiments to demonstrate our theoretical results.

\textbf{Robust Policy Gradient.}
We provide more experiment results on robust policy gradient v.s. non-robust one. In \cref{Fig.arpg} we compare the two algorithms on $\mathcal{G}(12,6)$, and in \cref{Fig.ag20}, we compare them on $\mathcal{G}(20,10)$. All the results show that our robust policy gradient can find a policy that has higher accumulated discounted reward under the worst-case transition kernel.

\begin{figure}[htbp]
\begin{center}
\subfigure[$R=0.12$.]{
\label{Fig.ag1}
\includegraphics[width=0.47\linewidth]{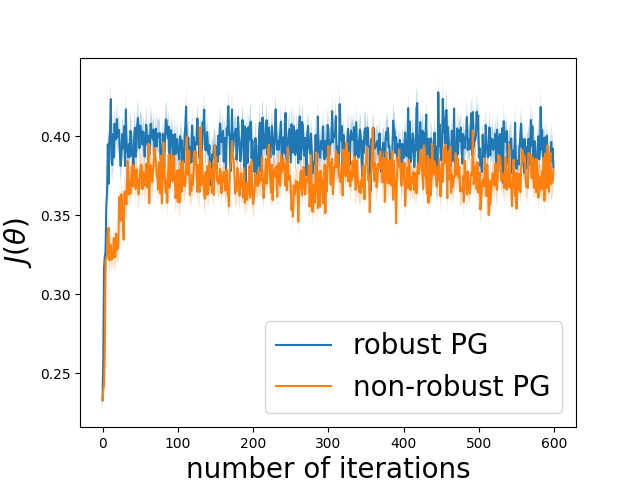}}
\subfigure[$R=0.15$.]{
\label{Fig.ag2}
\includegraphics[width=0.47\linewidth]{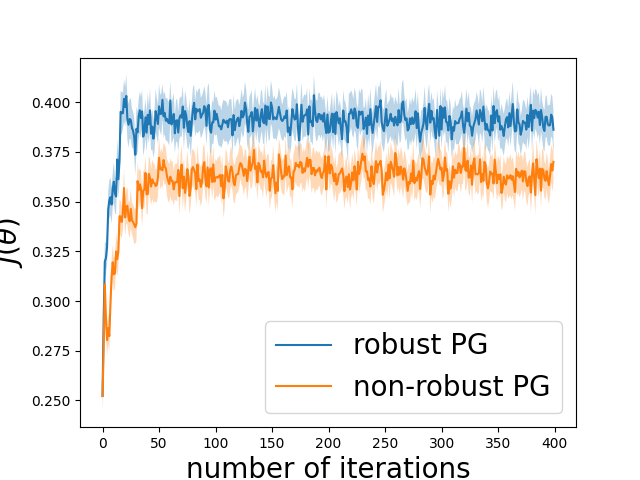}}
\subfigure[$R=0.2$.]{
\label{Fig.ag3}
\includegraphics[width=0.47\linewidth]{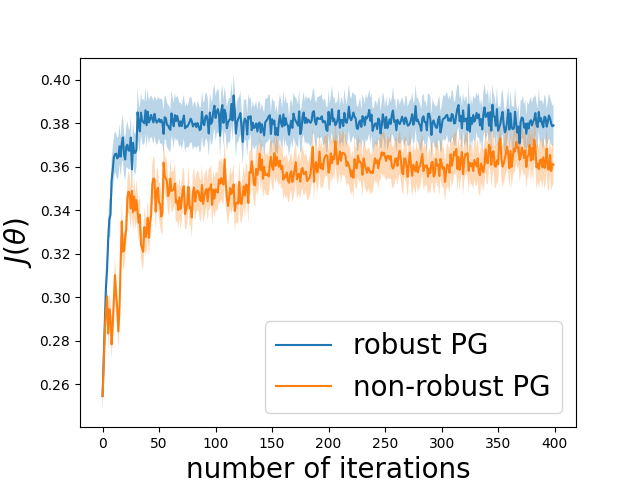}}
\subfigure[$R=0.25$.]{
\label{Fig.ag4}
\includegraphics[width=0.47\linewidth]{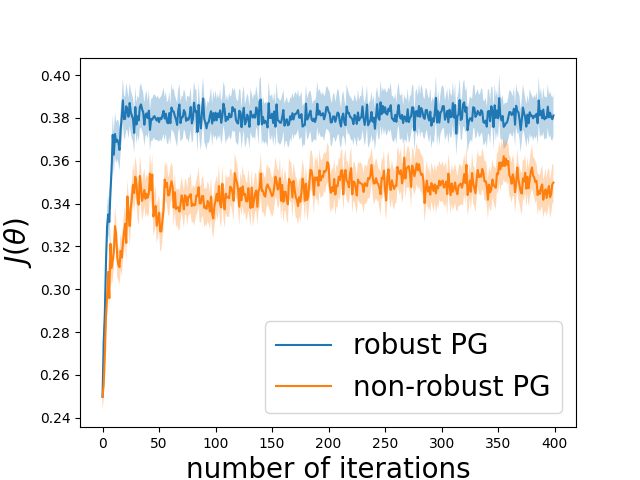}}
\caption{Robust Policy Gradient v.s. Non-robust Policy Gradient on Garnet Problem $\mathcal{G}(12,6)$.}
\label{Fig.arpg}
\end{center}
\vskip -0.2in
\end{figure}

\begin{figure}[ht]
\vskip 0.2in
\begin{center}
\subfigure[$R=0.1$.]{
\label{Fig.ag201}
\includegraphics[width=0.47\linewidth]{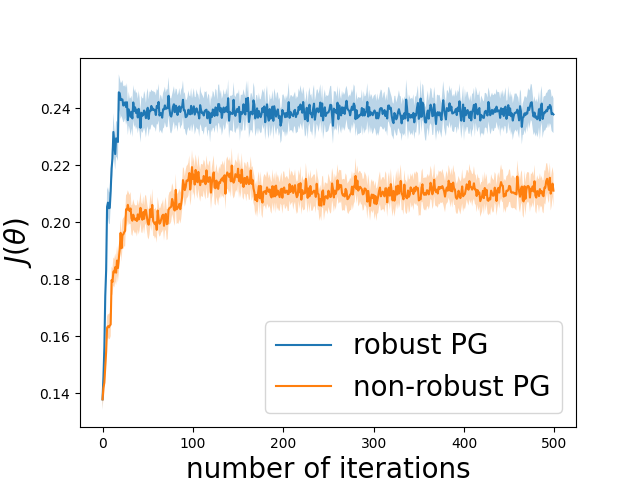}}
\subfigure[$R=0.2$.]{
\label{Fig.ag202}
\includegraphics[width=0.47\linewidth]{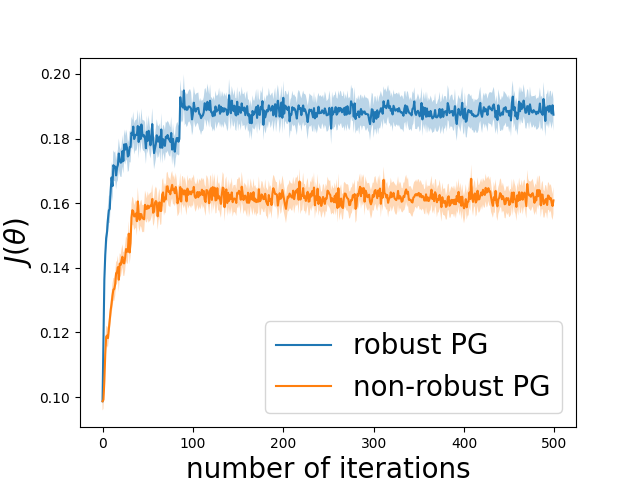}}
\caption{Robust Policy Gradient v.s. Non-robust Policy Gradient on Garnet Problem $\mathcal{G}(20,10)$.}
\label{Fig.ag20}
\end{center}
\vskip -0.2in
\end{figure}

\newpage
\textbf{Robust Actor-Critic.}

In \cref{Fig.AC2}, we compare robust actor-critic and vanilla one on Garnet problem $\mathcal{G}(40,15)$ using neural policy and neural network to approximate robust value functions. We plot the discounted accumulative reward under the worst case v.s. number of iterations. Our results suggest that robust actor-critic is more robust than vanilla actor-critic under the model mismatch. 
\begin{figure}[ht]
\vskip 0.2in
\begin{center}
\subfigure[$R=0.15$.]{
\label{Fig.ac3}
\includegraphics[width=0.47\linewidth]{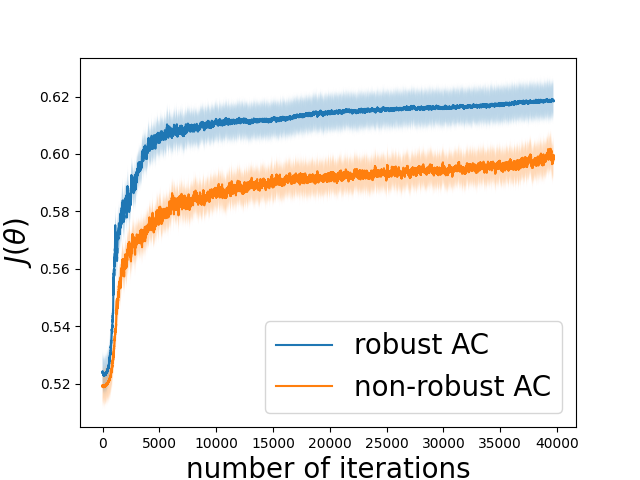}}
\subfigure[$R=0.2$.]{
\label{Fig.ac4}
\includegraphics[width=0.47\linewidth]{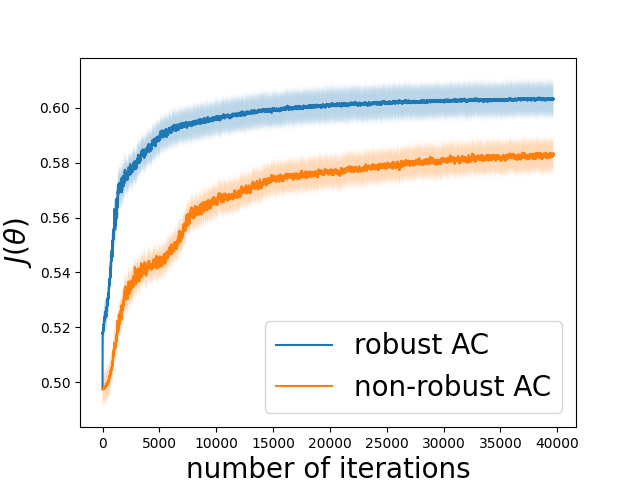}}
\caption{Robust Actor-Critic v.s. Non-robust Actor-Critic on Garnet Problem $\mathcal{G}(40,15)$.}
\label{Fig.AC2}
\end{center}
\vskip -0.2in
\end{figure}

In \cref{Fig.AC}, we consider Garnet problem $\mathcal{G}(30,10)$ using direct policy parameterization, and we use a two-layer neural network (with 20 neurons in the hidden layer) in the critic to approximate the robust value function.
As the results show, our robust actor-critic algorithm finds a policy that achieves a higher accumulated discounted reward under the worst-case transition kernel than the vanilla actor-critic algorithm. 
\begin{figure}[ht]
\vskip 0.2in
\begin{center}
\subfigure[$R=0.15$.]{
\label{Fig.ac1}
\includegraphics[width=0.47\linewidth]{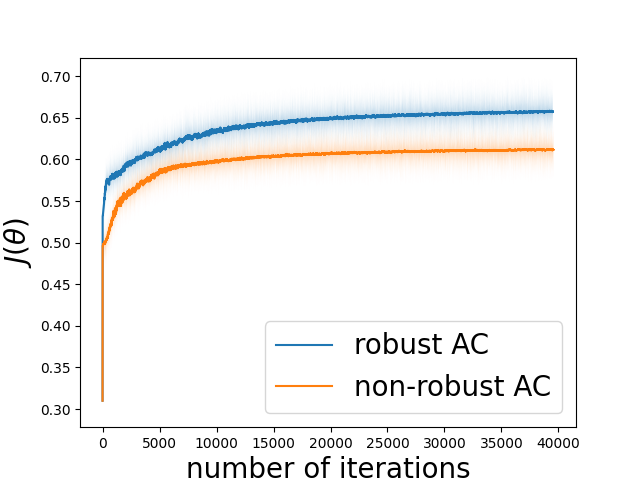}}
\subfigure[$R=0.2$.]{
\label{Fig.ac2}
\includegraphics[width=0.47\linewidth]{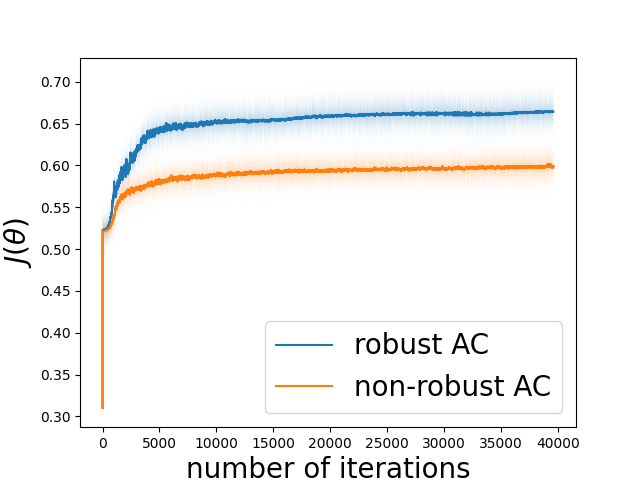}}
\caption{Robust Actor-Critic v.s. Non-robust Actor-Critic on Garnet Problem $\mathcal{G}(30,10)$.}
\label{Fig.AC}
\end{center}
\vskip -0.2in
\end{figure}

\newpage
\section{Constants}
In this section we list the definition of all the constants in this paper.
\begin{align}
    L_V&=\frac{k_\pi |\mca|}{(1-\gamma)^2},\nn\\
    C_{PL}&={\frac{1}{(1-\gamma)\mu_{\min}}},\nn\\
    C_\sigma&=\frac{1}{1-\gamma}(1+2\gamma R \frac{\log|\mcs|}{\sigma}),\nn\\
    C^V_\sigma&=\frac{1}{1-\gamma} |\mca|k_\pi C_\sigma,\nn\\
    k_B&=\frac{1}{1-\gamma +\gamma R} \left( |\mca|C_\sigma l_\pi+|\mca|k_\pi C^V_\sigma\right) +\frac{2|\mca|^2\gamma (1-R)}{(1-\gamma+\gamma R)^2}k_\pi^2C_\sigma,\nn\\
    L_\sigma&=k_B+\frac{\gamma R}{1-\gamma}\left(\sqrt{|\mcs|}k_B+ 2\sigma|\mcs|C^V_\sigma \frac{1}{1-\gamma+\gamma R}k_\pi|\mca|C_\sigma\right),\nn\\
    b_g&=2\sigma\epe e^{\sigma\epe}\frac{\gamma R}{1-\gamma} \frac{|\mca|(\epe+C_\sigma)}{1-\gamma+\gamma R}+\frac{\gamma R}{1-\gamma} \frac{|\mca|\epe}{1-\gamma+\gamma R}+\frac{|\mca| \epe}{1-\gamma+\gamma R},\nn\\
    C_g&= \left(\frac{\gamma R}{1-\gamma}+1\right)\frac{|\mca|}{(1-\gamma +\gamma R)} (C_\sigma+\epe),\nn\\
    C_\Omega&=b_g^2+\frac{ 2 \left(C_g^2 +{C^V_\sigma}^2\right)}{M},\nn\\
    \epsilon_G&= \frac{8L_\sigma(J_\sigma(\theta_0)-J_\sigma^*)}{T}+4C^V_\sigma b_g+6 C_\Omega.
\end{align}

\end{document}